\documentclass{article}

\usepackage[final]{neurips_2022_arxiv}

\usepackage[utf8]{inputenc} 
\usepackage[T1]{fontenc}    
\usepackage{hyperref}       
\usepackage{url}            
\usepackage{booktabs}       
\usepackage{amsfonts}       
\usepackage{nicefrac}       
\usepackage{microtype}      
\usepackage{xcolor}         

\title{Achieving the Pareto Frontier of Regret Minimization and Best Arm Identification in Multi-Armed Bandits}

\author{%
  Zixin Zhong$^1$\thanks{
  $^1$ Department of Computing Science, University of Alberta, Canada
$^2$ Department of Industrial Systems and Management, National University of Singapore, Singapore 
$^3$ Institute of Operations Research and Analytics, National University of Singapore, Singapore
$^4$ Department of Mathematics, National University of Singapore, Singapore
$^5$ Department of Electrical and Computer Engineering,
National University of Singapore, Singapore.
} \\
  \texttt{zixin.zhong@u.nus.edu} \\
   \And
   Wang Chi Cheung$^{2,3}$ \\
   \texttt{isecwc@nus.edu.sg}\\
   \And
   Vincent Y.~F. Tan$^{3,4,5}$ \\
   \texttt{vtan@nus.edu.sg}\\
}

\usepackage[mathscr]{eucal}
\usepackage{epsfig,epsf,psfrag}
\usepackage{amssymb,amsmath,amsthm,amsfonts,latexsym}
\usepackage{amsmath,graphicx,bm,xcolor,url,overpic}
\usepackage{fixltx2e}
\usepackage{array}
\usepackage{verbatim}
\usepackage{bm,bbm}
\usepackage{verbatim}
\usepackage{textcomp}
\usepackage{mathrsfs}
\usepackage{epstopdf}

\usepackage{algorithm}
\usepackage{algorithmic}

\numberwithin{equation}{section}

\usepackage{thmtools}
\usepackage{thm-restate}
 
\catcode`~=11 \def\UrlSpecials{\do\~{\kern -.15em\lower .7ex\hbox{~}\kern .04em}} \catcode`~=13

\newcommand{\iout}{i_{ \mathrm{out} }}

\newcommand{\calB}{\mathcal{B}}

\newcommand{\calE}{\mathcal{E}}
\newcommand{\calF}{\mathcal{F}}

\newcommand{\calI}{\mathcal{I}}

\newcommand{\calN}{\mathcal{N}}

\newcommand{\bbE}{\mathbb{E}}

\newcommand{\bbN}{\mathbb{N}}

\newcommand{\bbP}{\mathbb{P}}

\newcommand{\bbR}{\mathbb{R}}

\DeclareMathAlphabet{\mathbsf}{OT1}{cmss}{bx}{n}
\DeclareMathAlphabet{\mathssf}{OT1}{cmss}{m}{sl}

\DeclareSymbolFont{bsfletters}{OT1}{cmss}{bx}{n}  
\DeclareSymbolFont{ssfletters}{OT1}{cmss}{m}{n}
\DeclareMathSymbol{\bsfGamma}{0}{bsfletters}{'000}
\DeclareMathSymbol{\ssfGamma}{0}{ssfletters}{'000}
\DeclareMathSymbol{\bsfDelta}{0}{bsfletters}{'001}
\DeclareMathSymbol{\ssfDelta}{0}{ssfletters}{'001}
\DeclareMathSymbol{\bsfTheta}{0}{bsfletters}{'002}
\DeclareMathSymbol{\ssfTheta}{0}{ssfletters}{'002}
\DeclareMathSymbol{\bsfLambda}{0}{bsfletters}{'003}
\DeclareMathSymbol{\ssfLambda}{0}{ssfletters}{'003}
\DeclareMathSymbol{\bsfXi}{0}{bsfletters}{'004}
\DeclareMathSymbol{\ssfXi}{0}{ssfletters}{'004}
\DeclareMathSymbol{\bsfPi}{0}{bsfletters}{'005}
\DeclareMathSymbol{\ssfPi}{0}{ssfletters}{'005}
\DeclareMathSymbol{\bsfSigma}{0}{bsfletters}{'006}
\DeclareMathSymbol{\ssfSigma}{0}{ssfletters}{'006}
\DeclareMathSymbol{\bsfUpsilon}{0}{bsfletters}{'007}
\DeclareMathSymbol{\ssfUpsilon}{0}{ssfletters}{'007}
\DeclareMathSymbol{\bsfPhi}{0}{bsfletters}{'010}
\DeclareMathSymbol{\ssfPhi}{0}{ssfletters}{'010}
\DeclareMathSymbol{\bsfPsi}{0}{bsfletters}{'011}
\DeclareMathSymbol{\ssfPsi}{0}{ssfletters}{'011}
\DeclareMathSymbol{\bsfOmega}{0}{bsfletters}{'012}
\DeclareMathSymbol{\ssfOmega}{0}{ssfletters}{'012}

\newcommand{\hatg}{\hat{g}}

\newcommand{\bare}{\bar{e}}

\newcommand{\bari}{\bar{i}}

\newcommand{\barA}{\bar{A}}

\newcommand{\barN}{\bar{N}}

\newcommand{\barR}{\bar{R}}

\DeclareMathOperator*{\argmax}{arg\,max}

\DeclareMathOperator{\var}{\mathrm{Var}}

\theoremstyle{plain}
\newtheorem{theorem}{Theorem}[section]

\newtheorem{lemma}[theorem]{Lemma}

\theoremstyle{definition}

\theoremstyle{remark}

\newcommand{\rmKL}{\mathrm{KL}}
\newcommand{\drug}{\mathsf{Dandit}}
\counterwithin{figure}{section}
\counterwithin{table}{section}

\newcommand{\DeltaMin}{\underline{\Delta}}
\newcommand{\HtwoMax}{\overline{H}_2}
\newcommand{\RMax}{\overline{R}}
\newcommand{\varMax}{\overline{V}}

\usepackage{hyperref}
\usepackage{url}
\usepackage{multirow}

\usepackage{enumitem} 
\newcommand{\barDelta}{\bar{\Delta}}

\newcommand{\barcalB}{\bar{\mathcal{B}}}

\usepackage{hhline}
\makeatletter
\newcommand{\thickhline}{%
    \noalign {\ifnum 0=`}\fi \hrule height 1.5pt
    \futurelet \reserved@a \@xhline
}
\newcolumntype{"}{@{\hskip\tabcolsep\vrule width 1.05pt\hskip\tabcolsep}}
\makeatother

\begin{document}

\maketitle

\begin{abstract}
We study the Pareto frontier of two archetypal objectives in multi-armed bandits, namely, regret minimization (RM) and best arm identification (BAI) with a fixed horizon. It is folklore that the balance between exploitation and exploration is crucial for both RM and BAI, but exploration is more critical in achieving the optimal performance for the latter objective. To this end, we design and analyze the {\sc BoBW-lil'UCB$(\gamma)$} algorithm. Complementarily, by establishing lower bounds on the regret achievable by any algorithm with a given BAI failure probability, we show that (i) no algorithm can simultaneously perform optimally for both the RM and BAI objectives, and (ii) {\sc BoBW-lil'UCB$(\gamma)$} achieves order-wise optimal performance for RM or BAI under different values of $\gamma$.
Our work elucidates the trade-off more precisely by showing how the constants in previous works depend on certain hardness parameters. Finally, we show that {\sc BoBW-lil'UCB} outperforms a close competitor {\sc UCB$_\alpha$} \citep{pmlr-v89-degenne19a} in terms of the time complexity and the regret 
on diverse datasets such as 
MovieLens and Published Kinase Inhibitor Set.
\end{abstract}

\section{Introduction}
Consider a drug company $\drug$ (\underline{D}rug B\underline{andit}) that wants to design an effective vaccine for a certain  virus. It has a certain number of  feasible options, say $L = 10$. 
Because $\drug$ has a limited budget, it can only test vaccines for
a fixed number of times, say $T = 1,000$. Using the limited
number of tests, it wants to find the option that will lead
to the ``best'' outcome, e.g., the maximum efficacy of the drug.  
At the same time, $\drug$ aims to protect individuals from potentially adverse side effects of the vaccines to be tested.
How can $\drug$ find the optimal drug design and, at the same time, protect the health of participants?
We design an algorithm   {\sc BoBW-lil'UCB} that allows $\drug$ to balance between these two competing targets. In complement, we also show that it is {\em impossible} for $\drug$ to achieve optimal performances for both targets {\em simultaneously}, and $\drug$ has to settle for operating on the Pareto frontier of the two objectives. 
 
 To solve $\drug$'s problem, we study the {\em Cumulative Regret Minimization} (RM) and {\em Best Arm Identification} (BAI) problems for stochastic bandits with a {\em fixed time horizon or budget}.
While most existing works only study one of these two targets~\citep{auer2002finite,audibert2010best}, 
\citet{pmlr-v89-degenne19a} designed the {\sc UCB$_\alpha$} algorithm for both RM and BAI with a {\em fixed confidence}. 
Therefore, these studies are not directly applicable to $\drug$'s problem as $\drug$ is interested in obtaining the optimal item and minimizing the damage across a {\em fixed} number of tests.
However, our setting dovetails neatly with company $\drug$'s goals. $\drug$ can utilize our algorithm to sequentially and adaptively select different design options to test the vaccines and to eventually balance between choosing the optimal vaccine and, in the process, mitigating any physical damage on the participants.   We also show that $\drug$  cannot achieve  both targets optimally and simultaneously.

Beyond any specific applications, we believe 
this problem is of fundamental theoretical importance 
in the broad context of {\em multi-armed bandits} (MAB).
In order to design an efficient bandit algorithm, a well-known challenge is to balance between {\em exploitation} and {\em exploration} \citep{auer2002finite,lattimore2020bandit,kaufmann2017learning}.
Our work {\bf quantifies} the Pareto frontier of RM and BAI, as well as the effects of exploitation and exploration on these two aims.
%
%
%

{\bf Main contributions.}
 In stochastic bandits,
there are $L$ items with different unknown reward distributions. 
 At each time step, a random reward is generated from each item's distribution.
Based on the previous observations,
a learning agent selects an item and observes its reward. 
Given the number of time steps $T \in \bbN$, the agent aims to maximize the cumulative rewards and to identify the optimal item with high probability. 

Our first main contribution is the {\sc BoBW-lil'UCB$(\gamma)$} algorithm. 
{\sc BoBW-lil'UCB$(\gamma)$}
 is designed for both RM and BAI over a fixed time horizon,
which achieves Pareto-optimality of RM and BAI in some regimes.
\begin{itemize}[ itemsep = -3pt,   topsep = -1pt, leftmargin =  15pt ]
    \item[(i)]  On one hand, we can shrink the confidence radius of each item by increasing $\gamma$, which encourages {\sc BoBW-lil'UCB$(\gamma)$} to pull items with high empirical mean rewards (exploitation) and generally  leads to high rewards (i.e., small regret).
    \item[(ii)] On the other hand, we can enlarge the confidence radius by decreasing $\gamma$ to encourage the exploration of items that have not been sufficiently pulled in previous time steps (exploration); this will result in a high  BAI success probability.
\end{itemize}
The parameter $\gamma$ in  {\sc BoBW-lil'UCB$(\gamma)$} can be tuned such that either its cumulative regret or its failure probability almost matches
the corresponding state-of-the-art lower bound~\citep{lai1985asymptotically,carpentier16tight}.
The performance of {\sc BoBW-lil'UCB$(\gamma)$} implies that exploitation is more critical in achieving the optimal performance for RM, while exploration is more crucial for BAI in stochastic bandits.
We also analyze the {\sc Exp3.P} algorithm proposed by \citet{Auer02} for both RM and BAI, which indicates the similar trade-off between these two aims in   adversarial bandits.
%


Moreover, we evaluate the Pareto frontier of RM and BAI theoretically.
In \citet{lattimore2020bandit}, Note 33.2 and Exercise 33.5 only explore the sub-optimality for BAI of an asymptotically-optimal RM algorithm, and provide asymptotic bounds with constant $\varepsilon$ 
(See Section \ref{sec:discussion} for the meaning of $\varepsilon$). 
Our work goes beyond the asymptotic regimes in that observation, by exploring the Pareto frontier of RM and BAI for any algorithm in Theorems~\ref{thm:sto_rm_bai_bern} and \ref{thm:sto_rm_bai_gauss} in the finite horizon / budget setting.
Our {\bf non-asymptotic} bounds  {\bf quantify} how the trade-off between regret and BAI probability depends on the hardness quantity $H_2$ and gap parameter $\Delta_{1,i}$'s of an instance (see definitions in Section~\ref{sec:sto_prob_set}), instead of fixed constants such as $\varepsilon$ \citep{lattimore2020bandit}. 
Another relevant work is \citet{bubeck2009pure}, which explores the trade-off between cumulative regret and \emph{simple regret}.\footnote{When there is no ambiguity, we abbreviate  ``cumulative regret'' as ``regret''.}
 Due to the relation between BAI and simple regret, our results precisely {\bf quantify} the values of constants $C$ and $D$ in \citet{bubeck2009pure} (see  Section~\ref{sec:sto_rm_bai} for details).
%
While these two works focus on the stochastic bandits, we also analyze   the Pareto frontier between RM and BAI in adversarial bandits  in Appendix~\ref{sec:discuss_adv}.

Furthermore,   {\sc BoBW-lil'UCB$(\gamma)$}   empirically outperforms a close competitor {\sc UCB$_\alpha$} \citep{pmlr-v89-degenne19a} in  difficult scenarios in which the differences between the optimal and suboptimal items are small. 
While both algorithms identify the optimal item with high probability,   UCB$_\alpha$, designed for the fixed-confidence case,  requires a longer horizon to do so and also suffers from larger regret.
This demonstrates the superiority of {\sc BoBW-lil'UCB$(\gamma)$} under the fixed-budget setting, which it is specifically designed for.

 {\bf Novelty.}
 (i) We are the first to design an algorithm for both RM and BAI with a fixed budget. We can adjust the proposed {\sc BoBW-lil'UCB$(\gamma)$} algorithm to perform (near-)optimally for both RM and BAI with proper choices of $\gamma$.
(ii) The performance of {\sc BoBW-lil'UCB$(\gamma)$} implies that 
exploitation is more crucial to obtain a small regret, while
exploration is more critical to shrink the BAI failure probability.
(iii) We quantify the Pareto frontier of RM and BAI.
We show that it is inevitable for any algorithm to compromise between RM and BAI in a fixed horizon setting.
Beyond the stochastic bandits, we also provide a preliminary study on the adversarial bandits.

{\bf Literature review.}
Both the RM and BAI problems have been  studied  extensively for stochastic multi-armed bandits.
Firstly, an RM algorithm aims to maximize its cumulative rewards,
i.e., to minimize its regret (the gap between the highest cumulative rewards  and the obtained rewards).
One line of seminal works on RM involve the class of {\em Upper Confidence Bound} (UCB) algorithms~\citep{auer2002finite,Garivier2011the},
while another line of works study 
{\em Thompson sampling} (TS) algorithms~\citep{AgrawalG12,RussoV14,agrawal2017near}.
\citet{lai1985asymptotically} derived a lower bound on the regret of any online algorithm.
%

Secondly, there are two complementary settings for BAI: (i) given $T \in\bbN$, the agent aims to maximize the probability of finding the optimal  item in at most $T$ steps~\citep{audibert2010best,karnin2013almost,zhong2021probabilistic}; (ii) given $\delta > 0$, the agent aims to find the optimal item with the probability of at least $1 - \delta$ in the smallest number of time steps~\citep{bubeck2013multiple,kaufmann13information}.
These two settings are known as the {\em fixed-budget} and {\em fixed-confidence} settings respectively. 
Moreover, 
\citet{kaufmann2016complexity} presented theoretical findings for both settings,
including a lower bound for two-armed bandits and a lower bound for multi-armed Gaussian bandits under the fixed-budget setting. 
\citet{carpentier16tight} established a lower bound on the failure probability of any algorithm in a fixed time horizon.

While most existing works focus solely on  RM or BAI, \citet{pmlr-v89-degenne19a} explored both goals with a fixed confidence and proposed the {\sc UCB$_\alpha$} algorithm.
Recently \citet{kim2023pareto} also focused on the fixed-confidence setting and studied the trade-off between RM and Pareto Front Identification (PFI) in linear bandits; PFI is a generalization of BAI since in this setting each arm has a {\em vector} reward instead of a scalar one.
\citet{simchi2023multi} provided a novel definition of Pareto Optimality and aimed to solve a corresponding minimax multi-objective optimization problem, which has a different focus from this work.
%
To the best of our knowledge, there is no existing analysis of a single, unified algorithm for both RM and BAI given a fixed horizon.
Our work fills in this gap by 
proposing the {\sc BoBW-lil'UCB$(\gamma)$} algorithm and proving that it achieves Pareto-optimality in some regimes.
We also study the Pareto frontier of RM and BAI, which depends on the balance between exploitation and exploration.
%
We show that a single algorithm cannot perform optimally for both RM and BAI simultaneously.

\vspace{-.1in}
\section{Problem Setup}
\vspace{-.1in}
\label{sec:sto_prob_set}
For any $n \in \bbN$, we denote the set $\{1, \ldots , n\}$ as $[n]$. Let there be $L \in \bbN$ ground items, contained in $[L]$.
A random variable $X $ (or its distribution) is $\sigma$-sub-Gaussian  ($\sigma$-SG) if
$\mathbb{E}[ \mathrm{e}^{\lambda (X -\bbE X)} ] \leq \exp( {\lambda^{2} \sigma^{2}}/{2}).$
Each item $i \in [L]$ is associated with a $\sigma$-SG
reward distribution $\nu_i$, 
mean $w_i  $, and variance $\sigma_i^2$.
The distributions
$\{\nu_i  \}_{i\in[L]}$, means $\{w_i   \}_{i\in[L]}$, and variances $\{\sigma_i^2   \}_{i\in[L]}$ are unknown to the agent.  
We let $\{ g_{i,t}  \}_{t=1}^T$ be the i.i.d.\ sequence of rewards associated with item $i$ during the $T$ time steps;
each $g_{i,t}$ is an independent sample from~$\nu_i$. 


 {\setlength{\abovedisplayskip}{2pt}
 \setlength{\belowdisplayskip}{2pt}
We focus on stochastic instances with a {\em unique} item having the highest mean reward, and assume that 
$w_1 > w_2  \ge \ldots \ge w_L $, 
 so the unique {\em optimal} item $i^*=1$. 
Note that the items can, in general, be arranged in any order; the ordering that $w_i\ge w_j$ for $i< j$ is employed to ease our discussion.
We denote $\Delta_{1,i} := w_1 -w_i $ as the {\em optimality gap} of item $ i $, and assume $\Delta_{1,i}\le 1$ for all $i\in[L]$; this can be achieved by rescaling the instance if necessary.
We define the {\em minimal optimality gap} 
\begin{align*}
    \Delta:= \min_{i\neq 1} \Delta_{1,i}.
\end{align*} 
Clearly, $\Delta>0$.
We characterize the {\em hardness} of an instance with the following  canonical quantities:
\begin{align}
	H_1
	&
	:=  \sum_{i\ne 1} \frac{ 1 }{\Delta_{1,i}  }\quad\mbox{and}\quad
	H_2
	:= \sum_{i\ne 1} \frac{ 1 }{\Delta_{1,i}^2 }
	.\label{eq:def_H_term} 
\end{align}
The hardness quantity $H_1$ is involved in some near-optimal regret bounds~\citep{auer2002finite,AgrawalG12}.
The quantity $H_2$ was first introduced in \citet{audibert2010best} and appears in many landmark works on BAI~\citep{jamieson2014lil,karnin2013almost,carpentier16tight}.
}

 The agent uses an {\em online algorithm} $\pi$ to decide the item $i_t^{\pi }$ to pull at each time step $t$, and the item $\iout^{ \pi ,T }$ to output eventually.
More formally, an algorithm 
consists of  a tuple $\pi\!:=\!( (\pi_t)_{t=1}^T,   \psi_{T}^{\pi,T} )$, where 
\begin{itemize}[ itemsep = -4pt,   topsep = -2pt, leftmargin =  15pt ]
    \item the {\em sampling rule} $\pi_t$ determines, based on the observation history, the item $i_t^{\pi }$ to pull at time step $t$. That is, the random variable $i_t^{\pi }$ is $\calF_{t-1} $-measurable, where     
    $\calF_t  :=  \sigma ( i_1^{\pi },  g_{ i_1^{\pi },1} , \ldots,  i_t^{\pi },   g_{  i_t^{\pi },t} ) $;
    \item the recommendation rule $ \psi_{T}^{ \pi , T} $ chooses an item $\iout^{\pi, T}$, that is, by definition, $\calF_T $-measurable.
\end{itemize}
%
Moreover, we define the {\em pseudo-regret} $R_T$ of $\pi$
as
\begin{align*}
     R_T (\pi ) :
    & 
    =  \max_{1\le i \le L} \bbE \Bigg[  \sum_{t=1}^T g_{i,t}  \Bigg] - \bbE  \Bigg[  \sum_{t=1}^T    g_{ i_t^{\pi },t} \Bigg]
     = T \cdot w_1 -  \bbE \Bigg[ \sum_{t=1}^T  w_{i_t^\pi} \bigg].
\end{align*}
The algorithm $\pi$ aims to both minimize the pseudo-regret $R_T(\pi )$ and at the same time, to identify the  optimal item with high probability, i.e., to minimize the {\em failure probability} $ e_T(\pi  ) := \Pr  (\iout^{\pi ,T} \ne 1 )$.
We omit $T$ and / or $\pi$ in the superscript or subscript 
when there is no cause of confusion.
We write $R_T(\pi )$ as $R_T(\pi ,\calI  )$, $e_T(\pi )$ as $e_T(\pi ,\calI  )$ when we wish to emphasize their dependence  on both the algorithm $\pi$ and the instance~$\calI$.

\section{Discussion on Existing Algorithms}\label{sec:discussion}

Although there is no existing work that analyzes a single algorithm for both RM and BAI in a fixed horizon,
it is natural to {ask} if an algorithm which is originally designed for RM can also perform well for BAI, and vice versa.
%
%
%
In Table~\ref{tab:comp_result_exist}, we present the theoretical results from some existing works.
We focus on algorithms that are with (potential) theoretical guarantees for both RM and BAI.
We define
 \begin{align*}
     H'_p := \max_{i\ne 1} \frac{ i^p }{ \Delta_i^2 } \;\;
     \text{ and }\;\;
     C_p : = 2^{-p} + \sum_{r=2}^L r^{-p} 
 \end{align*}
for $p>0$ as in \citet{shahrampour2017on}.
We abbreviate 
    {\sc Sequential Halving} as {\sc SH}, 
    {\sc Nonlinear Sequential Elimination} with parameter $p$ as {\sc NSE$(p)$},
    and {\sc UCB-E} with parameter $a$ as {\sc UCB-E$(a)$}.
Also see Appendix~\ref{append:detail_discuss_exist_algo} for more discussions.

 \begin{table}[ht]
    \caption{ Comparison among upper bounds for algorithms and lower bounds in  stochastic bandits.
        }
    \label{tab:comp_result_exist}
    \resizebox{\textwidth}{!}{
    \renewcommand{\arraystretch}{2}
    \begin{tabular}{ l      l  l l }
        Algorithm/Instance &   Pseudo-regret $R_T$ & Failure Probability $e_T$ & References\\
        \thickhline
        \multirow{1}{*}{\sc SH} &  $ {\Theta} ( T)$ & $ \displaystyle  \approx \exp \bigg( -\frac{  T}{  8 H_2 \log_2 L } \bigg)  \vphantom{\Bigg(}$
        & \citet{karnin2013almost}
        \\
        \hline  
        \multirow{1}{*}{\sc NSE$(p)$} &  $  {\Theta} ( T)$  & $ \displaystyle \approx \exp\bigg( -\frac{ 2(T-L) }{ H'_p C_p } \bigg)  \vphantom{\Bigg(}$  
        & \citet{shahrampour2017on} 
        %
        %
        %
        \\
        \hline 
        \multirow{1}{*}{\sc UCB-E$(\alpha\log T)$} 
        & $ 6.3 \alpha^2 H_1\log T  $
        &
        $2 L T^{ 1 - 2\alpha/25}$
        & Corollary~\ref{coro:rm_bai_bd_ucbE_para}, \citet{audibert2010best}
        \\
        \hline 
       \multirow{2}{*}{\sc BoBW-lil'UCB$(\gamma  )$} 
     &  
     $
      H_1  \log T 
    $  (Prob-dep)
    & 
    \multirow{2}{*}{   $
    \approx L 
    \exp \!\bigg( -
    \frac{ T-L  }{ 144 H_2   } 
    \bigg) 
    $  }
    & 
 \multirow{2}{*}{Theorems~\ref{thm:rm_bd_lilUcb_para} and \ref{thm:bai_bd_lilUcb_para}} 
      \\
      & $ 
    \sqrt{TL} \log  T     
    $  (Prob-indep)   & &
    \\
        \hhline{====}
            Stochastic Bandits
          &
            \multirow{2}{*}{$ \approx 4 H_1 \log T  $ }
            &
            \multirow{2}{*}{$   \displaystyle  \frac{1}{6} \exp\bigg( - \frac{ 400 T }{ H_2 \log L } \bigg) \vphantom{\Bigg(}$      }
          &  \multirow{2}{*}{\citet{lai1985asymptotically,carpentier16tight} }
          \vspace*{- .4em}\\
          (Lower Bound) 
          & & &  
    \end{tabular}
    }
    \renewcommand{\arraystretch}{1}
\end{table} 

According to the discussions on RM and BAI in \citet{lattimore2020bandit},
any algorithm with an asymptotically optimal regret would incur a failure probability lower bounded by 
 $\Omega(T^{-1})$;
this  is  much larger than the state-of-the-art lower bound 
$\Omega( \exp(-400T/( H_2\log L) ) )$
by \citet{carpentier16tight}.
Therefore, we only include algorithms that were   designed for BAI in Table~\ref{tab:comp_result_exist}.

Among the various BAI algorithms, {\sc SH} and {\sc NSE$(p)$} perform almost the best. However,  
their bounds on the failure probabilities are incomparable in general.
The comparison among more BAI algorithms is provided in Table~\ref{tab:comp_BAI_fix_budget}.
Due to the designs of {\sc SH} and {\sc NSE$(p)$},
we surmise their regrets   grow linearly with $T$, which is vacuous for the RM task.

Although {\sc UCB-E$(\alpha\log T)$} has upper bounds on both pseudo-regret and failure probability, its bound on the latter, which decays only polynomially fast with $T$ when $\alpha$ is an absolute constant, 
 is clearly suboptimal vis-\`a-vis the state-of-the-art lower bound by \citet{carpentier16tight}.
 In order to achieve an  exponentially decaying upper bound on $e_T$ (i.e.,  $\exp(-  \Theta(T) )$), we need to set $\alpha = O(T/\log T)$, and hence the regret bound (see Corollary~\ref{coro:rm_bai_bd_ucbE_para} in the supplementary) will be $O(T^2/ \log T)$, which is vacuous.
 
The discussion above raises a natural question. 
Is it possible to provide a non-trivial   bound on the regret for an algorithm that performs optimally for BAI over a fixed horizon?
This motivates us to design {\sc BoBW-lil'UCB},
which can be tuned to perform near-optimally for both RM and BAI.


\section{The {\sc BoBW-lil'UCB} Algorithm}
\label{sec:alg_lilUcb}
We design and analyze   {\sc BoBW-lil'UCB$(\gamma)$} ({\sc Best of Both Worlds-Law of Iterated Logs-UCB}), an algorithm for both RM and BAI in a fixed horizon.
By choosing  parameter $\gamma$ judiciously, the   guarantees of {\sc BoBW-lil'UCB$(\gamma)$}    match those of the state-of-the-art algorithms for both RM (up to log factors) and BAI (concerning the exponential term).
 
 \begin{algorithm}[ht]
	\caption{{\sc BoBW-lil'UCB$(\gamma)$} } \label{alg:lilUcb_para} 
		\begin{algorithmic}[1]
			\STATE {\bfseries Input:}  time budget $T$, size of ground set of items $L$, scale $\sigma>0$, $\varepsilon\in (0,1)$, $\beta\ge 0 $, and  $\gamma\in(0,1)$.
			\STATE Sample $i_t = i$ for $t=1, \ldots, L$ and set $t=L$.
			\STATE For all $i \in [L]$, compute $N_{i,L}$, $\hat{g}_{i,L}$, $C_{i,L,\gamma}$, $U_{i,L,\gamma}$:
			    \vspace{-.6em}
			    \begin{align*}
			    	&
			        N_{i,t} =  \sum_{u=1}^t \mathsf{1}\{  i_u = i \},
			        \ 
			        \hat{g}_{i,t} = \frac{ \sum_{u=1}^t g_{i,t} \cdot \mathsf{1}\{  i_u = i \}  }{ N_{i,t} },
                        \\[-.4em]& 
			        C_{ i,t ,\gamma} = 5\sigma  (1 +  \sqrt{\varepsilon} )  \sqrt{ \frac{ 2 (1+\varepsilon)}{N_{i,t }} \cdot  \log \Big(  \frac{  \log( \beta+(1+\varepsilon) N_{i,t } ) }{\gamma} \Big)  },  
			        \quad
			        U_{i,t,\gamma} = \hat{g}_{i,t} + C_{i,t,\gamma}.
			    \end{align*}
			    \vspace{-.5em}
			\FOR{$t =  L+1,  \ldots, T $} 
                \STATE Pull item $i_t = \argmax_{i\in[L]\vphantom{\big]}}  U_{i,t-1,\gamma} $.
                \STATE Update $N_{i_t,t}$, $\hat{g}_{i_t,t}$, $C_{i_t,t,\gamma}$, and $U_{i_t,t,\gamma}$.
			\ENDFOR
			\STATE Output  $\iout = \argmax_{ i\in [L] } \hat{g}_{i,T}  $.\vspace{-.2em}
		\end{algorithmic}
\end{algorithm}   
 {\bf Design of algorithm.}
We design {\sc BoBW-lil'UCB} in the spirit of the law of the iterated logarithm (LIL)~\citep{darling1967iterated,jamieson2014lil}. 
We remark that it is a variation of the {\sc lil'UCB} algorithm proposed by \citet{jamieson2014lil}.
The three differences are:
\begin{itemize}[ itemsep = -1.25pt,   topsep = -2pt, leftmargin =  15pt ]
    \item[(i)]  
    to construct the confidence radius $C_{i,t,\gamma}$, we replace $(1+\beta)$ and $\delta$ in {\sc lil'UCB} by $5$ and $\gamma$ in {\sc BoBW-lil'UCB$(\gamma)$} respectively;  
    \item[(ii)]  in the design of $C_{i,t,\gamma}$, we also replace $\log ( (1+\varepsilon) N_{i,t} )$  by  $\log ( \beta +(1+\varepsilon) N_{i,t} )$;
    \item[(iii)]   {\sc BoBW-lil'UCB$(\gamma)$}, which is designed for both RM and BAI in a fixed horizon,
    involves no stopping rule since it proceeds for {\em exactly} $T$ time steps; while {\sc lil'UCB} is designed for
    BAI with a fixed confidence.
\end{itemize}  
Although our algorithm depends on the choices of $\varepsilon$, $\beta$, and $\gamma$, we term it as {\sc BoBW-lil'UCB$(\gamma)$} instead of the more verbose {\sc BoBW-lil'UCB$(\varepsilon,\beta,\gamma)$} because 
we    scale the confidence radius by only varying $\gamma$ which adjusts the performance of the algorithm.
%
More precisely, inspired by the LIL (see Theorem~\ref{thm:conc_log}), we design item $i$'s confidence radius $C_{i,t,\gamma}$ with $N_{i,t}$ (the number of time steps when item $i$ is pulled up to and including the $t^{\mathrm{th}}$ time step) and $\hat{g}_{i,t}$ (the empirical mean of item $i$ at time step $t$), and its upper confidence bound $U_{i,t,\gamma}$ accordingly.  

The design of {\sc BoBW-lil'UCB$(\gamma)$}
allows us to shrink $C_{i,t,\gamma}$, the confidence radius of each item $i$, by increasing $\gamma$; and vice versa.
Moreover, with a fixed $\gamma$, if item $i$ is rarely pulled in previous time steps, it has a small $N_{i,t}$ and hence a large 
$C_{i,t,\gamma}$; and vice versa.
\begin{itemize}[ itemsep = -3pt,   topsep = -2pt, leftmargin =  15pt ]
    \item[(i)] Therefore, when $\gamma$ increases, the dominant term in $U_{i,t,\gamma} = \hat{g}_{i,t} + C_{i,t,\gamma}$ becomes the empirical mean $\hat{g}_{i,t}$. Since {\sc BoBW-lil'UCB} pulls the item with the largest $U_{i,t-1,\gamma}$ at time step $t$, the algorithm tends to pull the item with the largest empirical mean in this case.
In other words, a large $\gamma$ encourages exploitation.
    \item[(ii)] 
    When $\gamma$ decreases, the confidence radius $C_{i,t,\gamma}$ dominates $U_{i,t,\gamma}$. Consequently, {\sc BoBW-lil'UCB} is likely to pull items with large $C_{i,t,\gamma}$, i.e., the rarely pulled items with small $N_{i,t}$.
This indicates that a small $\gamma$  encourages exploration.
\end{itemize}
%
Altogether, we can scale $U_{i,t,\gamma}$  by adjusting $\gamma$, which allows us to balance exploitation and exploration and trade-off between the twin objectives --- RM and BAI.


{\bf Analysis for RM.}
We first derive problem-dependent and  problem-independent bounds on the pseudo-regret of {\sc BoBW-lil'UCB$(\gamma)$}.

\begin{restatable}[Bounds on the pseudo-regret of {\sc BoBW-lil'UCB}]{theorem}{thmRmBdLilUcbPara}

\label{thm:rm_bd_lilUcb_para} 
%
Let  $\varepsilon\in (0,1)$,
 $\beta\ge 0$, and  $\gamma \in (0, \min\{ \log(   \beta +    1+\varepsilon)/e ,1 \})\vphantom{ \big( }$.
 For all $T\ge 1$, the pseudo-regret of {\sc BoBW-lil'UCB$(\gamma)$} satisfies 
{\setlength{\abovedisplayskip}{4pt}
\setlength{\belowdisplayskip}{\abovedisplayskip}
\begin{align*}
    R_T
    \!\le\! O\bigg( \sigma^2  \cdot  \sum_{i\ne 1} \frac{  \log (  {1}/{\gamma} )  }{  \Delta_{1,i}  } +  2TL  \gamma^{1+\varepsilon}  \bigg)
    ,
    \quad
    R_T
    \!\le\! O \bigg(  \sigma^2   \sqrt{TL} \log \bigg( \frac{ \log(  T/L\gamma  )  }{\gamma }
    \bigg)
    +  2TL  \gamma^{1+\varepsilon} 
    ~\bigg).
\end{align*} 
Furthermore, we can set  $\gamma = (\log T)/ {T}$ to obtain
\begin{align*}
    R_T
    \le O\bigg( \sigma^2    \cdot \sum_{i\ne 1} \frac{  \log T  }{  \Delta_{1,i}  }  \bigg)
    ,
    \quad
    R_T
    \le O \big(  \sigma^2    \sqrt{TL} \log  T     
    ~\big).
\end{align*}%
}\vspace{-1.3em}
\end{restatable}


We observe that the order of the problem-dependent upper bound on the pseudo-regret of {\sc BoBW-lil'UCB($  (\log T)/ {T } 
$)}
 almost  matches that of the lower bound~\citep{lai1985asymptotically}.
Moreover, the worst-case (problem-independent) upper bound of {\sc BoBW-lil'UCB($ (\log T)/ {T} $)} is $\tilde{O}(\sqrt{TL})$, which matches the lower bound $O(\sqrt{TL})$
\citep{bubeck2012regret} up to log factors.
This implies that we can tune the parameter $\gamma$ in the {\sc BoBW-lil'UCB$(\gamma)$} algorithm to obtain close-to-optimal performance for RM.

We remark that when the optimal item is not unique, we can also derive analogous upper bounds on the pseudo-regret of  {\sc BoBW-lil'UCB$(\gamma)$} using a similar line of analysis (see Proposition~\ref{prop:rm_bd_lilUcb_para}).

{\bf Analysis for BAI.}
Next, we upper bound the failure probability of {\sc BoBW-lil'UCB$(\gamma)$}. 

\begin{restatable}[Bounds on the failure probability of {\sc BoBW-lil'UCB}]{theorem}{thmBaiBdLilUcbPara} 

\label{thm:bai_bd_lilUcb_para} 
Let  $\varepsilon\in (0,1)$,   $\beta\ge 0$, and  $\gamma \in (0, \min\{ \log(   \beta +    1+\varepsilon)/e ,1 \} )\vphantom{ \big( }$.
%
Let $\Delta_i = \max\{ \Delta, \Delta_{1,i} \}$ for all $i\in[L]$. 
For all $T\ge 1$, the failure probability of {\sc BoBW-lil'UCB$(\gamma)$} satisfies  
{\setlength{\abovedisplayskip}{4pt}
\setlength{\belowdisplayskip}{\abovedisplayskip}
\begin{align}
    \resizebox{.92\textwidth}{!}{$
    e_T 
    \!\le\! \frac{ 2L(2+\varepsilon)}{ \varepsilon }  \Big( \! \frac{ \gamma }{ \log(1+\varepsilon) } \! \Big)^{1+\varepsilon},\
    \text{if}~~
    \frac{ T-L }{  (1+ \varepsilon )^3 }
    \!\ge\!
    \sum\limits_{i=1}^L \!
    \frac{  72 \sigma^2 }{  \Delta_{ i}^2  } \! \cdot \log \!\Big( \! \frac{2.8}{ \gamma^2 } \log \!\Big( \!  \frac{  11 \sigma  (1+\varepsilon)^2 }{  \Delta_{ i}    }  \!+\! \beta   \Big) \!\Big).
    $}
    \label{eq:lilUcb_bai_err_bd}
\end{align}
In particular, the bound on $e_T$ in \eqref{eq:lilUcb_bai_err_bd} holds when $\gamma \ge \gamma_1( \Delta,H_2 )  $, where
\begin{align*} 
    \gamma_1( \Delta,H_2 ) 
    &
    =
      \sqrt{ 2.8 \log \bigg(   \frac{  6 \sqrt{2.8}  \sigma  (1+\varepsilon)^2 }{  \Delta    }  +\beta   \bigg) }
    \cdot
    \exp \Bigg( -
    \frac{ T-L  }{ 144 \sigma^2 (1+ \varepsilon )^3 (H_2 + \Delta^{-2})  } 
    \Bigg). 
\end{align*} 
 For all $T\ge 1$, when $\gamma $ assumes its lower bound $ \gamma_1(\Delta,H_2)$, we have
\begin{align}  
    \!\! 
    e_T  
    \le 
    \tilde{O} \bigg(\!
    L 
    \exp \!\bigg( \!\!-\!
    \frac{ T-L  }{ 144 \sigma^2 (1+ \varepsilon )^2 (H_2 + \Delta^{-2})  } 
    \bigg) 
    \bigg).\! 
    \label{eq:lilUcb_bai_err_bd_best}
\end{align}%
}\vspace{-0.8em}
\end{restatable}

When $T\gg L$, the gap between our upper bound in~\eqref{eq:lilUcb_bai_err_bd_best} and 
$\Omega( \exp ( - { 400 T }/{ (H_2 \log L) } ) )$,
the state-of-the-art lower bound~\citep{carpentier16tight},
is manifested by the (pre-exponential) term $L$
as well as 
the constant in the exponent. 
This indicates that 
{\sc BoBW-lil'UCB$(\gamma)$} can be adjusted to perform near-optimally for BAI over a fixed horizon.


{\bf Further observation.}
As discussed earlier,
{\sc BoBW-lil'UCB$(\gamma)$} encourages more exploitation than exploration when $\gamma$ is large (e.g. $\gamma =    (\log T)/ {T}  $)
and it stimulates more exploration when $\gamma$ is small (e.g. $\gamma =\gamma_1(\Delta,H_2)$).
Besides, Theorems~\ref{thm:rm_bd_lilUcb_para} and \ref{thm:bai_bd_lilUcb_para} imply that
the pseudo-regret of {\sc BoBW-lil'UCB$(\gamma)$} decreases with $\gamma$ while its failure probability increases with $\gamma$.
Therefore, to minimize the regret, we should increase $\gamma$ to stimulate exploitation; and we should decrease $\gamma$ to encourage exploration for obtaining a small failure probability.
 %
This indicates that an optimal RM algorithm  encourages more exploitation compared to an optimal BAI one, 
and vice versa.

\section{Pareto Frontier of RM and BAI}
 
 \label{sec:sto_rm_bai}

Theorems~\ref{thm:rm_bd_lilUcb_para} and \ref{thm:bai_bd_lilUcb_para}  together suggest that 
 {\sc BoBW-lil'UCB$(\gamma)$}   cannot perform optimally for both
RM and BAI simultaneously with  a universal (or single) choice of $\gamma$.
 %
In this section, we 
prove that no algorithm can perform optimally for these two objectives simultaneously. 
Given a certain failure probability of an algorithm,
our goal is to establish a non-trivial lower bound on its pseudo-regret. 
 
 We first consider bandit instances in which items have bounded rewards.
Let $\calB_1( \DeltaMin, \RMax  )$ denote the set of stochastic instances where
(i) the minimal optimality gap $\Delta\ge \DeltaMin$;
and (ii) there exists $R_0\in \bbR$ such that the rewards are bounded in $[R_0,R_0+\RMax]$.
Let $\calB_2(\DeltaMin, \RMax ,\HtwoMax) $ denote the set of instances 
that
(i) belong to $\calB_1(\DeltaMin, \RMax)$,
and
(ii)
have hardness quantities $H_2\le \HtwoMax$. 

%
 
\begin{restatable}{theorem}{thmStoRmBaiBern}
\label{thm:sto_rm_bai_bern} 
Let $\phi_{T}, \DeltaMin, \RMax, \HtwoMax >0$.
%
Let $\pi$ be any algorithm  with $e_T(\pi,\calI) \le   \exp( -  \phi_{T} )/4$ for all $\calI \in \calB_1( \DeltaMin, \RMax  )$.
%
Then
{\setlength{\abovedisplayskip}{0pt}
\setlength{\belowdisplayskip}{\abovedisplayskip}
\begin{align*}
	\sup_{ \calI \in \calB_1( \DeltaMin, \RMax  ) }
	R_T (\pi, \calI )
	\ge 
	 \phi_{T}   \cdot  \frac{  (L-1) \RMax   }{ 8 \DeltaMin  },
	\qquad 
    \sup_{ \calI \in \calB_2( \DeltaMin, \RMax, \HtwoMax ) }
	R_T (\pi,\calI)
	\ge  
     \phi_{T}   \cdot  \frac{ \DeltaMin \HtwoMax \RMax^3}{8}.
\end{align*}%
}\vspace{-0.4em}
%
\end{restatable} 

In Theorem~\ref{thm:sto_rm_bai_bern}, we apply the bounds $R_0$ and $R_0+\RMax$ on items' rewards to classify instances.
In general, $\Delta H_2\le (L-1)/\Delta$ holds for any instance, and equality holds when
 $\Delta_{1,i}=\Delta $ for all $i\ne 1$.
 Therefore, $\calB_1(\DeltaMin, \RMax ) = \calB_2(\DeltaMin, \RMax, (L-1)/(\DeltaMin^2))  $.
When $\RMax> 1$,  
the analysis for the set $\calB_2(\DeltaMin, \RMax, (L-1)/(\DeltaMin^2))$ provides a better bound (higher lower bound) for the set $\calB_1(\DeltaMin, \RMax  )$.

Our {\bf non-asymptotic} bounds complete the asymptotic observation on the trade-off between regret and BAI in  \citet{lattimore2020bandit} and  {\bf quantify} how the trade-off 
depends on the hardness quantity $H_2$ and gap $\Delta_{1,i}$'s of an instance,
instead of fixed constants such as $\varepsilon$. 
Another relevant work is \citet{bubeck2009pure}.
%
On one hand, \citet{bubeck2009pure} explores the trade-off between the cumulative regret $R_T$ and the simple regret $r_T$ and shows that any algorithm with $R_T\le C \psi(T)$ satisfies $r_T\ge { \DeltaMin } \exp(-D\psi(T))/{2}$ in some instance.
On the other hand, our work studies the Pareto frontier of $R_T$ and the BAI failure probability $e_T$. 
Since $e_T$ and $r_T$ satisfy that $\DeltaMin \cdot e_T \le r_T \le e_T$,
our Theorem~\ref{thm:sto_rm_bai_bern}  indicates that  
\begin{restatable}{corollary}{coroStoRmBaiBern}
\label{coro:sto_rm_bai_bern_to_simple_reg} 
Let $\phi_{T}, \DeltaMin >0$.
%
Let $\pi$ be any algorithm satisfying  
\begin{align*}
	\sup_{ \calI \in \calB_1( \DeltaMin, 1  ) }
	R_T (\pi, \calI )
	\le 
	 \phi_T \cdot \frac{ L-1}{ 8\DeltaMin},
\end{align*}%
then $e_T \ge \exp(-\phi_T)/4$ and $r_T  \ge  \DeltaMin \exp(-\phi_T)/4$.
\end{restatable}
In view of Corollary \ref{coro:sto_rm_bai_bern_to_simple_reg},   we have precisely {\em quantified} 
that
$C =  (L - 1)/(8\DeltaMin)$ and $D = 1$ in the work of \citet{bubeck2009pure}.


Furthermore, we establish a similar analysis  for instances in which the variance of each item's reward distribution is bounded.
Let $\calB_1'(\DeltaMin, \varMax )$ denote the set of   instances where 
(i) the minimal optimality gap $\Delta\ge \DeltaMin $;
(ii) for each item $i$, the variance $\sigma_i^2\le \varMax$. 
Let $\calB_2'(\DeltaMin, \varMax,\HtwoMax)$ denote the set of  instances (i) that belong to  $\calB_1'(\DeltaMin, \varMax )$, and (ii) 
have hardness quantities $H_2\le \HtwoMax$.  
%
%
The key difference between the proofs of these two theorems lies in  the design of hard instances. We elaborate on the details in 
Appendix~\ref{append:sto_analyze_trade_off}.

\begin{restatable}{theorem}{thmStoRmBaiGauss}
\label{thm:sto_rm_bai_gauss} 
Let $\phi_{T}, \DeltaMin, \varMax,\HtwoMax>0$.
%
Let $\pi$ be any algorithm  with $e_T(\pi,\calI) \le   \exp( -  \phi_{T} )/4$ for all $\calI \in \calB_1'( \DeltaMin ,\varMax )$.
%
Then
{\setlength{\abovedisplayskip}{1pt}
\setlength{\belowdisplayskip}{2pt}
\begin{align*}
	\sup_{ \calI \in \calB_1'( \DeltaMin, \varMax ) }
	R_T (\pi, \calI )
	\ge 
	 \phi_{T}   \cdot  \frac{ (L-1)  \varMax }{2 \DeltaMin } ,
	\qquad 
    \sup_{ \calI \in \calB_2'( \DeltaMin, \varMax,\HtwoMax ) }
	R_T (\pi, \calI)
	\ge  
     \phi_{T}   \cdot   \frac{ \DeltaMin \HtwoMax \varMax   }{2  }.
\end{align*}%
}\vspace{-1.1em}
%
\end{restatable}

By characterizing stochastic rewards with different statistics, 
Theorems~\ref{thm:sto_rm_bai_bern} and \ref{thm:sto_rm_bai_gauss} provide different lower bounds on the pseudo-regret.
We observe that when the rewards of items are bounded in $[R_0,R_0+\RMax]$ for some $R_0\in \bbR$, the variances of the rewards are bounded by $ \RMax^2/4$.
Therefore,
{\setlength{\abovedisplayskip}{1pt}
\setlength{\belowdisplayskip}{\abovedisplayskip}
\begin{align*} 
    \calB_1( \DeltaMin, \RMax ) 
    \subset \calB_1'\bigg(\DeltaMin, \frac{\RMax^2}{4} \bigg),
    \qquad 
    \calB_2(\DeltaMin, \RMax, \HtwoMax) 
    \subset \calB_2' \bigg(\DeltaMin,  \frac{\RMax^2}{4}, \HtwoMax \bigg).
\end{align*}    
Besides, it is clear that  
\begin{align*} 
    \calB_1(\DeltaMin,\RMax), \ 
    \calB_2(\DeltaMin,\RMax,h),\ 
     \calB_2' \bigg(\DeltaMin, \frac{\RMax^2}{4}, h \bigg) 
    \subset \calB_1'\bigg(\DeltaMin, \frac{\RMax^2}{4} \bigg).
\end{align*}
Due to the relationship among these four sets of instances, 
we let $\pi$ be an algorithm with $e_T (\pi,\calI) \le \exp(-\phi_{T} )/4 $ in {\em any} instance of $\calB_1' ( \DeltaMin,  {\RMax^2}/{4}  )$,
and compare the derived lower bounds on its pseudo-regret $ R_T (\pi,\calI)$ in Table~\ref{tab:comp_lower_bd}.%
}
Table~\ref{tab:comp_lower_bd} indicates that 
\begin{itemize}[ itemsep = -4pt,   topsep = -1pt, leftmargin =  15pt ]
    \item when the bound for $\calB_1(\DeltaMin,   \RMax   )$ (
     second column  of Table~\ref{tab:comp_lower_bd}) holds for $\calB_1'( \DeltaMin,   \RMax^2/4  )$, 
    the quantities $L$ and $\DeltaMin$ are of the same order in the bounds derived for $\calB_1(\DeltaMin,   \RMax  )$ and $\calB_1'( \DeltaMin,   \RMax^2/4 )$ respectively;
    \item similarly, when the bound for $\calB_2(\DeltaMin,   \RMax , \HtwoMax)$ (
 third column) holds for $\calB_2'(\DeltaMin,  {\RMax^2}/{4}, \HtwoMax)$, 
the quantities $L$ and $\DeltaMin$ are of the same order in the bounds for $\calB_2(\DeltaMin,   \RMax , \HtwoMax)$ and $\calB_2'(\DeltaMin,  {\RMax^2}/{4}, \HtwoMax)$.
%
\end{itemize}
\begin{table}[ht]
    \caption{Lower bounds on 
    $R_T$ when 
    $e_T \le \mathrm{e}^{-\phi_{T} }/4 $. 
    }
    \label{tab:comp_lower_bd}
    \resizebox{\textwidth}{!}{
    \renewcommand{\arraystretch}{1.5}
\centering
    \begin{tabular}{ l " l | l | l  | l }
        Instance Set &  
        $\calB_1(\DeltaMin,  \RMax )$ &
        $\calB_2(\DeltaMin,   \RMax , \HtwoMax)$ &
        $  \calB_1' ( \DeltaMin,  {\RMax^2}/{4} )$ &
        $   \calB_2' (\DeltaMin,  {\RMax^2}/{4}, \HtwoMax  ) $
        \\
        \hline
        Bound on $R_T$  
        &
    $     \phi_{T}   \cdot {  (L\!-\!1)\RMax    }{/( 8\DeltaMin)  }  $ &
    $    \phi_{T}   \cdot  { \DeltaMin \HtwoMax \RMax^{3\vphantom{\hat{3}}}   }/{ 8  }  $ &
    $     \phi_{T}   \cdot  {  (L\!-\!1)\RMax^2    }/{ (8\DeltaMin )  }  $ &
    $     \phi_{T}   \cdot  { \DeltaMin \HtwoMax \RMax^2   }/{ 8  }  $
    \end{tabular}%
    }
    \renewcommand{\arraystretch}{1}
\end{table}
Moreover,
when $\RMax>1$, we can apply the analysis of $\calB_2(\DeltaMin,  \RMax , \HtwoMax)$ to obtain a better bound (higher lower bound) for $\calB_2'(\DeltaMin,  {\RMax^2}/{4}, \HtwoMax)$.

In {\em any} set of instances studied in Theorems~\ref{thm:sto_rm_bai_bern} or \ref{thm:sto_rm_bai_gauss}, 
\begin{itemize}[ itemsep = -4pt,   topsep = -1.2pt, leftmargin =  15pt ]
    \item when $\phi_{T}$ linearly grows with $T$, which is typical in the bounds on $e_T$~\citep{karnin2013almost,carpentier16tight}, the corresponding bound on $R_T$
    grows linearly with $T$ (vacuous);
    \item  when the bound on $R_T$ grows with $\log T$ as in~\citet{Garivier2011the} and \citet{lai1985asymptotically}, $\phi_{T}$ grows logarithmically with $T$ (i.e.,  the failure probability only decays polynomially).
\end{itemize}
Thus, we cannot achieve   optimal performances for both RM and BAI using any algorithm with fixed  parameters. 
Alternatively, we can apply {\sc BoBW-lil'UCB$(\gamma)$} to achieve the best of both objectives with proper choices of the single parameter $\gamma$.  


\textbf{Tightness of the upper and lower bounds.}
We compare the upper and lower bounds on the 
pseudo-regret of 
{\sc BoBW-lil'UCB($\gamma$)}
when the horizon $T\to\infty$.

\begin{restatable}{corollary}{coroRmLowerBdBernLilUcbPara}
\label{coro:rm_lower_bd_bern_lilUcb_para}  
Define the interval $\mathcal{I} (\nu, T) = [ \gamma_1(\DeltaMin,\HtwoMax), \min \{ \log(   \beta +    1+\varepsilon)/e , (\log T)/T,1/L\}  ] $, which is a function of the instance $\nu$ and the fixed horizon $T$.
When $\mathcal{I} (\nu, T) \neq \emptyset$,%
 let $\pi_0$ denote the online algorithm {\sc BoBW-lil'UCB$ ( \gamma )$} with $\gamma$ satisfying the condition that  $\gamma\in \mathcal{I} (\nu, T)$.
Then
{\setlength{\abovedisplayskip}{0.8pt}
\setlength{\belowdisplayskip}{2pt}
\begin{align*}
    &
    \sup_{ \calI \in \calB_2( \DeltaMin, 1, \HtwoMax ) } R_T( \pi_0 , \calI) 
    \in
    \Omega 
    \bigg(
      \DeltaMin \HtwoMax   \log \bigg( \frac{1}{\gamma L} \bigg) 
    \bigg) 
    \bigcap  
    O \bigg( \frac{   L   }{ \DeltaMin }  \log \bigg( \frac{1}{\gamma  } \bigg)  \bigg).
\end{align*}%
}\vspace{-1.3em} 
\end{restatable}
 We observe from Corollary~\ref{coro:rm_lower_bd_bern_lilUcb_para},%
    \footnote{$\HtwoMax $ has the same units as $L \cdot \DeltaMin^{-2}$; hence, the terms in $\Omega(\cdot)$ and $O(\cdot)$ also have the same units.}
     which combines Theorems~\ref{thm:rm_bd_lilUcb_para}, \ref{thm:bai_bd_lilUcb_para}, and~\ref{thm:sto_rm_bai_bern},  that
 the gap between the upper and lower bounds depend on the term  $\DeltaMin \HtwoMax$ 
    in the lower bound
    and $L/\DeltaMin$ in the upper bound. 
    As $H_2 \le (L-1)\Delta^{-2}$ for any instance, 
    when $\Delta = \Delta_{1,i}$ for all $i\ne 1$ (all suboptimal items have the same suboptimality gap),   equality holds, and hence the bounds  match up to a small additive $\log(1/L)$ term.   
Corollary~\ref{coro:rm_lower_bd_bern_lilUcb_para} implies that the parameter $\gamma$ in {\sc BOBW-lil-UCB$(\gamma)$} is essential in tuning the 
algorithm such that it can perform optimally for either RM or BAI. 
 This implies that in some regimes, {\sc BOBW-lil-UCB$(\gamma)$} achieves Pareto-optimality up to constant or small   additive (e.g., $\log(1/L)$) terms. 
%

Besides, to show that there are cases for which $\mathcal{I}(\nu, T)\ne\emptyset$, we use several examples here.  In the instance where $L=256$, $w_1=0.5$, $w_i=0.45$ for $i\ne 1$ and rewards are drawn from Bernoulli distributions, if we let $\varepsilon=0.01$, $\beta = e$, $\mathcal{I} (\nu, T)$ is always non-empty when time horizon $T\ge 10^6$ as shown in Table \ref{tab:range_gamma_differ_T}.
\vspace*{-.5em}
\begin{table}[htbp]
  \centering
  \caption{Intervals $\mathcal{I} (\nu, T)$ under different time horizons $T$}
    \begin{tabular}{l | ll}
    Time horizon $T$ & $\mathcal{I} (\nu, T)$ &  \\
    \thickhline
    $10^6 \vphantom{\Big(} $ & $[1.85\times 10^{-7},$ & $ 1.38\times 10^{-5}]$ \\
    \hline
    $10^7 \vphantom{\Big(} $ & $[4.33\times 10^{-73},$ & $1.61\times 10^{-6}]$ \\
    \hline
    $10^8 \vphantom{\Big(} $ & $[0,$ & $1.84 \times 10^{-7}]$ \\
    \hline
   $10^9 \vphantom{\Big(} $ & $[0,$ & $2.07\times 10^{-8}]$ \\
    \end{tabular}%
  \label{tab:range_gamma_differ_T}%
\end{table}

Furthermore, 
Corollary~\ref{coro:rm_lower_bd_bern_lilUcb_para} 
suggests that the  lower bound in Theorem~\ref{thm:sto_rm_bai_bern} is almost tight, as it is
achieved by   {\sc BoBW-lil'UCB$(\gamma)$}.
Hence, up to terms logarithmic in the parameters such as $L$, we have quantified the Pareto frontier for the trade-off between RM and BAI in stochastic bandits.

\section{Numerical Experiments}
\label{sec:experiment}
We numerically compare {\sc BoBW-lil'UCB$(\gamma)$} and {\sc UCB$_\alpha$} as they are the only algorithms that can be tuned to perform (near-)optimally for both RM and BAI.
Since   {\sc BoBW-lil'UCB$(\gamma)$}   is designed for the fixed-budget setting and  {\sc UCB$_\alpha$} is for the fixed-confidence setting, there cannot be a completely fair comparison between them.
However, we attempt to perform  fair comparisons as much as possible.

We evaluate the algorithms with both synthetic and real data.
For {\sc BoBW-lil'UCB$(\gamma)$}, we fix $\varepsilon=0. 01$,  $\beta=e$,  and vary $\gamma$.
For {\sc UCB$_\alpha$}, we 
 vary $\alpha$.
%
%
%
We run {\sc BoBW-lil'UCB$(\gamma)$} for $T$ (fixed in a specific instance) time steps,
when the horizon (stopping time) of {\sc UCB$_\alpha$} depends on its stopping rule and the instance.
Due to the difference between the fixed-horizon and fixed-confidence settings,
 the regrets of each algorithm may be accumulated over different time horizons.
 
For each choice of algorithm and instance, we run $10^4$ independent trials.
%
Since the empirical failure probability of {\sc BoBW-lil'UCB$(\gamma)$} is below $1\%$ in each instance  (see Table  \ref{tab:err_prob_err001}  in Appendix \ref{append:experiment_table_emp_failure}),
we set $\delta=0.01$ for {\sc UCB$_\alpha$}, which guarantees that the failure probability of {\sc UCB$_\alpha$} is also below $1\%$. 
We also present the experiment results with empirical failure probabilities below $2\%$ in Appendix \ref{append:extra_experiment} (where we set $\delta=0.02$ for {\sc UCB$_\alpha$}).
We
focus on the comparison on
(i) the time horizon each algorithm runs; 
 and (ii) the  regret incurred over  its corresponding horizon. 
We present the averages and standard deviations of the time horizons and \text{red}{the} regrets of each algorithm.
More  numerical results 
that reinforce the conclusions  herein are presented in   Appendix~\ref{append:extra_experiment}.

\subsection{Experiments using synthetic data}

\label{sec:experiment_syn}
 We set
  $  w_1 = 0.5 $,
 and
 $ w_i = 0.5- \Delta   $ for all $ i \ne 1$.
 We let $\mathrm{Bern}(a)$ denote the Bernoulli distribution with parameter $a$.
 We consider Bernoulli bandits, i.e., $\nu_i = \mathrm{Bern}(w_i)$.
We display some numerical results in Figure~\ref{pic:bern_L64_wGapMin0_05_0_1_err001} ; more results are postponed to Appendix~\ref{append:extra_experiment_syn_data}. 

\begin{figure}[ht]
    \begin{flushright}
        \includegraphics[width=.3\textwidth]{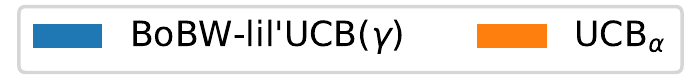}      
        \hspace{.4em}
        \vspace{-.3em}
    \end{flushright}
	\includegraphics[width=.48\textwidth]{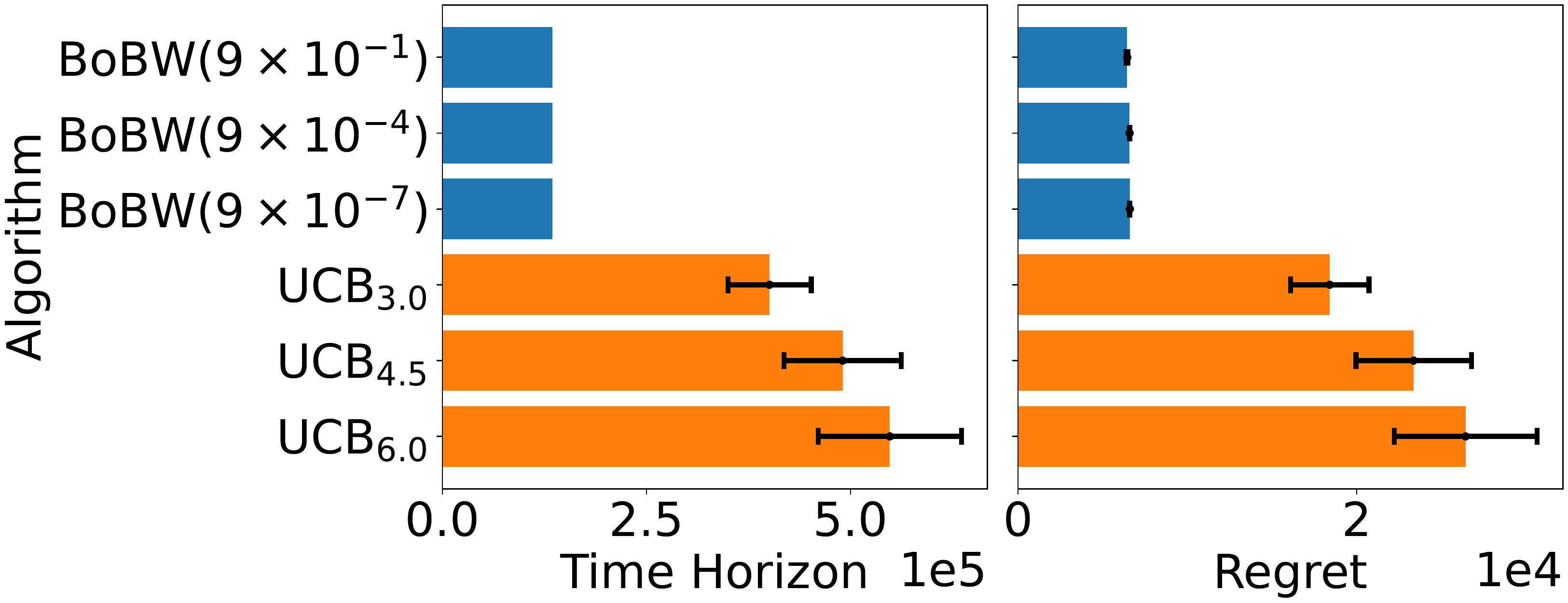}
	\hspace{.8em}
	\includegraphics[width=.48\textwidth]{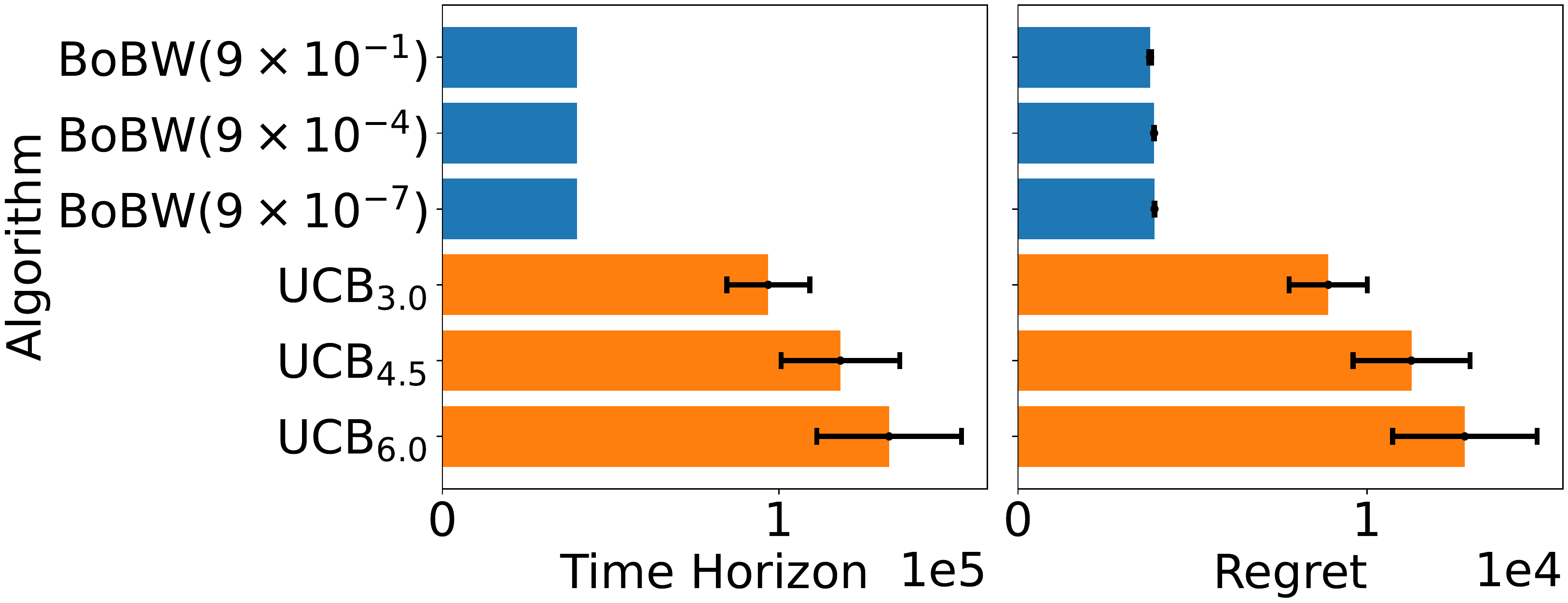}\\		
	\caption{Bernoulli instances with $L\!=\!64$,  failure probability $\!\le \!1\%$. Left: $ \Delta\! =\!0.05$; right:  $\Delta\! =\!0.1$.} 
	\label{pic:bern_L64_wGapMin0_05_0_1_err001}   
\end{figure}

Under each instance presented in Figure~\ref{pic:bern_L64_wGapMin0_05_0_1_err001},
the regret of {\sc BoBW-lil'UCB$(\gamma)$} is reduced when $\gamma$ grows (see Table~\ref{tab:experiment_reg_syn} for exact values), which corroborates with     Theorem~\ref{thm:rm_bd_lilUcb_para}.
Both the regret and the stopping time of {\sc UCB$_\alpha$} grow  with $\alpha$, which corroborates with \citet[Theorem 3]{pmlr-v89-degenne19a}.
Moreover, we observe that the standard deviations of the regrets are larger for {\sc UCB$_\alpha$}  compared to {\sc BoBW-lil'UCB$(\gamma)$}, which suggests that {\sc BoBW-lil'UCB$(\gamma)$} is more statistically robust and consistent in terms of the regret.
Note that
a larger $\Delta$ means that the difference between the optimal and suboptimal items is
more pronounced, 
resulting in 
an easier instance.
Given a fixed horizon $T$, our {\sc BoBW-lil'UCB$(\gamma)$} algorithm outperforms the  {\sc UCB$_\alpha$} algorithm with a varying range of parameters $\gamma$ and $\alpha$ in instances with different values of $\Delta$.

\subsection{Experiments on real datasets}
\label{sec:experiment_real}
We use two real-world datasets, the {\em MovieLens 25M} (ML-25M) dataset  
\citep{harper2015movielens}
  and the {\em Published Kinase Inhibitor Set 2} (PKIS2) dataset  
\citep{drewry2017progress},
 to evaluate the performances of {\sc BoBW-lil'UCB$(\gamma)$} and {\sc UCB$_\alpha$} in two types of practical applications, namely, content recommendation and drug recovery.
Similarly as in \citet{ZongNSNWK16,hong2020latent,jmlrsubmit,mason2020finding, mukherjee2021mean}, we generate data based on the real-world datasets.

{\bf ML-25M dataset.} 
GroupLens Research provides a collection of datasets online,\footnote{%
\url{https://grouplens.org/datasets/movielens}} %
 including the ML-25M dataset. 
These datasets describe the rating activities from MovieLens, a movie recommendation service, and are widely used to evaluate the performances of bandit algorithms~\citep{ZongNSNWK16,hong2020latent,jmlrsubmit}. 
The ML-25M dataset contains about $25$ million ratings across about $62$ thousand movies.
We choose movies with a high number of ratings in our simulations.
For each selected movie, we compute the empirical mean rating and generate random ratings according to a standard Gaussian distribution with the corresponding mean.
We aim to obtain cumulatively high ratings (RM) and identify the movie with the highest rating (BAI); these are standard objectives in online recommendation systems.


{\bf PKIS2 dataset.}
This repository%
 \footnote{%
 Table 4 in \url{https://www.biorxiv.org/content/10.1101/104711v1.supplementary-material}.}  
  tests 
$641$
 small molecule compounds (kinase inhibitor)  against 
 $406$
 protein kinases.
This experiment aims to find the most effective inihibitor against a targeted kinase, and is a fundamental study in cancer drug discovery. 
The entries in PKIS2 indicate the \emph{percentage control} of each inhibitor,
which show the effectiveness of inhibitors and follow log-normal distributions~\citep{christmann2016unprecedently}.
Accordingly, we generate random variables as in \citet{mason2020finding, mukherjee2021mean} (see Appendix~\ref{append:extra_experiment_real_data} for details). 
We aim to find out the most effective inhibitor with the highest percentage control against  one specific kinase MAPKAPK5, and also obtain high percentage controls cumulatively during the online learning process.
Our study may aid in understanding how best to design experimental studies that aim to identify the most effective inhibitor in a fixed number of tests (BAI in a fixed horizon), as well as to provide effective inhibitors throughout the course of study (RM).



\begin{figure}[th]
	\centering
	\includegraphics[width=.48\textwidth]{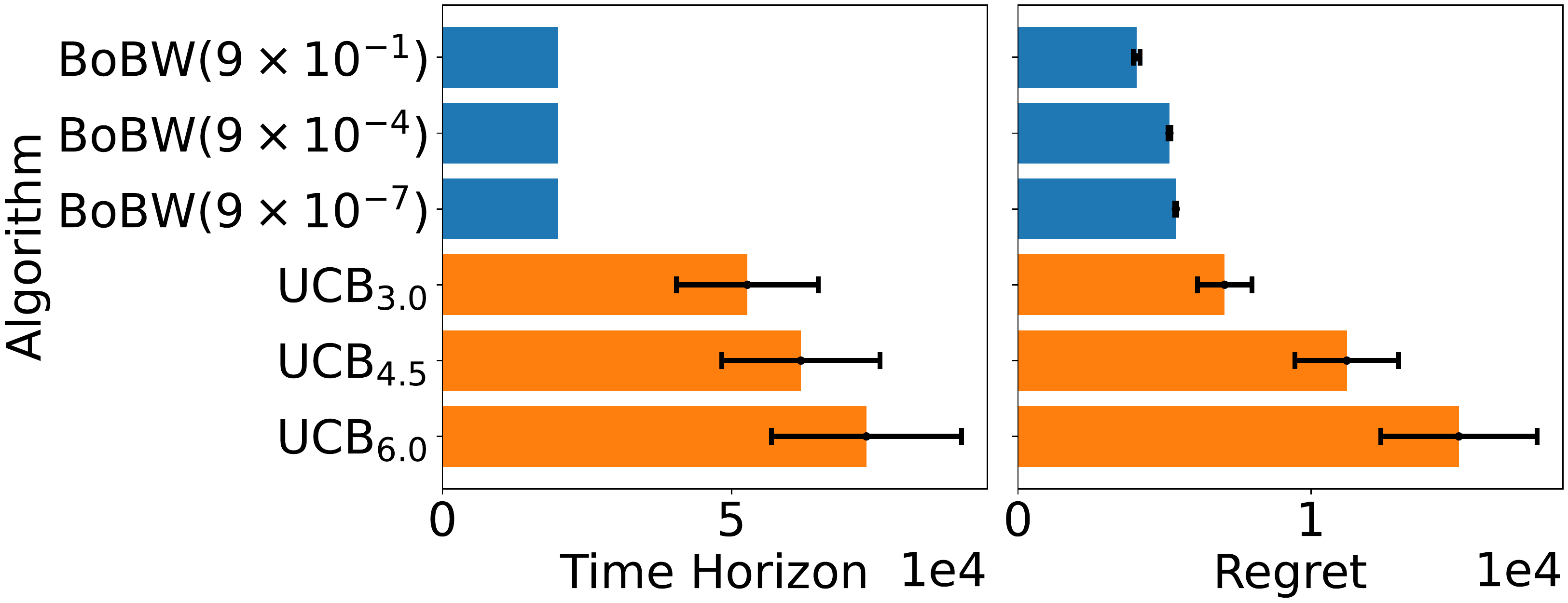} 
	\hspace{.5em}
	\includegraphics[width=.48\textwidth]{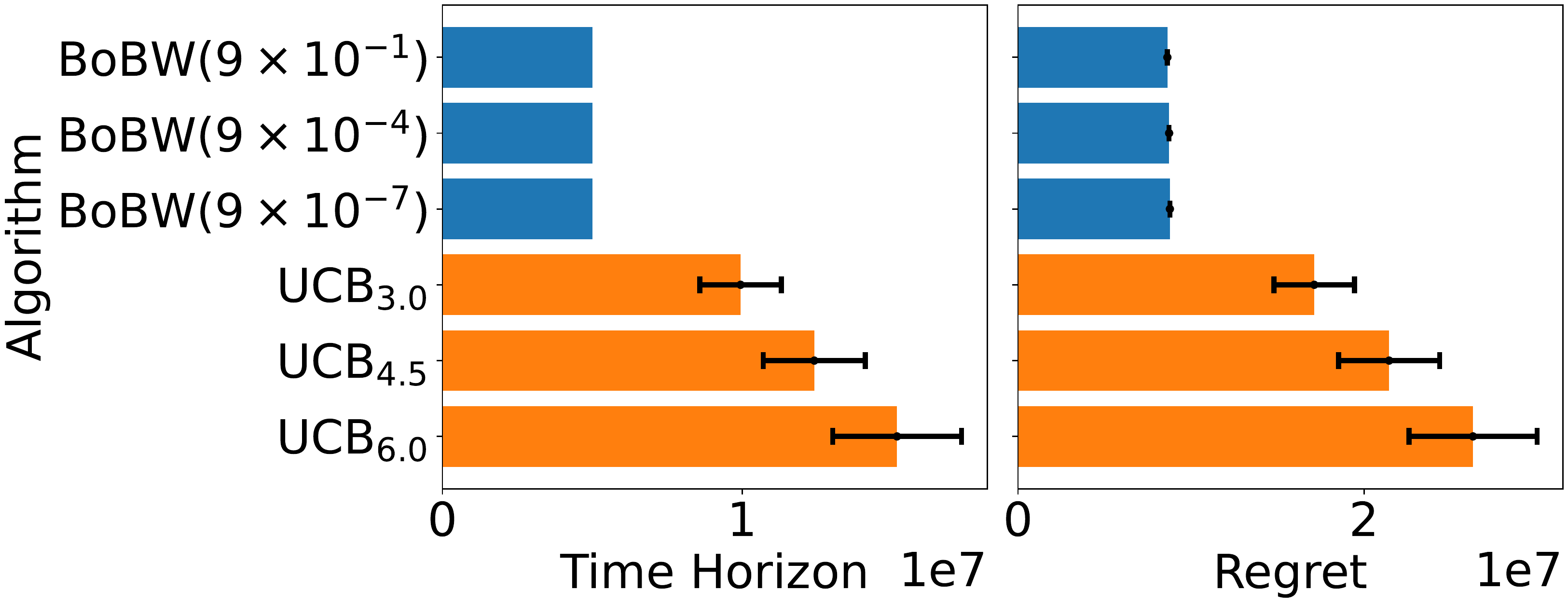} 
	\\
		\caption{Empirical failure probability $\le 1\%$. Left: ML-25M dataset;
	right: PKIS2 dataset.}
	\label{pic:ml25m_rating50_pkis2_MAPKAPK5_err001}   
\end{figure}

On the left of Figure~\ref{pic:ml25m_rating50_pkis2_MAPKAPK5_err001},
we report the results of the experiments on the $22$ movies with at least $50,000$ ratings from the ML-25M dataset.
The other plot in Figure~\ref{pic:ml25m_rating50_pkis2_MAPKAPK5_err001}
 considers the effectiveness of $109$ inhibitors test against the MAPKAPK5 kinase in the PKIS2 dataset.
%
%
%
%
%
Both figures suggest that with high probability,
{\sc BoBW-lil'UCB$(\gamma)$} can identify the most popular movie with the highest rating or the most effective inhibitor against  MAPKAPK5 with the highest percentage control within a fixed horizon. {\sc UCB$_\alpha$}  takes longer to do so, and  
also
suffers from a larger regret. 
These results from the real-life datasets suggest that given a fixed horizon $T$ and a wide range of parameters, {\sc BoBW-lil'UCB$(\gamma)$}
 outperforms {\sc UCB$_\alpha$} in 
these 
real-life instances,
which demonstrates the potential of {\sc BoBW-lil'UCB$(\gamma)$}    in practical settings.


%
%

\bibliography{trade_off_rm_bai_ref}

\begin{thebibliography}{}

\bibitem[Abbasi-Yadkori et~al., 2018]{abbasi2018best}
Abbasi-Yadkori, Y., Bartlett, P., Gabillon, V., Malek, A., and Valko, M.
  (2018).
\newblock {Best of both worlds: Stochastic \& adversarial best-arm
  identification}.
\newblock In {\em Proceedings of the 31st Conference On Learning Theory},
  volume~75 of {\em Proceedings of Machine Learning Research}, pages 918--949.
  PMLR.

\bibitem[Abramowitz and Stegun, 1964]{AbramowitzS64}
Abramowitz, M. and Stegun, I.~A. (1964).
\newblock {\em Handbook of Mathematical Functions with Formulas, Graphs, and
  Mathematical Tables}.
\newblock Dover, New York, ninth edition.

\bibitem[Agrawal and Goyal, 2012]{AgrawalG12}
Agrawal, S. and Goyal, N. (2012).
\newblock Analysis of {Thompson} sampling for the multi-armed bandit problem.
\newblock In {\em Proceedings of the 25th Annual Conference on Learning
  Theory}, pages 39.1--39.26.

\bibitem[Agrawal and Goyal, 2013]{AgrawalG13b}
Agrawal, S. and Goyal, N. (2013).
\newblock Thompson sampling for contextual bandits with linear payoffs.
\newblock In {\em Proceedings of the 30th International Conference on Machine
  Learning}, volume~28, pages 127--135.

\bibitem[Agrawal and Goyal, 2017]{agrawal2017near}
Agrawal, S. and Goyal, N. (2017).
\newblock Near-optimal regret bounds for {T}hompson sampling.
\newblock {\em Journal of the ACM (JACM)}, 64(5):1--24.

\bibitem[Audibert and Bubeck, 2010]{audibert2010best}
Audibert, J.-Y. and Bubeck, S. (2010).
\newblock Best arm identification in multi-armed bandits.
\newblock In {\em Proceedings of the 23th Conference on Learning Theory}, pages
  41--53.

\bibitem[Auer et~al., 2002a]{auer2002finite}
Auer, P., Cesa-Bianchi, N., and Fischer, P. (2002a).
\newblock Finite-time analysis of the multiarmed bandit problem.
\newblock {\em Machine Learning}, 47(2-3):235--256.

\bibitem[Auer et~al., 2002b]{Auer02}
Auer, P., Cesa-Bianchi, N., Freund, Y., and Schapire, R.~E. (2002b).
\newblock The nonstochastic multiarmed bandit problem.
\newblock {\em SIAM Journal of Computing}, 32(1):48--77.

\bibitem[Bubeck et~al., 2012]{bubeck2012regret}
Bubeck, S., Cesa-Bianchi, N., et~al. (2012).
\newblock Regret analysis of stochastic and nonstochastic multi-armed bandit
  problems.
\newblock {\em Foundations and Trends{\textregistered} in Machine Learning},
  5(1):1--122.

\bibitem[Bubeck et~al., 2009]{bubeck2009pure}
Bubeck, S., Munos, R., and Stoltz, G. (2009).
\newblock Pure exploration in multi-armed bandits problems.
\newblock In {\em The 20st International conference on Algorithmic learning
  theory}, pages 23--37. Springer.

\bibitem[Bubeck et~al., 2013]{bubeck2013multiple}
Bubeck, S., Wang, T., and Viswanathan, N. (2013).
\newblock Multiple identifications in multi-armed bandits.
\newblock In {\em Proceedings of the 30th International Conference on Machine
  Learning}, pages 258--265.

\bibitem[Carpentier and Locatelli, 2016]{carpentier16tight}
Carpentier, A. and Locatelli, A. (2016).
\newblock Tight (lower) bounds for the fixed budget best arm identification
  bandit problem.
\newblock In Feldman, V., Rakhlin, A., and Shamir, O., editors, {\em 29th
  Annual Conference on Learning Theory}, volume~49 of {\em Proceedings of
  Machine Learning Research}, pages 590--604, Columbia University, New York,
  New York, USA. PMLR.

\bibitem[Christmann-Franck et~al., 2016]{christmann2016unprecedently}
Christmann-Franck, S., van Westen, G.~J., Papadatos, G., Beltran~Escudie, F.,
  Roberts, A., Overington, J.~P., and Domine, D. (2016).
\newblock Unprecedently large-scale kinase inhibitor set enabling the accurate
  prediction of compound--kinase activities: A way toward selective promiscuity
  by design?
\newblock {\em Journal of chemical information and modeling}, 56(9):1654--1675.

\bibitem[Chung and Lu, 2006]{chung2006concentration}
Chung, F. and Lu, L. (2006).
\newblock Concentration inequalities and martingale inequalities: a survey.
\newblock {\em Internet Mathematics}, 3(1):79--127.

\bibitem[Darling and Robbins, 1967]{darling1967iterated}
Darling, D.~A. and Robbins, H. (1967).
\newblock Iterated logarithm inequalities.
\newblock {\em Proceedings of the National Academy of Sciences of the United
  States of America}, 57(5):1188.

\bibitem[Degenne et~al., 2019]{pmlr-v89-degenne19a}
Degenne, R., Nedelec, T., Calauzenes, C., and Perchet, V. (2019).
\newblock Bridging the gap between regret minimization and best arm
  identification, with application to {A/B} tests.
\newblock In {\em Proceedings of the 22nd International Conference on
  Artificial Intelligence and Statistics}, pages 1988--1996.

\bibitem[Drewry et~al., 2017]{drewry2017progress}
Drewry, D.~H., Wells, C.~I., Andrews, D.~M., Angell, R., Al-Ali, H., Axtman,
  A.~D., Capuzzi, S.~J., Elkins, J.~M., Ettmayer, P., Frederiksen, M., et~al.
  (2017).
\newblock Progress towards a public chemogenomic set for protein kinases and a
  call for contributions.
\newblock {\em PloS one}, 12(8):e0181585.

\bibitem[Dubhashi and Panconesi, 2009]{dubhashi2009concentration}
Dubhashi, D.~P. and Panconesi, A. (2009).
\newblock {\em Concentration of measure for the analysis of randomized
  algorithms}.
\newblock Cambridge University Press.

\bibitem[Gabillon et~al., 2012]{gabillon2012best}
Gabillon, V., Ghavamzadeh, M., and Lazaric, A. (2012).
\newblock Best arm identification: A unified approach to fixed budget and fixed
  confidence.
\newblock In {\em Proceedings of the 25th International Conference on Neural
  Information Processing Systems}.

\bibitem[Garivier and Cappé, 2011]{Garivier2011the}
Garivier, A. and Cappé, O. (2011).
\newblock The kl-ucb algorithm for bounded stochastic bandits and beyond.
\newblock In {\em Proceedings of the 24th Annual Conference on Learning
  Theory}, pages 359--376.

\bibitem[Gerchinovitz and Lattimore, 2016]{gerchinovitz2016refined}
Gerchinovitz, S. and Lattimore, T. (2016).
\newblock Refined lower bounds for adversarial bandits.
\newblock In {\em Proceedings of the 29th International Conference on Neural
  Information Processing Systems}, pages 1198--1206.

\bibitem[Götze et~al., 2019]{friedrich2019higher}
Götze, F., Sambale, H., and Sinulis, A. (2019).
\newblock {Higher order concentration for functions of weakly dependent random
  variables}.
\newblock {\em Electronic Journal of Probability}, 24:1 -- 19.

\bibitem[Harper and Konstan, 2015]{harper2015movielens}
Harper, F.~M. and Konstan, J.~A. (2015).
\newblock The movielens datasets: History and context.
\newblock {\em Acm transactions on interactive intelligent systems (tiis)},
  5(4):1--19.

\bibitem[Hong et~al., 2020]{hong2020latent}
Hong, J., Kveton, B., Zaheer, M., Chow, Y., Ahmed, A., and Boutilier, C.
  (2020).
\newblock Latent bandits revisited.
\newblock In Larochelle, H., Ranzato, M., Hadsell, R., Balcan, M.~F., and Lin,
  H., editors, {\em Proceedings of the 34th Advances in Neural Information
  Processing Systems}, volume~33, pages 13423--13433. Curran Associates, Inc.

\bibitem[Jamieson et~al., 2014]{jamieson2014lil}
Jamieson, K., Malloy, M., Nowak, R., and Bubeck, S. (2014).
\newblock lil’ucb: An optimal exploration algorithm for multi-armed bandits.
\newblock In {\em Proceedings of the 27th Conference on Learning Theory}, pages
  423--439.

\bibitem[Karnin et~al., 2013]{karnin2013almost}
Karnin, Z., Koren, T., and Somekh, O. (2013).
\newblock Almost optimal exploration in multi-armed bandits.
\newblock In {\em Proceedings of the 13th International Conference on Machine
  Learning}, pages 1238--1246.

\bibitem[Kaufmann et~al., 2016]{kaufmann2016complexity}
Kaufmann, E., Capp{\'e}, O., and Garivier, A. (2016).
\newblock On the complexity of best-arm identification in multi-armed bandit
  models.
\newblock {\em The Journal of Machine Learning Research}, 17(1):1--42.

\bibitem[Kaufmann and Garivier, 2017]{kaufmann2017learning}
Kaufmann, E. and Garivier, A. (2017).
\newblock Learning the distribution with largest mean: two bandit frameworks.
\newblock {\em ESAIM: Proceedings and surveys}, 60:114--131.

\bibitem[Kaufmann and Kalyanakrishnan, 2013]{kaufmann13information}
Kaufmann, E. and Kalyanakrishnan, S. (2013).
\newblock Information complexity in bandit subset selection.
\newblock In {\em Proceedings of the 26th Annual Conference on Learning
  Theory}, pages 228--251.

\bibitem[Kim et~al., 2023]{kim2023pareto}
Kim, W., Iyengar, G., and Zeevi, A. (2023).
\newblock Pareto front identification with regret minimization.

\bibitem[Lai and Robbins, 1985]{lai1985asymptotically}
Lai, T.~L. and Robbins, H. (1985).
\newblock Asymptotically efficient adaptive allocation rules.
\newblock {\em Advances in applied mathematics}, 6(1):4--22.

\bibitem[Lattimore and Szepesv{\'a}ri, 2020]{lattimore2020bandit}
Lattimore, T. and Szepesv{\'a}ri, C. (2020).
\newblock {\em Bandit algorithms}.
\newblock Cambridge University Press.

\bibitem[Mason et~al., 2020]{mason2020finding}
Mason, B., Jain, L., Tripathy, A., and Nowak, R. (2020).
\newblock Finding all $\epsilon$-good arms in stochastic bandits.
\newblock In Larochelle, H., Ranzato, M., Hadsell, R., Balcan, M.~F., and Lin,
  H., editors, {\em Proceedings of the 34th Advances in Neural Information
  Processing Systems}, volume~33, pages 20707--20718. Curran Associates, Inc.

\bibitem[Mukherjee et~al., 2021]{mukherjee2021mean}
Mukherjee, A., Tajer, A., Chen, P.-Y., and Das, P. (2021).
\newblock Mean-based best arm identification in stochastic bandits under reward
  contamination.
\newblock In {\em Proceedings of the 35th Advances in Neural Information
  Processing Systems}.

\bibitem[Russo and Van~Roy, 2014]{RussoV14}
Russo, D. and Van~Roy, B. (2014).
\newblock Learning to optimize via posterior sampling.
\newblock {\em Mathematics of Operations Research}, 39(4):1221--1243.

\bibitem[Shahrampour et~al., 2017]{shahrampour2017on}
Shahrampour, S., Noshad, M., and Tarokh, V. (2017).
\newblock On sequential elimination algorithms for best-arm identification in
  multi-armed bandits.
\newblock {\em IEEE Transactions on Signal Processing}, 65(16):4281--4292.

\bibitem[Simchi-Levi and Wang, 2023]{simchi2023multi}
Simchi-Levi, D. and Wang, C. (2023).
\newblock Multi-armed bandit experimental design: Online decision-making and
  adaptive inference.
\newblock In Ruiz, F., Dy, J., and van~de Meent, J.-W., editors, {\em
  Proceedings of the 26th International Conference on Artificial Intelligence
  and Statistics}, volume 206 of {\em Proceedings of Machine Learning
  Research}, pages 3086--3097. PMLR.

\bibitem[Tsybakov, 2008]{tsybakov2008introduction}
Tsybakov, A.~B. (2008).
\newblock {\em Introduction to Nonparametric Estimation}.
\newblock Springer Publishing Company, Incorporated, 1st edition.

\bibitem[Wei and Luo, 2018]{wei2018more}
Wei, C.-Y. and Luo, H. (2018).
\newblock More adaptive algorithms for adversarial bandits.
\newblock In {\em Proceedings of the 31st Conference on Learning Theory}, pages
  1263--1291. PMLR.

\bibitem[Zhong et~al., 2021a]{zhong2021probabilistic}
Zhong, Z., Cheung, W.~C., and Tan, V. (2021a).
\newblock Probabilistic sequential shrinking: {A} best arm identification
  algorithm for stochastic bandits with corruptions.
\newblock In {\em Proceedings of the 38th International Conference on Machine
  Learning}.

\bibitem[Zhong et~al., 2021b]{jmlrsubmit}
Zhong, Z., Cheung, W.~C., and Tan, V. (2021b).
\newblock Thompson sampling for cascading bandits.
\newblock {\em \textrm{Accepted by} Journal of Machine Learning Research}.

\bibitem[Zimmert and Seldin, 2021]{zimmert2021tsallis}
Zimmert, J. and Seldin, Y. (2021).
\newblock {Tsallis-INF}: An optimal algorithm for stochastic and adversarial
  bandits.
\newblock {\em Journal of Machine Learning Research}, 22:28--1.

\bibitem[Zong et~al., 2016]{ZongNSNWK16}
Zong, S., Ni, H., Sung, K., Ke, N.~R., Wen, Z., and Kveton, B. (2016).
\newblock Cascading bandits for large-scale recommendation problems.
\newblock In {\em Proceedings of the 32nd Conference on Uncertainty in
  Artificial Intelligence}, pages 835--844.

\end{thebibliography}
\bibliographystyle{apalike}

\newpage
\appendix

\begin{center}
	    {\LARGE {\bf Supplementary Material for \\ ``Achieving the Pareto Frontier of Regret Minimization\\ \vphantom{\big(} and Best Arm Identification in   Multi-Armed Bandits''}}
	\end{center}
	In Appendix~\ref{append:detail_discuss_exist_algo}, we discuss the existing algorithms and relevant theoretical findings in stochastic bandits.
	In Appendix~\ref{sec:discuss_adv},	
	we 
	(i) study the performance of {\sc UP-ADV}~\citep{audibert2010best} and {\sc Exp3.P}~\citep{Auer02} for both RM and BAI in  adversarial bandits,
	and
	(ii) provide a lower bound on the BAI failure probability and the Pareto frontier of RM and BAI in  adversarial bandits.
	In Appendix~\ref{sec:useful_fact}, we list the useful facts that are used in the analysis.
	In Appendices~\ref{append:sto_analyze_lilUcb} to \ref{sec:pf_adv_lb_trade_off}, we present detailed proofs of our theoretical results.
	In Appendix~\ref{append:extra_experiment}, more numerical results are provided.

\section{Detailed discussion on existing algorithms}

\label{append:detail_discuss_exist_algo}

While most existing works only aim to perform either RM or BAI,
\citet{pmlr-v89-degenne19a} 
designed and analyzed an algorithm called {\sc UCB$_\alpha$}  for both RM and BAI under the {\em fixed-confidence} setting.
Given any $\delta$, {\sc UCB$_\alpha$} aims to minimize the number of time steps $\tau$ so that $e_\tau  \leq \delta$, and, at the same time, the incurred regret $R_\tau$ can also be upper bounded. 
Therefore, the focus of \citet{pmlr-v89-degenne19a} differs from that of our work.
 We aim to study 
the pseudo-regret of an algorithm which can identify the best item with high probability in a {\em fixed horizon} $T$ in this work.  

To the best of our knowledge, their is no existing work that analyzes a single algorithm for both RM and BAI under the fixed-budget setting.
However, it is natural to question if an algorithm which is originally designed for RM can also perform well for BAI, and vice versa. 
We study some algorithms that are originally designed to achieve optimal performance for either RM or BAI. 

{\bf RM.}
According to the discussions on RM and BAI in \citet{lattimore2020bandit} (see the second point in Note 33.3),
for any algorithm with a regret that (nearly) matches the state-of-the-art lower bound~\citep{carpentier16tight}:
\begin{align}
    \liminf_{T\rightarrow \infty} \frac{ R_T(\pi)  }{ \log T }
    \ge 
    \sum_{i\ne 1} \frac{  \Delta_{1,i}  }{  \rmKL ( \nu_i \| \nu_1 ) },
    \nonumber 
\end{align}
we can construct two instances $\calI$ and $\calI'$  with
\begin{align*}
    w_1^\calI  > w_2^\calI  \ge \ldots \ge  w_L^\calI,\qquad
    w_i^{\calI'} = w_i^\calI  + \varepsilon (w_{1 }^\calI -w_{i }^\calI  ),\mbox{ for some }
    \varepsilon>0 
\end{align*} 
such that 
\begin{align}
    e_T ( \pi, \calI ) + e_T ( \pi, \calI' )
    \ge \Omega( T^{  (1+o(1) )(1+\varepsilon)^2 } ).
    \label{eq:sto_rm_bai_coarse_lower_bd}
\end{align}
This serves as a basic observation on the limitation for BAI of an algorithm that performs (near-)optimally for RM.

{\bf BAI.}
\citet{audibert2010best} were the first to explore the BAI problem under the fixed-budget setting.
\citet{carpentier16tight} provided a lower bound 
on the failure probability of any algorithm.
%
%
 
 In the spirit of {\sc UCB1}~\citep{auer2002finite},
\citet{audibert2010best} designed  {\sc UCB-E} for BAI.
We let {\sc UCB-E$(a)$} denote the UCB-E algorithm when it is run with parameter $a$.
 When $T$ is sufficiently large,
we can upper bound the pseudo-regret of {\sc UCB-E$(\alpha\log T)$}~$(\alpha\ge 2)$ 
with a similar analysis as that for  {\sc UCB1} (see Proof of Theorem 1 in \citet{auer2002finite}).
Besides, 
we can upper bound its failure probability with Theorem 1 in \citet{audibert2010best}.
 
\begin{restatable}{corollary}{coroRmBaiUcbEPara}
\label{coro:rm_bai_bd_ucbE_para}
Let $ \alpha > 12.5 $.
Assume that $g_{i,t}\in[0,1]$ for all $i\in[L]$, 
 and $ \alpha \log T \le 25(T-L)/(36H_2)$. {\sc UCB-E$(\alpha\log T)$} satisfies 
\begin{align} 
    R_T 
    &
    \le 2 \alpha^2 \sum_{i\ne 1} \bigg( \frac{\log T}{ \Delta_{1,i} } \bigg)
    + \bigg( 1+ \frac{\pi^2}{3}  \bigg) \cdot \bigg( \sum_{i \ne 1 } \Delta_{1,i}  \bigg),  \nonumber \\
    \nonumber
    \quad  
    e_T
    &
    \le 
      2L T^{(1-2\alpha/25)}
     .
     \nonumber
\end{align}
\end{restatable}
When the horizon $T$ grows, Corollary~\ref{coro:rm_bai_bd_ucbE_para} indicates that the BAI failure probability of {\sc UCB-E$(\alpha\log T)$} decays only polynomially fast.
In order to achieve the upper bound on $e_T$ as $\exp(-  \Theta(T) )$, we need to set $\alpha = O(T/\log T)$, and hence the regret bound as shown in Corollary~\ref{coro:rm_bai_bd_ucbE_para} will be $O(T^2/ \log T)$, which is vacuous.

%

\subsection{Existing results under the fixed-budget setting of BAI}

We abbreviate {\sc Sequential Rejects} as {\sc SR}, {\sc Sequential Halving} as {\sc SH}, {\sc Nonlinear Sequential Elimination} with parameter $p$ as {\sc NSE$(p)$}.
Besides, 
we simplify the bounds for algorithms which were initially analyzed for more general problems than identification of the optimal item $i^*$.
we define
 \begin{align*}
     H'_p := \max_{i\ne 1} \frac{ i^p }{ \Delta_i^2 },
     \quad
     C_p : = 2^{-p} + \sum_{i=2}^L i^{-p} 
 \end{align*}
for $p>0$ as in \citet{shahrampour2017on}.
We let {\sc UGapEb$(a)$} denote the {\sc UGapEb} algorithm when it is run with parameter $a$.
In Table~\ref{tab:comp_BAI_fix_budget}, We present existing bounds from some seminal works. The algorithms are listed in chronological order.

 \begin{table}[ht]
    \caption{Comparison under the fixed-budget setting of BAI: upper bounds for algorithms and lower bounds in stochastic bandits.
    }
    \label{tab:comp_BAI_fix_budget}
    \vspace{1em}
    \resizebox{\textwidth}{!}{
    \centering
    \renewcommand{\arraystretch}{2.5}
    \begin{tabular}{ l      l  l}
        Algorithm/Instance &   Reference & Failure probability $e_T$ \\
        \thickhline
        {\sc UCB-E$ \displaystyle \bigg(\frac{25(T-L)}{36H_2}\bigg)$} &  \citet{audibert2010best}   & $ \displaystyle 2TL\exp \bigg( -\frac{  T-L }{  18 H_2 } \bigg) \vphantom{\Bigg(} $
        \\ 
        \hline
        {\sc SR}
        & \citet{audibert2010best}   & $ \displaystyle L(L-1)\exp \bigg( -\frac{  T-L }{  ( 1/2 + \sum_{i=2}^L 1/i ) H_2 } \bigg) \vphantom{\Bigg(_a} $
        \\ 
        \hline
        {\sc UGapEb$\displaystyle \bigg(  \frac{ T-L }{ 16H_2 } \bigg)$} &  \citet{gabillon2012best}  & $  \displaystyle 2TL  \exp\bigg(-  \frac{ T-L }{ 8H_2 } \bigg) \vphantom{\Bigg(}$
        \\  
        \hline  
         {\sc SAR} & \citet{bubeck2013multiple}   & $\displaystyle 2L^2 \exp \bigg( -\frac{  T-L }{  8( 1/2 + \sum_{i=2}^L 1/i ) H_2 } \bigg) \vphantom{\Bigg(_a} $
        \\  
        \hline 
        {\sc SH} &    \citet{karnin2013almost}  & $\displaystyle 3\log_2 L \cdot \exp \bigg( -\frac{  T}{  8 H_2 \log_2 L } \bigg) \vphantom{\Bigg(} $ 
        \\  
        \hline 
        {\sc NSE$(p)$} &  \citet{shahrampour2017on}  & $ \displaystyle (L-1)  \exp\bigg( -\frac{ 2(T-L) }{ H'_p C_p } \bigg) \vphantom{\Bigg(}  $ 
        \\ 
        \hhline{===} 
          Stochastic Bandits 
          &
            \citet{carpentier16tight} 
            &
            $  \displaystyle \frac{1}{6} \exp\bigg( - \frac{ 400 T }{ H_2 \log L } \bigg)$  \quad    (Lower Bound) 
    \end{tabular}
    }
    \renewcommand{\arraystretch}{1}
\end{table} 

Since SH and NSE pull a number of items ``uniformly'' in each phase, we surmise the regret grows like $\Theta(T)$. 
For instance, there are $\log L$ many phases in SH and at least two items are uniformly pulled during each phase, so at least one non-optimal item $j\ne 1$ is pulled for at least $ T / (L \log L)  $ times, leading to a regret at least $\Delta_{1,j} \cdot T /(L \log L)$.

 As discussed in \citet{shahrampour2017on}, $H_p' C_p \le H_2\log L$ in some special cases.
Therefore, {\sc SH} is better than {\sc NSE$(p)$} if we disregard the sub-exponential term,
while {\sc NSE$(p)$} is better in some cases in its dependence on   the exponential term.
However, they are incomparable in general.


\section{Conclusion and Further Discussion on Adversarial Bandits}
\label{sec:discuss_adv}

In Sections~\ref{sec:sto_rm_bai} and \ref{sec:experiment}, we explore the Pareto frontier of RM and BAI over a fixed horizon in stochastic bandits.
The performance of our {\sc BoBW-lil'UCB} algorithm sheds light on the different emphases of RM and BAI. 
Moreover, we prove that no algorithm can simultaneously perform optimally for both objectives
and {\sc BoBW-lil'UCB} nearly achieves the Pareto-optimality in some parameter regimes. 
However, as described in the discussion after Corollary~\ref{coro:rm_lower_bd_bern_lilUcb_para}, although our {\sc BoBW-lil'UCB} algorithm nearly achieves the Pareto frontier, we acknowledge that there remains a small gap $\log(1/L)$ which may be closed in the future by developing another more sophisticated algorithm.
 
 In real-life applications, it may be unrealistic to  assume   i.i.d.\ stochastic rewards,
 meaning that 
 the stochastic bandit model may not be appropriate.
 This brings the  study of {\em adversarial bandits}~\citep{Auer02,abbasi2018best} to the fore. Here, the rewards of each item are not necessarily drawn independently from the same distribution.
 In adversarial bandits, while there exists a lower bound on the regret of any algorithm~\citep{gerchinovitz2016refined},
 there is no lower bound on the failure probability for BAI.
 We fill this gap by proving a lower bound $\Omega( \exp( -150T\underline{\Delta}^2 ) )$  in Theorem~\ref{thm:adv_bai_lb}, where $\underline{\Delta}$ is the minimal gap between the empirically-optimal items and the other items (see Appendix~\ref{sec:adv_prob_set} for the definitions). This bound is almost tight as it nearly matches the upper bound of {\sc UP-ADV} \citep{abbasi2018best}.
 
 Furthermore, there is no existing analysis of a {\em single} algorithm that is applicable to both RM and BAI in adversarial bandits.
 We fill this gap by studying the performance of {\sc Exp3.P$(\gamma,\eta)$} \citep{Auer02} for both targets.
 Theorems~\ref{thm:rm_bd_expThreeP_beta_zero} and \ref{thm:bd_bai_expThreeP_zero} imply that by adjusting {\sc Exp3.P$(\gamma,\eta)$} with $\gamma$, we can balance between exploitation and exploration, and trade-off
between the twin objectives: RM and BAI. 
 Besides, Theorem~\ref{thm:adv_rm_bai_gauss} implies that no algorithm can simultaneously perform optimally for both objectives in adversarial bandits.
 However, since the regret bound of {\sc Exp3.P$(\gamma,\eta)$} is problem-independent, we cannot ascertain if {\sc Exp3.P$(\gamma,\eta)$} achieves the Pareto frontier between RM and BAI.
 The further study of the Pareto frontier in  adversarial bandits, especially the stochastically constrained adversarial bandits~\citep{zimmert2021tsallis,wei2018more},
 may serve as an interesting direction for future work.%
 


{\bf Outline.} In this section, we first formulate the RM and BAI problem in adversarial bandits in Appendix~\ref{sec:adv_prob_set}.
Next, we study the performance of {\sc UP-ADV}~\citep{audibert2010best} and {\sc Exp3.P}~\citep{Auer02} for both RM and BAI in Appendix~\ref{sec:adv_up_expThreeP}.
Subsequently, we provide a lower bound on the BAI failure probability and the Pareto frontier of RM and BAI in Appendix~\ref{sec:adv_lb_bai_trade_off_rm_bai}.
We summarize some theoretical findings in Table~\ref{tab:comp_result_adv}.
%

\subsection{Problem setup for adversarial bandits}
\label{sec:adv_prob_set}
In an adversarial bandit instance, we let $g_{i,t}\in [0,1]$ be the reward of item $i$ at time $t$, and let $ G_{i,t} := \sum_{u=1}^t g_{i,u} $ for all $1\le t\le T$.
We define the $\barDelta_{ i,j,T}$, \emph{empirical gap} between item $i$ and $j$ in $[L]$ and the {\em empirically-optimal} item $\bari^{* }_{T} $ as follows:
\begin{align*}
    \barDelta_{ i,j,T}  : = \frac{1}{T} \cdot (G_{i,t}   - G_{j,t} ),
    \quad
    \bari^{*}_{T}:= \argmax_{1\le i \le L } G_{i,T}.
\end{align*}
Moreover, we define the {\em empirically-minimal optimality gap} as
\begin{align*}
    \barDelta_{ T} := \min_{j \ne \bari_T^{* } } \barDelta_{ \bari_T^{* },j } .
\end{align*}
%
%
We say an instance  is {\em obliviously adversarial}%
    \footnote{When there is no ambiguity, we say an instance is adversarial to indicate that it is obliviously adversarial.}%
    , if $\{ \bm{g}_{i,t}  \}_{i,t} $ is a sequence of rewards obliviously generated by the instance before online process.  
    We assume the empirically-optimal item $\bari_T^{* }$ is {\em unique}, which implies that $\barDelta_{\min,T}  >0$.

Moreover, we define the {\em empirical-regret} $\barR_T^{\pi }$ 
of an online algorithm $\pi$ (as defined in Section~\ref{sec:sto_prob_set}) as
\begin{align*}
    &
    \barR_T^{\pi } :=  \max_{1\le i \le L} \sum_{t=1}^T g_{i,t} - \sum_{t=1}^T  g_{ i_t^{\pi },t} 
     = G_{ \bari^{* }_{T} , T }^\calI - \sum_{t=1}^T  g_{ i_t^{\pi },t}.
\end{align*}
Recall the definition of pseudo-regret $R_T(\pi)$ in Section~\ref{sec:sto_prob_set}: if an instance is stochastic, $ \bbE \barR_T^{\pi } = R_T^{\pi }$;
 if it is adversarial, $ \bbE \barR_T^{\pi } \le R_T^{\pi }$. 
The aim of the agent is slightly different in   stochastic and adversarial bandits:
\begin{itemize}[ itemsep = -4pt,   topsep = 8pt, leftmargin =  15pt ]
    \item if the instance is stochastic, the algorithm $\pi$ aims to both minimize the pseudo-regret $R_T(\pi )$ and identify the {pseudo-optimal} item with high probability, i.e., to minimize $ e_T (\pi ):= \Pr (\iout^{\pi ,T} \ne  i^{* }_T )$;
    \item if the instance is adversarial, the algorithm $\pi$ aims to both minimize the empirical-regret $\barR_T (\pi )$ and identify the  {empirically-optimal} item with high probability, i.e., to minimize $\bare_T(\pi ) :=\Pr  (\iout^{\pi, T} \ne \bari^{* }_T )$.
\end{itemize}
%
%
We omit $T$ and/or $\pi$ in the
superscript or subscript when there is no cause of confusion.
We write $\barR_T (\pi)$ as $\barR_T (\pi, \calI)$, $\bare_T (\pi)$ as $\bare_T (\pi, \calI)$ when we
wish to emphasize their dependence on both the algorithm
$\pi$ and the instance $\calI$.

\subsection{Adversarial algorithms: {\sc UP-ADV} and {\sc Exp3.P}}
\label{sec:adv_up_expThreeP}
We discuss the theoretical performances of two basic algorithms in this section.

Hence, we take the performance of this basic algorithm as a benchmark to evaluate any algorithm for this target.
Besides, it is clearly that the uniform pull algorithm is the same as Exp3.P algorithm with $\beta=0$, $\gamma=1$. We see that UP-ADV satisfies that $\bbE \barR_T \le T$, which is consistent with Theorem~\ref{thm:bd_bai_expThreeP_zero}.

{\bf The {\sc UP-ADV} algorithm.}
First of all, \citet{abbasi2018best} shows that a simple algorithm, which is termed as {\sc UP-ADV} and chooses an item based on the uniform distribution at each time step $t$, 
satisfies that
\begin{align}
        \bare_T
        \le 
        L \exp \bigg(  - \frac{ 3T  \barDelta_T^2 }{ 28 L  } \bigg).
        \label{eq:bai_bd_up_adv}
    \end{align} 
\citet{abbasi2018best} claimed that {\sc UP-ADV} performs near-optimally for BAI in  adversarial bandits, which is verified by our Theorem~\ref{thm:adv_bai_lb} in the next section.
Besides, it is obvious that {\sc UP-ADV} satisfies $\bbE \barR_T\le T$.
\begin{restatable}{algorithm}{algUpAdv} 
	\caption{{\sc Uniform Pull-ADV} (UP-ADV) \citep{abbasi2018best}} \label{alg:up_adv}
		\begin{algorithmic}[1]
			\STATE {\bfseries Input:}  time budget $T$, size of ground set of items $L$.
			\FOR{$t =  1,  \ldots, T $} 
                \STATE Choose item $i_t\in[L]$ with probability $1/L$.
                \STATE Update the estimated cumulative gain $ \tilde{G}_{i,t}  = \sum_{u=1}^t  g_{i,u} \cdot \mathbb{I} \{ i_u = i \} $. 
			\ENDFOR
			\STATE Output  $\iout = \argmax_{ i\in [L] } \tilde{G}_{i,T}  $.
		\end{algorithmic}
\end{restatable}

{\bf The {\sc Exp3.P} algorithm.}
After the {\sc Exp3} algorithm and its variations were proposed by \citet{Auer02} for RM in  adversarial bandits, this class of algorithms has been widely discussed as in
\citet{lattimore2020bandit,bubeck2012regret}.
We present {\sc Exp3.P$(\gamma, \eta)$} in Algorithm~\ref{alg:exp3p}.
%

\begin{restatable}{algorithm}{algExpThreeP} 
	\caption{{\sc Exp3.P$(\gamma, \eta)$} (\citet{bubeck2012regret}, Section 3.3, Fig. 3.1)} \label{alg:exp3p}
		\begin{algorithmic}[1]
			\STATE {\bfseries Input:}  time budget $T$, size of ground set of items $L$, parameters $\eta >0$ and $\gamma\in [0,1] $.   
			\STATE Set $p_1$ be the uniform distribution over $[L]$, i.e., $p_{i,1}=1/L \ \forall i\in[L]$. 
			\FOR{$t =  1,  \ldots, T $} 
                \STATE Choose item $i_t\in[L]$ with probability $p_{i,1}$.
                \STATE Compute the estimated gain for each item 
                    \begin{align*}
                        \tilde{g}_{i,t} = \frac{  g_{i,t} \cdot \mathbb{I} \{ i_t = i \}
                        }{  p_{i,t} }
                    \end{align*}
                    and update the estimated cumulative gain $ \tilde{G}_{i,t}  = \sum_{u=1}^t \tilde{g}_{i,u}  $.
			    \STATE Compute the new probability distribution over the items
			        $p_{t+1} = (p_{1,t+1}, \ldots , p_{L,t+1} )$ where
			        \begin{align*}
			            p_{i,t+1}  = ( 1- \gamma ) \cdot \frac{  \exp( \eta \tilde{G}_{i,t}  ) }{  \sum_{\ell=1}^L \exp( \eta \tilde{G}_{\ell,t}  )   } + \frac{ \gamma }{  L }.
			        \end{align*} 
			\ENDFOR
			\STATE Output  $\iout = \argmax_{ i\in [L] } \tilde{G}_{i,T}  $.
		\end{algorithmic}
\end{restatable}
We first provide the upper bound on the regret of {\sc Exp3.P$(\gamma, \eta)$}. The proof is similar to that in \citet{bubeck2012regret} and is postponed to Appendix~\ref{pf:thm_rm_bd_expThreeP_beta_zero}
\begin{restatable}[Bounds on the regret of {\sc Exp3.P$(\gamma, \eta)$}]
{theorem}{thmRmBdExpThreePBetaZero} 
    \label{thm:rm_bd_expThreeP_beta_zero}
    Let $\eta >0$, $\gamma\in[0,1/2]$  satisfying that $ L\eta \le \gamma$.
    Then we can upper bound the regret of {\sc Exp3.P$(\gamma, \eta)$} as follows.
    (i) Fix any given $\delta\in (0,1)$, with probability at least $1-\delta$, 
\begin{align*}
    \barR_T  \le  \gamma T +  \eta L T +  \ln  \bigg(   \frac{    L^2 T }{  \eta \delta } \bigg) +  \frac{ \ln L }{ \eta}.
\end{align*}
(ii) Moreover, 
\begin{align*}
    \bbE \barR_T \le     \gamma T + \eta L T +  \ln  \bigg(   \frac{    L^2 T }{  \eta   } \bigg) +  \frac{ \ln L }{ \eta}  + 1.
\end{align*} 
   
\end{restatable}	
 We observe that {\sc Exp3.P$(1, \eta)$} is exactly the same as {\sc UP-ADV}, and the corresponding bound provided in Theorem~\ref{thm:bd_bai_expThreeP_zero} is with the same order as in \eqref{eq:bai_bd_up_adv} derived by \citet{abbasi2018best}.
 Our upper bound is even slightly smaller regarding the constants since we apply tighter concentration inequalities.
Next, we upper bound its failure probability to identify the empirically-optimal item $\bari_T^*$.
\begin{restatable}[Bound on the failure probability of {\sc Exp3.P}]{theorem}{thmBdBaiExpThreePZero}
\label{thm:bd_bai_expThreeP_zero}
    Assume $G_{1,T} \ge G_{2,T} \ge \ldots \ge G_{L,T}$.  We see that the optimal item $ \bari_T^*=1$.
    The failure probability of {\sc Exp3.P$(\gamma, \eta)$} satisfies
\begin{align*}
    &
    \bare_T
    \le 
    \exp \bigg(  - \frac{  \gamma T\barDelta_{1,2,T}^2  }{ 4L   }  \bigg)
        +
        \sum_{i=2}^L \exp \bigg(  - \frac{  3\gamma T(\barDelta_{1,2,T}/2 + \barDelta_{2,i,T} )^2 }{   L(3  +\barDelta_{1,2,T}/2 + \barDelta_{2,i,T}   )}    \bigg)
    \le 
    L \exp \bigg(  - \frac{  \gamma T\barDelta_{ T}^2  }{ 4L   }  \bigg)
    .
\end{align*}
\end{restatable}

The key idea among the analysis of Theorem~\ref{thm:bd_bai_expThreeP_zero} is to derive high-probability one-sided bounds on $ \tilde{G}_{i,T}  - G_{i,T}  $ for all $i\in[L]$ with Theorems~\ref{thm:conc_var_ub} and~\ref{thm:conc_var_lb}. 
The detailed proof is postponed to Appendix~\ref{sec:analyze_bd_bai_expThreeP}.
 
Theorems~\ref{thm:rm_bd_expThreeP_beta_zero} and \ref{thm:bd_bai_expThreeP_zero} imply that
by adjusting {\sc Exp3.P$(\gamma,\eta)$} with $\gamma$, we can balance between exploitation and exploration and trade-off
between the twin objectives --- RM and BAI. In detail,
\begin{itemize}
    \item When $\gamma$ increases, the {\sc Exp3.P$(\gamma,\eta)$} algorithm tends to bahave more similarly to {\sc UP-ADV}, which leads to a larger regret and a smaller failure probability. This indicates that a large $\gamma$ encourages exploitation.
    \item When $\gamma$ decreases, the {\sc Exp3.P$(\gamma,\eta)$} algorithm tends to emphasize more on the observation from previous time steps and pull the items with high empirically means, which leads to a smaller regret and a larger failure probability. In other words, a small $\gamma$ encourages exploitation.
\end{itemize} 
 

\subsection{Global performances of adversarial algorithms}
\label{sec:adv_lb_bai_trade_off_rm_bai}
In this section, we first lower bound the failure probability to identify the empirically-optimal item in adversarial bandits.
Next,
given a certain failure probability of an algorithm, we
establish a non-trivial lower bound on its empirical-regret.
The proofs are in Appendix \ref{sec:pf_adv_lb_trade_off}.

We consider bandit instances in which items have
bounded rewards.
Let $\barcalB_1( \underline{\Delta}_{ T} ,\barR  )$ denote the set of instances where 
(i)  the empirically-minimal optimality gap $  \barDelta_{ T}
\ge \underline{\Delta}_{ T} $ in $T$ time steps;
and (ii) there exists $R_0 \in \bbR$ such the rewards are
bounded in $[R_0, R_0 + R]$. 
We focus on $\barcalB_1( \underline{\Delta}_{ T} ,1  )$ for brevity; the analysis can be generalized for any $\barcalB_1( \underline{\Delta}_{ T} ,\barR  )$.

\subsubsection{Lower bound on the BAI failure probability in adversarial bandits}

\begin{restatable}{theorem}{thmAdvBaiLb} 
\label{thm:adv_bai_lb}
Let $0<\underline{\Delta}_{ T}\le 1$.
Then any algorithm $\pi$ satisfies that  
\begin{align*}
	\sup_{ \barcalB_1(  \underline{\Delta}_{ T}, 1) }
	\bare_T(\pi,\calI) 
	\ge 
	\frac{ 1 -   \exp (     - {  3T    }/{ 200 } )  }{4}  \cdot
      \exp\bigg(   -   \frac{  150 T  \underline{\Delta}_{ T}^2 }{  L } \bigg) 
     .
\end{align*}
Furthermore, when $T\ge 10$, 
\begin{align*}
	\sup_{ \barcalB_1(  \underline{\Delta}_{ T},1) }
	\bare_T(\pi,\calI) 
	\ge 
    \frac{ 2}{65 } \exp\bigg(   -   \frac{  150 T  \underline{\Delta}_{ T}^2 }{  L } \bigg) 
     .
\end{align*}
%
%
\end{restatable}
We construct $L$ instances with clipped Gaussian distributions, which are similar to those designed for the analysis of lower bound on regret in \citet{gerchinovitz2016refined}.

Besides,  
the gap between our lower bound in Theorem~\ref{thm:adv_bai_lb} and the upper bounds of {\sc UP-ADV/Exp3.P$(1,\eta)$} in \eqref{eq:bai_bd_up_adv} and Theorem~\ref{thm:bd_bai_expThreeP_zero} is manifested by the
(pre-exponential) term $L$ as well as the constant in the exponential term.
This indicates that {\sc UP-ADV/Exp3.P$(1,\eta)$} perform near-optimally for BAI
and our lower bound in Theorem~\ref{thm:adv_bai_lb} is almost tight. 


\subsubsection{Trade-off between RM and BAI in adversarial bandits}

\begin{restatable}{theorem}{thmAdvRmBaiGauss}
\label{thm:adv_rm_bai_gauss} 
Let $0<\underline{\Delta}_{ T}\le 1$ and $T\ge 10$.
Let $\pi$ be any algorithm with $\bare_T (\pi,\calI) \le 2\exp(-\psi_T )/65$ for all $\calI \in  \barcalB_1( \underline{\Delta}_{ T},1   ) $.
Then
\begin{align*}
    \sup_{ \calI \in \barcalB_1(\underline{\Delta}_{ T}, 1) }
    \bbE \barR_T (\pi,\calI) \ge \psi_T \cdot \frac{  L-1 }{  103 \underline{\Delta}_{ T} }.
\end{align*}
%
%

%
\end{restatable}

Theorem~\ref{thm:adv_rm_bai_gauss} implies that, as shown for the stochastic bandits (see Theorems~\ref{thm:sto_rm_bai_bern} and \ref{thm:sto_rm_bai_gauss}), 
we cannot achieve optimal performances for both
RM and BAI using any algorithm with fixed parameters in adversarial bandits.
Besides, Theorems~\ref{thm:bd_bai_expThreeP_zero} and \ref{thm:adv_rm_bai_gauss} indicates that
\begin{align*}
    \sup_{ \calI \in \barcalB_1(\underline{\Delta}_{ T}, 1) }
    \bbE \barR_T ( \textsc{Exp3.P$(\gamma,\eta)$} ,\calI) \ge 
    \bigg(  \log \bigg( \frac{ 2 }{65}  \bigg) +   \frac{  \gamma T\underline{\Delta}_T^2  }{ 4L   }  \bigg)
    \cdot \frac{  L-1 }{  103 \underline{\Delta}_{ T} }
    = \Omega (  \gamma T\barDelta_T   )
    .
\end{align*}
However, since the upper bound on the regret of \textsc{Exp3.P$(\gamma,\eta)$} in  Theorem~\ref{thm:rm_bd_expThreeP_beta_zero} is problem-independent, we cannot ascertain if the algorithm achieves the Pareto optimality, which may serve as an interesting direction for future work.
Lastly, we summarize some theoretical findings of the adversarial bandits in Table~\ref{tab:comp_result_adv}.
 \begin{table}[ht]

    \caption{Comparison among upper bounds for algorithms and lower bounds in adversarial bandits.
        }
    \label{tab:comp_result_adv}
    \vspace{.5em}
    \centering
    \renewcommand{\arraystretch}{1.5}
    \begin{tabular}{ l      l  l}
        Algorithm/Instance &   Expected empirical-regret $\bbE \barR_T$ & Failure Probability $\bare_T$ \\
        \thickhline
        \multirow{2}{*}{\sc UP-ADV} &  $ {\Theta} ( T)$ & $ \displaystyle  L \exp \bigg(  - \frac{ 3T  \barDelta_T^2 }{ 28 L  } \bigg)  \vphantom{\Bigg(}$ 
        \\  
        & &  \citep{abbasi2018best}
        \\
        \hline  
        \multirow{2}{*}{\sc Exp3.P$(\gamma,\eta)$} &  $ \displaystyle \gamma T + \eta L T +  \ln  \bigg(   \frac{    L^2 T }{  \eta   } \bigg) +  \frac{ \ln L }{ \eta} + 1$  \hphantom{aaa} & 
        $ \displaystyle L \exp \bigg(  - \frac{  \gamma T\barDelta_T^2  }{ 4L   }  \bigg) \vphantom{\Bigg(}$  
        \\
        & (Theorem~\ref{thm:rm_bd_expThreeP_beta_zero} )
        &
        (Theorem~\ref{thm:bd_bai_expThreeP_zero}) 
        \\ 
        \hhline{===}
            Adversarial Bandits
          &    
            &
            $ \displaystyle  \frac{2}{65} \exp \bigg(  - \frac{ 150T  \barDelta_T^2 }{   L  } \bigg)  \vphantom{\Bigg(}$    
          \\ 
          &  
          & (Lower Bound, Theorem~\ref{thm:adv_bai_lb})
        \\
        \hline
            Adversarial Bandits
          &    
          $\displaystyle  \psi_T \cdot \frac{  L-1 }{  103 \underline{\Delta}_{ T} }$
            &
            $ \displaystyle  \frac{2}{65} \exp(-\psi_T  )  \vphantom{\Bigg(}$    
          \\ 
          &  (Lower Bound, Theorem~\ref{thm:adv_rm_bai_gauss} )
          & (Theorem~\ref{thm:adv_rm_bai_gauss})
    \end{tabular}
    \renewcommand{\arraystretch}{1}
\end{table}

\section{Useful facts}
\label{sec:useful_fact}

\subsection{Concentration}


\begin{theorem}[Non-asymptotic law of the iterated logarithm; \citet{jamieson2014lil}, Lemma 3]\label{thm:conc_log}
	Let $X_{1}, X_{2}, \ldots$ be $i . i . d .$ zero-mean sub-Gaussian random variables with scale $\sigma>0 ;$ i.e.
$\mathbb{E} [ \mathrm{e}^{\lambda X_{i}}] \leq \exp( {\lambda^{2} \sigma^{2}}/{2}) .$ 
For all  $\varepsilon\in (0,1)$ and  $\gamma \in (0, \log( 1+\varepsilon)/e ) $, we have
$$\Pr \bigg(\forall \tau \geq 1,  \frac{1}{\tau} \sum_{s=1}^{\tau} X_{s}  \leq 
 \sigma  (1+ \sqrt{\varepsilon} ) \sqrt{ \frac{ 2 (1+\varepsilon)}{\tau} \cdot \log  \bigg(  \frac{  \log( (1+\varepsilon) \tau) }{\gamma} \bigg)  } ~ \bigg) 
\geq 
1- \frac{2+\varepsilon}{ \varepsilon }  \bigg(  \frac{ \gamma }{ \log(1+\varepsilon) }  \bigg)^{1+\varepsilon} .$$
\end{theorem}  

\begin{restatable}[\citet{chung2006concentration}, Theorem 20]{theorem}{thmConcVarUb}
\label{thm:conc_var_ub}

Let $X_1, \cdots ,X_n$ be a martingale adapted to filtration $\calF = (\calF_i)_i$ satisfying
\begin{enumerate}
	\item $\var( X_i | \calF_{i-1} ) \le \sigma_i^2$, for $1\le i\le n$;
	\item $X_{i} - X_{i-1} \le a_i+M $, for $1\le i \le n$.
\end{enumerate}
Then we have
\begin{align*}
	\Pr\left( X_n - \bbE X_n \ge  \lambda \right) \le 
	\exp\left(
		-\frac{\lambda^2}{ 2 [ \sum_{i=1}^n (\sigma_i^2  + a_i^2 ) + M\lambda/3  ]  }
	\right).
\end{align*}
\end{restatable}

\begin{restatable}[\citet{chung2006concentration}, Theorem 22]{theorem}{thmConcVarLb}
\label{thm:conc_var_lb}

Let $X_1, \cdots ,X_n$ be a martingale adapted to filtration $\calF = (\calF_i)_i$ satisfying
\begin{enumerate}
	\item $\var( X_i | \calF_{i-1} ) \le \sigma_i^2$, for $1\le i\le n$;
	\item $X_{i-1} - X_{i} \le a_i+M$, for $1\le i \le n$.
\end{enumerate}
Then we have
\begin{align*}
	\Pr\left( X_n - \bbE X_n \le -\lambda \right) \le 
	\exp\left(
		-\frac{\lambda^2}{ 2 [ \sum_{i=1}^n (\sigma_i^2 + a_i^2) + M\lambda/3  ]  }
	\right).
\end{align*}
\end{restatable}

\begin{restatable}[\citet{AbramowitzS64}, Formula 7.1.13; \citet{AgrawalG13b}, Lemma 6; \citet{agrawal2017near}, Fact 4]{theorem}{thmConcGaussSingle}
\label{thm:conc_gauss_single}
    Let $Z\sim {\cal N}(\mu, \sigma^2)$. The following inequalities hold:
    \begin{align*} 
        \frac{1}{2\sqrt{\pi}}\exp\bigg(-\frac{7 z^2}{ 2}\bigg) &  {\leq } \Pr_Z (  | Z - \mu | > z\sigma )  {\leq}  \exp\bigg(-\frac{z^2}{2}\bigg) &  \forall  z > 0
        , \\
        \frac{1}{2\sqrt{\pi}z}\exp\bigg(-\frac{z^2}{ 2}\bigg) &  {\leq} \Pr_Z ( | Z - \mu | > z\sigma )   {\leq}  \frac{ \sqrt{2} }{ \sqrt{\pi} z }\exp\bigg(-\frac{z^2}{2}\bigg) &  \forall z\ge 1 .
    \end{align*} 

\end{restatable}

\begin{theorem}[Standard multiplicative variant of the Chernoff-Hoeffding bound; 
\citet{dubhashi2009concentration}, Theorem 1.1] \label{thm:conc_var_chernoff}
     Suppose that $X_1, \ldots  , X_T$ are independent $[0, 1]$-valued random variables, and let $X = \sum^T_{t=1} X_t$. Then for all $\varepsilon \in(0,1)$,
     \begin{align*}
         \Pr(       X - \mathbb{E} X \ge \varepsilon \mathbb{E} X  )    \le   \exp \bigg(     - \frac{ \varepsilon^2 }{3}  \mathbb{E} X  \bigg), \ 
         \Pr (      X - \mathbb{E} X  \le -\varepsilon \mathbb{E} X )    \le  \exp \bigg(     - \frac{ \varepsilon^2 }{3}  \mathbb{E} X  \bigg).
     \end{align*}

\end{theorem} 

\subsection{Change of measure}

\begin{lemma}[\citet{tsybakov2008introduction}, Lemma 2.6]
\label{lemma:kl_to_event}
    Let $P$ and $Q$ be two probability distributions on the same measurable space. Then, for every measurable subset $A$ (whose complement we denote by $\barA$),
    \begin{align*}
        P(A) + Q(\barA)  \ge \frac{ 1}{2 } \exp( -\mathrm{KL}(P \parallel Q ) ).
    \end{align*} 
\end{lemma}

\begin{lemma}[\citet{gerchinovitz2016refined}, Lemma 1]
\label{lemma:kl_decomp}
Consider two instances $1$ and $2$. We let $N_{i,t}$ denote the number of pulls of item $i$ up to and including time step $t$.
Under instance $j$~($j=1,2$), 
\begin{itemize}[ itemsep = -4pt,   topsep = 8pt, leftmargin =  15pt ]
    \item we let $(g_{i,t}^j)_{t =1}^T$ be the sequence of rewards of item $i$ and $i_t^j$ be the pulled item at time step $t$, and let  $P_{j,i}$ denote the distribution of the gain of item $i $; 
    \item we assume $\{ \bm{g}_t^j = ( g_{1,t}^j, g_{2,t}^j,\ldots, g_{L,t}^j ) \}_{t=1}^T$ is an i.i.d.\ sequence, i.e., $\bm{g}_{t_1}^j$ and $\bm{g}_{t_2}^j$ are i.i.d.\  for $t_1\ne t_2$ but $\{ g_{i,t}^j \}_{i=1}^L$ can be independent.
    \item we let $i_t^j$ be the pulled item at time step $t$, and let $\bbP_j$ denote the probability law of the process  
$\{  \{ i_t^{j},   g_{ i_t^ j ,t   }^j\} \}_{t=1}^T$. 
\end{itemize}

Then, we have
    \begin{align*}
        \mathrm{KL} ( \bbP_1 \parallel \bbP_{2}  )
        =
        \sum_{i=1}^L
        \bbE_{\bbP_1 }  [ N_{i,T} ] \cdot \mathrm{KL} ( P_{1,i} \parallel P_{2,i} ).
    \end{align*}

\end{lemma}

\subsection{KL divergence}

\begin{restatable}[Pinsker's and reverse Pinsker's inequality; \citet{friedrich2019higher}, Lemma 4.1]{theorem}{thmPinsker} 
\label{thm:pinsker}
     {Let $P$ and $Q$ be two distributions that are defined in the same finite space 
     $\mathcal{A}$ and have
      the same support.} We have
\begin{align*}
	\delta(P,Q)^2 \le \frac{1}{2} \mathrm{KL}(P,Q)
	\le \frac{1}{  \alpha_Q} \delta(P,Q)^2
\end{align*}
where
$
	\delta(P,Q) = \sup\{\ | P(A)-Q(A)| \ \big|  A \subset \mathcal{A}   \} 
 =\frac{1}{2} \sum_{x\in  \mathcal{A}} |P(x) - Q(x)|$
	is the total variational distance,
	and  
	$\alpha_Q = \min_{x\in X: Q(x)>0} Q(x)
$.

\end{restatable}

%
\begin{lemma}[KL divergence between two Gaussian distributions]
\label{lemma:kl_gauss}
    Let $P_1 = \calN( \mu_1, \sigma_1^2 )$, $P_2 = \calN( \mu_2, \sigma_2^2 )$. Then
    \begin{align*}
        \mathrm{KL}( P_1 || P_2 ) = \log \bigg(  \frac{ \sigma_2 }{ \sigma_1 } \bigg) + \frac{  \sigma_1^2 + ( \mu_1 - \mu_2 )^2 }{  2\sigma_2^2  } - \frac{1}{2}.
    \end{align*}
\end{lemma} 

\begin{lemma}[KL divergence between clipped Gaussian distributions; Lemma  7, \citet{gerchinovitz2016refined}]
\label{lemma:kl_gauss_clip}
    Let $Z$ be normally distributed with mean $1/2$ and variance $\sigma^2 > 0$. Let $\mathrm{clip}_{[a,b]}x:= \max\{  a, \min\{ b , x\}  \} $ for $a\le b$. Define $X = \mathrm{clip}_{[0,1]}(Z)$ and  $Y = \mathrm{clip}_{[0,1]}(Z -\varepsilon)$ for $ \varepsilon\in \bbR$.
Then 
$$\mathrm{KL}( P_X  \parallel P_Y )  \le \frac{ \varepsilon^2  }{2 \sigma^2 }.$$
\end{lemma}

\section{Analysis of {\sc BoBW-lil'UCB$(\gamma)$} in stochastic bandits}
\label{append:sto_analyze_lilUcb}

%
%

\begin{restatable}[Bounds on the pseudo-regret of {\sc BoBW-lil'UCB$(\gamma)$}]{proposition}{propRmBdLilUcbPara}

\label{prop:rm_bd_lilUcb_para}
Assume the distribution $\nu_i$ is sub-Gaussian  with scale $\sigma>0$ for all $i\in[L]$,
and
$w_1 \ge w_2  \ge \ldots \ge w_L $. 
Let  $\varepsilon\in (0,1)$,  
$\beta\ge 0$, and  $\gamma \in (0, \min\{ \log(   \beta +    1+\varepsilon)/e ),1 \}) $.
The pseudo-regret of {\sc BoBW-lil'UCB$(\gamma)$} satisfies 
\begin{align*}
    R_T
    \le O\bigg( \sigma^2  (1+  \varepsilon )^3   \cdot \sum_{i: \Delta_{1,i}>0} \frac{  \log (  {1}/{\gamma} )  }{  \Delta_{1,i}  }  \bigg)
    ,
    \quad 
    R_T
    \le O \bigg( \sigma^2  (1+  \varepsilon )^3  \sqrt{TL} \log \bigg( \frac{ \log(  T/L\gamma  )  }{\gamma }\bigg)
    ~\bigg).
\end{align*}
Furthermore, we can set $\gamma = 1/ \sqrt{T}$ to obtain
\begin{align*}
    R_T
    \le O\bigg( \sigma^2  (1+  \varepsilon )^3  \cdot \sum_{i: \Delta_{1,i}>0} \frac{  \log T  }{  \Delta_{1,i}  }  \bigg)
    ,
    \quad 
    R_T
    \le O \big( \sigma^2  (1+  \varepsilon )^3  \sqrt{TL} \log  T     
    ~\big).
\end{align*} 

\end{restatable}

\subsection{Proof of Theorem~\ref{thm:rm_bd_lilUcb_para} }
\thmRmBdLilUcbPara*
\label{pf:thm_rm_bd_lilUcb_para}

\begin{proof}

Recall that
we assume  $w_1 > w_2  \ge \ldots \ge w_L $. 
Therefore, item $ 1$ is optimal and $\Delta_{1,j} >0$ for all $j\ne 1$.

\noindent
{\bf Step 1: Concentration.}
Let  $ \calE_{i,\gamma} : = \{ \forall t \geq L,  
     | \hat{g}_{i,t} - w_i | \le C_{i,t,\gamma}  \}$ for all $i\in[L]$. We apply Theorem~\ref{thm:conc_log} to show that $  \bigcap_{i=1}^L  \calE_{i,\gamma}$ holds with high probability.
\begin{restatable}[Concentration of $\hatg_{i,t}$]{lemma}{lemmaLilUCBConc}
\label{lemma:lilUcb_conc} 
Fix any $\varepsilon\in (0,1)$ and  $\gamma \in (0, \log(    \beta+  1+\varepsilon)/e ) $.
We have
\begin{align*}
    \Pr \bigg(  \bigcap_{i=1}^L  \calE_{i,\gamma} \bigg)
    \ge 1- \frac{ 2L(2+\varepsilon)}{ \varepsilon }  \bigg(  \frac{ \gamma }{ \log(1+\varepsilon) }  \bigg)^{1+\varepsilon}.
\end{align*}  
\end{restatable}

\noindent
{\bf Step 2: Bound on $N_{i,T}$ for $i\ne 1$.}
Next, for all $t > L$, when 
\begin{align*}
    & 
    \{ ~ \hatg_{ 1,t-1 } > w_1 - C_{1,t-1,\gamma},
    \quad
    \hatg_{ i, t-1 } < w_i + C_{i,t-1,\gamma},
    \quad
    \Delta_{1,i} > 2 C_{i,t-1,\gamma},
    \quad
    \forall i \ne 1
    \}
\end{align*}
holds, we have
\begin{align*}
    & 
    \{~
    U_{1, t-1,\gamma } 
    =  \hatg_{ 1,t-1 }  + C_{1,t-1,\gamma}
    > w_1
    = w_i + \Delta_{1,i}
    > w_i + 2C_{ i,t-1,\gamma  }
    > \hatg_{ i, t-1 } + C_{ i,t-1,\gamma  }
    = U_{ i,t-1,\gamma  }
    \quad
    \forall i \ne 1 ~\},
\end{align*}
which indicates $i_t = 1$.
In other words,
when $i_t=i \ne 1$ for $t>L$, one of the following holds:
\begin{align*}
    & \hatg_{1,t-1 } \le w_1 - C_{1,t-1,\gamma},
    \quad
    \hatg_{i,t-1} \ge w_i + C_{ i,t-1,\gamma},
    \quad
    \Delta_{1,i} \le 2 C_{ i,t-1,\gamma},
\end{align*} 
We see that
\begin{align*}
    & \Delta_{1,i} \le 2 C_{ i,t-1,\gamma} = 10 \sigma  (1+ \sqrt{\varepsilon} ) \sqrt{ \frac{ 2 (1+\varepsilon)}{N_{i,t-1}} \cdot \log  \bigg(  \frac{  \log(    \beta +    (1+\varepsilon) N_{i,t-1} ) }{\gamma} \bigg)  }
    \\
    \Leftrightarrow~ &
    N_{i,t-1} \le    \frac{ 200 \sigma^2 (1+ \sqrt{\varepsilon})^2(1+\varepsilon)}{  \Delta_{1,i}^2  } \cdot \log  \bigg(  \frac{  \log(  \beta +  (1+\varepsilon) N_{i,t-1} ) }{\gamma} \bigg).
\end{align*}

In order to bound $N_{i,t-1}$, we derive the following lemma:
\begin{restatable}{lemma}{lemmaIneqLilT}
\label{lemma:ineq_lil_T}
For all  $\tau >0$, 
$1.4ac/\rho + b \ge e$,  we have
\begin{align*}
    \tau \le c \log \bigg(  \frac{ \log(a \tau +b) }{\rho} \bigg)
    ~\Rightarrow~
    \tau \le  c \log \bigg( \frac{ 1.4 }{ \rho } \log  \bigg( \frac{ 1.4 a   c}{\rho } + b \bigg) \bigg).
\end{align*}
\end{restatable}

We apply Lemma~\ref{lemma:ineq_lil_T} with
\begin{align*}
    c = \frac{ 200 \sigma^2 (1+ \sqrt{\varepsilon})^2(1+\varepsilon)}{  \Delta_{1,i}^2  } ,
    \quad
    a = 1 + \varepsilon,
    \text{ and }~
    \rho = \gamma, 
\end{align*}
to obtain
\begin{align*}
        N_{i,t-1} \le 
         \frac{ 200 \sigma^2 (1+ \sqrt{\varepsilon})^2(1+\varepsilon)}{  \Delta_{1,i}^2  } 
         \cdot 
          \log \bigg( \frac{ a_1 }{ \gamma } \log  \bigg(\frac{ 200 a_1 \sigma^2 (1+ \sqrt{\varepsilon})^2(1+\varepsilon)^2 }{  \Delta_{1,i}^2 \gamma  }     +\beta    \bigg) \bigg)
          := \bar{N}_{i,\gamma}.
\end{align*}
$1.4ac/\rho + b \ge e$ is satisfied when $\gamma\in (0,1)$.
Therefore, when $t>L$, $\bigcap_{i=1}^L  \calE_{i,\gamma}$ holds and $N_{i,t-1} > \barN_{i, \gamma}$ for all $i \neq 1 $,
we always have $i_t=1$.

\noindent
{\bf Step 3: Conclusion.}
Consequently,
\begingroup
\allowdisplaybreaks
\begin{align*}
    R_T 
    &
    = \bbE \bigg[ \sum_{t=1}^T g_{1,t} - g_{i_t,t} \bigg]
    = \bbE \bigg[  \bigg( \sum_{t=1}^T g_{1,t} - g_{i_t,t} \bigg) \cdot \mathsf{1} \bigg( \bigcap_{i=1}^L  \calE_{i,\gamma} \bigg) \bigg]
        + \bbE \Bigg[  \bigg( \sum_{t=1}^T g_{1,t} - g_{i_t,t} \bigg) \cdot \mathsf{1} \Bigg(~ \overline{ \bigcap_{i=1}^L  \calE_{i,\gamma} } ~\Bigg) \Bigg]
    \\&
    \le  \bbE \bigg[  \bigg( \sum_{t=1}^T g_{1,t} - g_{i_t,t} \bigg) \cdot \mathsf{1} \bigg( \bigcap_{i=1}^L  \calE_{i,\gamma} \bigg) \bigg]
           +  T \cdot  \Pr \Bigg(~ \overline{ \bigcap_{i=1}^L  \calE_{i,\gamma} } ~\Bigg)
     \\&
     \le \sum_{j\ne 1} \bbE \bigg[  \bigg( \sum_{t=1}^T g_{1,t} - g_{i_t,t} \bigg) \cdot \mathsf{1} \bigg( i_t=j, \bigcap_{i=1}^L  \calE_{i,\gamma} \bigg) \bigg]
     +  T \cdot  \Pr \Bigg(~ \overline{ \bigcap_{i=1}^L  \calE_{i,\gamma} } ~\Bigg)
     \\& 
     \le \sum_{j\ne 1} \Delta_{1,j} \cdot \bbE \bigg[    \sum_{t=1}^T   \cdot \mathsf{1}  ( i_t=j   )~\bigg|~ \bigcap_{i=1}^L  \calE_{i,\gamma} \bigg]
     +   T \cdot \Pr \Bigg(~ \overline{ \bigcap_{i=1}^L  \calE_{i,\gamma} } ~\Bigg)
     \\&
     = \sum_{j\ne 1} \Delta_{1,j  }\cdot \bbE \bigg[  N_{ j,T} ~\bigg|~ \bigcap_{i=1}^L  \calE_{i,\gamma} \bigg]
     +   T \cdot \Pr \Bigg(~ \overline{ \bigcap_{i=1}^L  \calE_{i,\gamma} } ~\Bigg)
     \\&
     \le \sum_{j \ne 1} \Delta_{1, j} \cdot ( 2 + \bar{N}_{j,\gamma} )  +   T \cdot \Pr \Bigg(~ \overline{ \bigcap_{i=1}^L  \calE_{i,\gamma} } ~\Bigg)
     \\&
     = \sum_{i \ne 1} 2 \Delta_{1,i}  + \sum_{i \ne 1} \frac{ 200 \sigma^2 (1+ \sqrt{\varepsilon})^2(1+\varepsilon)}{  \Delta_{1,i}  } 
         \cdot 
          \log \bigg( \frac{ a_1 }{ \gamma } \log  \bigg(\frac{ 200 a_1   \sigma^2 (1+ \sqrt{\varepsilon})^2(1+\varepsilon)^2 }{  \Delta_{1,i}^2 \gamma  }  + \beta     \bigg) \bigg)
          \\& \hspace{2em}
     + \frac{ 2TL(2+\varepsilon)}{ \varepsilon }  \bigg(  \frac{ \gamma }{ \log(1+\varepsilon) }  \bigg)^{1+\varepsilon}
     \\&
     = \sum_{i \ne 1} 2 \Delta_{1,i}  + \sum_{i \ne 1} \frac{ 200 \sigma^2 (1+ \sqrt{\varepsilon})^2(1+\varepsilon)}{  \Delta_{1,i}  } 
         \cdot 
          \log \bigg( \frac{ 2 a_1 }{ \gamma } \log  \bigg(\frac{ 10 \sqrt{2a_1} \cdot \sigma (1+ \sqrt{\varepsilon}) (1+\varepsilon)  }{  \Delta_{1,i}  \sqrt{\gamma}  } + \beta     \bigg) \bigg)
          \\& \hspace{2em}
     + \frac{2 TL(2+\varepsilon)}{ \varepsilon }  \bigg(  \frac{ \gamma }{ \log(1+\varepsilon) }  \bigg)^{1+\varepsilon}.
\end{align*}
\endgroup
We see that $a_1=1.4\le 2$.
If we divide the ground set into two classes depending on whether $\Delta_{1,i}\ge \sqrt{L/T}$, we have
\begin{align*}
    R_T
    &
    \le T \cdot \sqrt{ \frac{L}{T} } + 2L  + 
    \frac{ 200L \sigma^2 (1+ \sqrt{\varepsilon})^2(1+\varepsilon)}{ \sqrt{L/T}  } 
          \log \bigg( \frac{ 4 }{ \gamma } \log  \bigg(\frac{ 20 \sigma (1+ \sqrt{\varepsilon}) (1+\varepsilon)  }{  \sqrt{L/T} \sqrt{\gamma}  }    + \beta    \bigg) \bigg)
          \\& \hspace{2em}
     + \frac{ 2TL(2+\varepsilon)}{ \varepsilon }  \bigg(  \frac{ \gamma }{ \log(1+\varepsilon) }  \bigg)^{1+\varepsilon}
    \\&
    =  \sqrt{TL} \cdot \bigg[ 1+ 
    200  \sigma^2 (1+ \sqrt{\varepsilon})^2(1+\varepsilon) 
          \log \bigg( \frac{ 4 }{ \gamma } \log  \bigg(\frac{ 20  \sigma (1+ \sqrt{\varepsilon}) (1+\varepsilon)  }{  \sqrt{\gamma L/T}   }  + \beta     \bigg) \bigg)
    \bigg]
        \\ & \hspace{2em}
      + 2L  +  \frac{ 2TL(2+\varepsilon)}{ \varepsilon }  \bigg(  \frac{ \gamma }{ \log(1+\varepsilon) }  \bigg)^{1+\varepsilon}. 
\end{align*}
In short, we have
\begin{align*}
    R_T
    &
    \le O\bigg( \sigma^2  \cdot  \sum_{i\ne 1} \frac{  \log (  {1}/{\gamma} )  }{  \Delta_{1,i}  } +  2TL  \gamma^{1+\varepsilon}  \bigg)
    ,
    \\ 
    R_T
    & 
    \le O \bigg(  \sigma^2   \sqrt{TL} \log \bigg( \frac{ \log(  T/L\gamma  )  }{\gamma }
    \bigg)
    +  2TL  \gamma^{1+\varepsilon} 
    ~\bigg).
\end{align*}

Let $\gamma = (\log T)/ {T}$, we have
\begin{align*}
    R_T
    \le O\bigg( \sigma^2    \cdot \sum_{i\ne 1} \frac{  \log T  }{  \Delta_{1,i}  }  \bigg)
    ,
    \quad
    R_T
    \le O \big(  \sigma^2    \sqrt{TL} \log  T     
    ~\big).
\end{align*} 
%
\end{proof}

 \subsection{Proof of Theorem~\ref{thm:bai_bd_lilUcb_para} }
\label{pf:thm_bai_bd_lilUcb_para}

\thmBaiBdLilUcbPara*
\begin{proof}

Recall that we assume $w_1 > w_2 \ge \ldots \ge w_L$.
We let $\Delta_1 = w_1 - w_2$ and $\Delta_i = w_1 - w_i$ for all $i \ne 1$.
Then $\Delta = \Delta_1$ and $\Delta_{1,i} = \Delta_{i}$ for $i\ne 1$.

\noindent
{\bf Step 1: Concentration.}
Let  $ \calE'_{i,\gamma} : = \{ \forall t \geq L,  
     | \hat{g}_{ i,t } - w_i | \le C_{i,t,\gamma}/5  \}$ for all $i\in[L]$. Similarly to Lemma~\ref{lemma:lilUcb_conc}, we can apply Theorem~\ref{thm:conc_log} to show that 
\begin{align*}
    \Pr \bigg(  \bigcap_{i=1}^L  \calE'_{i,\gamma} \bigg)
    \ge 1- \frac{ 2L(2+\varepsilon)}{ \varepsilon }  \bigg(  \frac{ \gamma }{ \log(1+\varepsilon) }  \bigg)^{1+\varepsilon}.
\end{align*}  
In the following, we prove that conditioning on the event $ \big\{  \bigcap_{i=1}^L  \calE'_{i,\gamma} \big\}$, we have $\iout =1$, which concludes the proof. 

We assume $\bigcap_{i=1}^L  \calE'_{i,\gamma}$ holds from now on.
Since $\iout$ is the item with the largest empirical mean, we have
\begin{align*}
    \hat{g}_{\iout,T} \ge \hat{g}_{i,t} \quad \forall i \ne \iout ,\quad
    \hat{g}_{ \iout,T } \ge w_{ \iout } - C_{\iout,T,\gamma}/5,\quad
    w_i + C_{i,T,\gamma}/5  \ge \hat{g}_{i,t} \quad \forall i \ne \iout.
\end{align*}      
Consequently, to show $\iout =1$, it is sufficient to show that 
\begin{align}
    \frac{  C_{i,T,\gamma} }{  5 }    \le  \frac{ \Delta_{i} }{2}
    ~\Leftrightarrow~
    & \Delta_{ i} \ge \frac{2 C_{ i,T,\gamma}}{5} = 2  \sigma  (1+ \sqrt{\varepsilon} ) \sqrt{ \frac{ 2 (1+\varepsilon)}{N_{i,T}} \cdot \log  \bigg(  \frac{  \log(   \beta +  (1+\varepsilon) N_{i,T} ) }{\gamma} \bigg)  }
    \nonumber \\
    \Leftrightarrow~ &
    N_{i,T} \ge    \frac{ 8\sigma^2 (1+ \sqrt{\varepsilon})^2(1+\varepsilon)}{  \Delta_{ i}^2  } \cdot \log  \bigg(  \frac{  \log(    \beta + (1+\varepsilon) N_{i,T} ) }{\gamma} \bigg) \quad \forall i\in[L].
    \label{eq:lilUcb_bai_target}
\end{align}

\noindent
{\bf Step 2: Upper bound $N_{i,T}~(i \ne 1)$.}
To begin with, we let $a_1=2$ and prove prove by induction that
\begin{align} 
    N_{i,t} \le \frac{  72 \sigma^2 (1+ \sqrt{\varepsilon})^2(1+\varepsilon)}{  \Delta_{ i}^2  } \cdot \log \bigg(  \frac{ a_1 }{ \gamma } \log \bigg(   \frac{  72 a_1 \sigma^2 (1+ \sqrt{\varepsilon})^2(1+\varepsilon)^2 }{  \Delta_{ i}^2 \gamma  }  + \beta      \bigg)~\bigg) + 1
    \quad
    \forall i \ne 1.
    \label{eq:lilUcb_ub_num_pull} 
\end{align}
Clearly, this inequality holds for all $ i \ne 1 $ when $1\le t \le L$. Now we assume that the inequality holds for all $ i \ne 1 $ at time $t-1 (t>L)$. 
If $i_t \ne i$, we have $N_{i,t} = N_{i,t-1}$ and the inequality still holds for $i$.
Otherwise, we have $i_t = i$ and in particular $U_{ i,t-1,\gamma } \ge U_{1,t-1,\gamma}$.
Since
\begin{align*}
    U_{i,t-1,\gamma} = \hat{g}_{i,t-1} +  C_{i,t-1,\gamma} 
    \le w_i + \frac{ 6C_{i,t-1,\gamma} }{5},
    \quad 
    U_{1,t-1,\gamma} = \hat{g}_{1,t-1} +  C_{1,t-1,\gamma}  
    \ge w_1 + \frac{ 4C_{1,t-1,\gamma} }{5}
    \ge w_1 = w_i + \Delta_{i}, 
\end{align*}
we have
\begin{align*}
    \frac{ 6C_{i,t-1,\gamma} }{5} \ge \Delta_i
    &
    ~\Leftrightarrow~
    N_{i,t-1} \le \frac{  72 \sigma^2 (1+ \sqrt{\varepsilon})^2(1+\varepsilon)}{  \Delta_{ i}^2  } \cdot \log  \bigg(  \frac{  \log(   \beta + (1+\varepsilon) N_{i,t-1} ) }{\gamma} \bigg)
    \\&
    ~\overset{(a)}{\Rightarrow}~
    N_{i,t-1} \le \frac{  72 \sigma^2 (1+ \sqrt{\varepsilon})^2(1+\varepsilon)}{  \Delta_{ i}^2  } \cdot \log \bigg(  \frac{ a_1 }{ \gamma } \log \bigg(   \frac{  72 a_1 \sigma^2 (1+ \sqrt{\varepsilon})^2(1+\varepsilon)^2 }{  \Delta_{ i}^2 \gamma  } +  \beta     \bigg)~\bigg).
\end{align*}
We obtain (a) using Lemma~\ref{lemma:ineq_lil_T}   with $\gamma\in(0,1)$:
\lemmaIneqLilT* 
Subsequently, by using $N_{i,t} = N_{i,t-1}+1$, we obtain~\eqref{eq:lilUcb_ub_num_pull}.

\noindent
{\bf Step 3: Lower bound $N_{i,T}~(i \ne 1)$.}
Next, we again prove by induction  that
\begin{align}
    N_{i,t }  
       \ge 
       200\sigma^2 (1+\varepsilon) (1+ \sqrt{\varepsilon} )^2 \cdot  \log  \bigg(  \frac{  \log(    \beta + (1+\varepsilon) N_{i,t } ) }{\gamma} \bigg) 
       \cdot 
       \min \bigg\{
       \frac{ 1}{ 25  \Delta_i^2  }  
       ,\quad 
       \frac{ 1 }{ 36(C_{1,t-1,\gamma}  )^2  }
       \bigg\}
       ,
       \quad 
       \forall i\ne 1.
    \label{eq:lilUcb_lb_num_pull} 
\end{align}
Clearly, this inequality holds for all $ i \ne 1$ when $1\le t \le L$. Now we assume that these inequalities hold for all $ i \ne 1 $ at time $t-1 (t>L)$. 
If $i_t \ne 1$, we have 
$$N_{i,t} \ge N_{i,t-1} \quad \forall i\ne 1,\quad N_{1,t} = N_{1,t-1},$$ 
which implies that the inequalities still hold for all $i \ne 1$.
Otherwise, $i_t = 1$ indicates that $U_{1,t-1,\gamma} \ge U_{i,t-1,\gamma}$ for all $i \neq 1$.
Since
\begin{align*}
    U_{1,t-1,\gamma} = \hat{g}_{1,t-1} +  C_{1,t-1,\gamma} 
    \le w_1 + \frac{ 6C_{1,t-1,\gamma} }{5},
    \quad 
    U_{i,t-1,\gamma} = \hat{g}_{i,t-1} +  C_{i,t-1,\gamma}  
    \ge w_i + \frac{ 4C_{i,t-1,\gamma} }{5},
\end{align*}
we have
\begin{align*}
    &
    \frac{ 4C_{i,t-1,\gamma} }{5} \le \Delta_i + \frac{ 6C_{1,t-1,\gamma} }{5}
    \\
    \Leftrightarrow~
    &
    C_{i,t-1,\gamma} = 5\sigma  (1+ \sqrt{\varepsilon} ) \sqrt{ \frac{ 2 (1+\varepsilon)}{N_{i,t-1}} \cdot \log  \bigg(  \frac{  \log(    \beta + (1+\varepsilon) N_{i,t-1} ) }{\gamma} \bigg)  }   \le \frac{  5  \Delta_i +  6C_{1,t-1,\gamma} }{ 4 }
    \\
    \Leftrightarrow~
    &
    \frac{ 20\sigma  (1+ \sqrt{\varepsilon} )}{ 5  \Delta_i +  6C_{1,t-1,\gamma}   } \cdot \sqrt{  \log  \bigg(  \frac{  \log(    \beta + (1+\varepsilon) N_{i,t-1} ) }{\gamma} \bigg)  }   \le  
    \sqrt{ \frac{N_{i,t-1}}{ 2 (1+\varepsilon)}  }
    \\
    \Leftrightarrow~
    &
    \frac{ 400\sigma^2  (1+ \sqrt{\varepsilon} )^2}{ (5  \Delta_i +  6C_{1,t-1,\gamma}  )^2  } \cdot  \log  \bigg(  \frac{  \log(    \beta + (1+\varepsilon) N_{i,t-1} ) }{\gamma} \bigg)      \le  
    \frac{N_{i,t-1}}{ 2 (1+\varepsilon)}   
    \\
    \Leftrightarrow~
    &
      N_{i,t-1} \ge \frac{ 800\sigma^2 (1+\varepsilon) (1+ \sqrt{\varepsilon} )^2}{ (5  \Delta_i +  6C_{1,t-1,\gamma}  )^2  } \cdot  \log  \bigg(  \frac{  \log(    \beta + (1+\varepsilon) N_{i,t-1} ) }{\gamma} \bigg)     
       .
\end{align*}
We apply $u + v\le 2\max\{ u,v \} $ and  $N_{i,t} = N_{i,t-1}$ for all $i\ne 1$ to obtain \eqref{eq:lilUcb_lb_num_pull}.

\noindent
{\bf Step 4: Lower bound on $N_{1,T}$.}
Recall that we want to show \eqref{eq:lilUcb_bai_target}.
(i) To show \eqref{eq:lilUcb_bai_target} holds for all $i\ne 1$, 
\eqref{eq:lilUcb_lb_num_pull} indicates that
it is sufficiently to show that
\begin{align}
    &
     \frac{ 200\sigma^2 (1+\varepsilon) (1+ \sqrt{\varepsilon} )^2}{ 36(  C_{1,T-1,\gamma}  )^2  } \cdot  \log  \bigg(  \frac{  \log(    \beta + (1+\varepsilon) N_{i,T} ) }{\gamma} \bigg)     
     \ge   
     \frac{ 8\sigma^2 (1+ \sqrt{\varepsilon})^2(1+\varepsilon)}{  \Delta_{ i}^2  } \cdot \log  \bigg(  \frac{  \log(  \beta + (1+\varepsilon) N_{i,T} ) }{\gamma} \bigg)
     .
     \nonumber 
\end{align}
Moreover, since $ \Delta_1 = \min\limits_{i\in[L] } \Delta_i  $, it is sufficient to show
\begin{align*}
     \frac{ 25  }{ 36(C_{1,T-1,\gamma})^2 }  
     \ge  \frac{  1  }{  \Delta_{1}^2  }  
     ~\Leftrightarrow~ 
     C_{ 1,T-1,\gamma } \le \frac{  5 \Delta_1  }{ 6 }.
\end{align*}
(ii) In order to show \eqref{eq:lilUcb_bai_target} holds for all $i\in [L]$, 
it is sufficient to show that
\begin{align*}
    C_{ 1,T-1,\gamma } \le \frac{  5 \Delta_1  }{ 6}.
\end{align*}
 With $a_1=2$, this is implied by 
\begin{align*}
  &  N_{1,T-1} \ge  \frac{  72 \sigma^2 (1+ \sqrt{\varepsilon})^2(1+\varepsilon)}{  \Delta_{ i}^2  } \cdot \log \bigg(  \frac{a_1 }{ \gamma } \log \bigg(   \frac{  72 a_1  \beta   \sigma^2 (1+ \sqrt{\varepsilon})^2(1+\varepsilon)^2 }{  \Delta_{ i}^2 \gamma  }  \bigg)~\bigg)
    \\
    & 
    ~\Leftrightarrow~ 
    N_{1,T} \ge \frac{  72 \sigma^2 (1+ \sqrt{\varepsilon})^2(1+\varepsilon)}{  \Delta_{ i}^2  } \cdot \log \bigg(  \frac{ a_1 }{ \gamma } \log \bigg(   \frac{  72 a_1 \sigma^2 (1+ \sqrt{\varepsilon})^2(1+\varepsilon)^2 }{  \Delta_{ i}^2 \gamma  }  +  \beta    \bigg)~\bigg) +1.
\end{align*} 
We obtain the inequality in the above display by applying Lemma~\ref{lemma:ineq_lil_T}.
Meanwhile, \eqref{eq:lilUcb_ub_num_pull} and $t = \sum_{i=1}^L N_{ i,t}$ implies that
\begin{align*}
    N_{1,T} = T - \sum_{i\ne 1} N_{ i,T}
    \ge 
    T - ( L- 1) -\sum_{i\ne 1} \frac{  72 \sigma^2 (1+ \sqrt{\varepsilon})^2(1+\varepsilon)}{  \Delta_{ i}^2  } \cdot \log \bigg(  \frac{ a_1 }{ \gamma } \log \bigg(   \frac{  72 a_1 \sigma^2 (1+ \sqrt{\varepsilon})^2(1+\varepsilon)^2 }{  \Delta_{ i}^2 \gamma  }  + \beta    \bigg)~\bigg) .
\end{align*}
Altogether, we complete the proof with 
\begin{align}
    \sum_{i=1}^L 
    \frac{  72 \sigma^2 (1+ \sqrt{\varepsilon})^2(1+\varepsilon)}{  \Delta_{ i}^2  } \cdot \log \bigg(  \frac{ a_1 }{ \gamma } \log \bigg(   \frac{  72 a_1  \sigma^2 (1+ \sqrt{\varepsilon})^2(1+\varepsilon)^2 }{  \Delta_{ i}^2 \gamma  }  + \beta  \bigg)~\bigg)
    \le T-L +1 .
    \label{eq:lilUcb_condition_omega}
\end{align}

\noindent
{\bf Step 5: Conclusion.} 
%
%
%
%
Since 
\begin{align*}
	&  
    \frac{  72 \sigma^2 (1+ \sqrt{\varepsilon})^2(1+\varepsilon)}{  \Delta_{ i}^2  } \cdot \log \bigg(  \frac{ a_1 }{ \gamma } \log \bigg(   \frac{  72 a_1 \sigma^2 (1+ \sqrt{\varepsilon})^2(1+\varepsilon)^2 }{  \Delta_{ i}^2 \gamma  } +\beta  \bigg)~\bigg)
    \\&
	\le 
	\frac{  72 \sigma^2 (1+ \varepsilon )^3 }{  \Delta_{ i}^2  } \cdot \log \bigg(  \frac{ 2 a_1}{ \gamma^2 } \log \bigg(   \frac{  6 \sqrt{2a_1} \cdot \sigma  (1+\varepsilon)^2 }{  \Delta_{ i}    }    +  \beta   \bigg)~\bigg),
\end{align*}
To show \eqref{eq:lilUcb_condition_omega},
it is sufficient to have
\begin{align}
    &
    \sum_{i=1}^L 
    \frac{  72 \sigma^2 (1+ \varepsilon )^3 }{  \Delta_{ i}^2  } \cdot \log \bigg(  \frac{ 2 a_1 }{ \gamma^2 } \log \bigg(   \frac{ 6 \sqrt{2a_1} \cdot  \sigma  (1+\varepsilon)^2 }{  \Delta_{ i}    }  + \beta    \bigg)~\bigg)
    \le T-L +1 
    \nonumber \\
    \Leftrightarrow~
    &
    \sum_{i=1}^L \frac{  72 \sigma^2 (1+ \varepsilon )^3 }{  \Delta_{ i}^2  } \cdot \log \bigg(  \frac{ 2 a_1 }{ \gamma^2 }  \bigg)
    \le T-L +1 - 
    \sum_{i=1}^L \frac{  72 \sigma^2 (1+ \varepsilon )^3 }{  \Delta_{ i}^2  } \cdot \log \bigg(   \log \bigg(   \frac{  6 \sqrt{2a_1} \cdot  \sigma  (1+\varepsilon)^2 }{  \Delta_{ i}    }  + \beta     \bigg)~\bigg) 
    \nonumber \\
    \Leftrightarrow~
    &
      \gamma   
    \ge 
    \sqrt{2a_1} \cdot 
    \exp \Bigg( -
    \frac{ T-L +1 - 
    \sum_{i=1}^L
     72 \sigma^2 \Delta_{ i}^{-2} \cdot(1+ \varepsilon )^3  \log \big(   \log \big(   \frac{  6 \sqrt{2a_1} \cdot \sigma  (1+\varepsilon)^2 }{  \Delta_{ i}    }   +\beta    \big)~\big) 
         }{ 
     \sum_{i=1}^L  {  144 \sigma^2 (1+ \varepsilon )^3 }{  \Delta_{ i}^{-2}  }  }
    \Bigg).
    \nonumber
\end{align}
Recall the definition of $H_2$ in \eqref{eq:def_H_term}:
\begin{align*}
	H_2 = \sum_{i\ne 1} \frac{ 1 }{\Delta_{1,i}^2 }.
\end{align*}
Furthermore, it suffices to have
\begin{align*} 
	\gamma    
	&
    \ge 
    \sqrt{2a_1} \cdot 
    \exp \Bigg( -
    \frac{ T-L  }{ 144 \sigma^2 (1+ \varepsilon )^3 (H_2 + 1/\Delta_{1,2}^2)  }
    + \frac{1}{2} \log \bigg(   \log \bigg(   \frac{  6 \sqrt{2a_1} \cdot  \sigma  (1+\varepsilon)^2 }{  \Delta_{ 1,2}    }  +\beta  \bigg)~\bigg) 
    ~
    \Bigg)
    \\&
    =
    \sqrt{   2.8  \cdot \log \bigg(   \frac{  6 \sqrt{2.8} \cdot  \sigma  (1+\varepsilon)^2 }{  \Delta_{ 1,2}    }  +\beta   \bigg) }
    \exp \bigg( -
    \frac{ T-L  }{ 144 \sigma^2 (1+ \varepsilon )^3 (H_2 + 1/\Delta_{1,2}^2)  } 
    \bigg)
    := \gamma_1 (\Delta_{1,2} , H_2) 
    . 
\end{align*}   
Note that $\Delta=\Delta_{1,2}$.
When $\gamma = \gamma_1 (\Delta , H_2) $,
\begin{align*} 
    e_T 
    &
    \le \frac{ 2L(2+\varepsilon)}{ \varepsilon }  \bigg(  \frac{ \gamma_1 (\Delta , H_2) }{ \log(1+\varepsilon) }  \bigg)^{1+\varepsilon}
    \\&
    = \frac{ 2L(2+\varepsilon)   }{ \varepsilon [\log(1+\varepsilon)]^{1+\varepsilon} }  
    \cdot  
    \bigg[  2.8  \log \bigg(   \frac{  6 \sqrt{2.8} \cdot \sigma  (1+\varepsilon)^2 }{  \Delta    }  +\beta    \bigg) \bigg]^{(1+\varepsilon)/2} 
    \cdot
    \exp \bigg( -
    \frac{ T-L  }{ 144 \sigma^2 (1+ \varepsilon )^2 (H_2 + 1/\Delta^2)  } 
    \bigg) .
\end{align*}
%
%
 \end{proof}

 \subsection{Proof of Lemma~\ref{lemma:lilUcb_conc} }
\lemmaLilUCBConc*
 \begin{proof} 
Let
 \begin{align*}
     \calE_{i,\gamma} : = \{ \forall t \geq L,  
     | \hat{g}_{ i,t } - w_i | \le C_{  i,t,\gamma }  \}.
 \end{align*}
Then 
\begin{align*}
    \Pr( \calE_{i,\gamma})
    &
    = \Pr \bigg(\forall t \geq L, \big| \hat{g}_{i,t} -  w_i \big| \leq 5\sigma  (1+ \sqrt{\varepsilon} ) \sqrt{ \frac{ 2 (1+\varepsilon)}{N_{i,t }} \cdot \log  \bigg(  \frac{  \log(   \beta +    (1+\varepsilon) N_{i,t-1} ) }{\gamma} \bigg)  }   ~ \bigg) 
    \\&
    = \Pr \bigg(\forall N_{i,t } \geq 1,
    \\& \hspace{4em}
         \bigg| \bigg( \frac{ 1 }{ N_{i,t} }  \sum_{ u=1 }^{t} g_{i,u} \cdot  \mathsf{1}\{ i_u=i \}  \bigg) -  w_i \bigg| \leq 5\sigma  (1+ \sqrt{\varepsilon} ) \sqrt{ \frac{ 2 (1+\varepsilon)}{N_{i,t }} \cdot \log  \bigg(  \frac{  \log(  \beta  +    (1+\varepsilon) N_{i,t } ) }{\gamma} \bigg)  }   ~ \bigg) 
\end{align*}
When $\varepsilon\in (0,1)$ and  $\gamma \in (0, \log( \beta+   1+\varepsilon)/e ) $, Theorem~\ref{thm:conc_log} indicates that 
$$\Pr( \calE_{i,\gamma}) \ge 1- \frac{2(2+\varepsilon)}{ \varepsilon }  \bigg(  \frac{ \gamma }{ \log(1+\varepsilon) }  \bigg)^{1+\varepsilon}.$$  
Furthermore,
\begin{align*}
    &\Pr \bigg(  \bigcap_{i=1}^L  \calE_{i,\gamma} \bigg)
    = 1 - \Pr \Bigg( ~ \overline{ \bigcap_{i=1}^L  \calE_{i,\gamma} } ~\Bigg)
    = 1 - \Pr \bigg(   \bigcup_{i=1}^L \overline{   \calE_{i,\gamma}  } \bigg) \\
    &\quad\ge 1 -  \sum_{i=1}^L \Pr \big(~   \overline{ \calE_{i,\gamma}  } ~\big)
    \ge 1- \frac{ 2L(2+\varepsilon)}{ \varepsilon }  \bigg(  \frac{ \gamma }{ \log(1+\varepsilon) }  \bigg)^{1+\varepsilon}.
\end{align*} 
 \end{proof}

\subsection{Proof of Lemma \ref{lemma:ineq_lil_T}}
\label{pf:lemma_ineq_lil_T}

\lemmaIneqLilT*

\begin{proof}
Let 
\begin{align*}
    f(\tau )  =  c \log \bigg(  \frac{ \log(a \tau + b ) }{\rho} \bigg), 
    \quad     
    \tau_{a_1,a_2} =  c \log \bigg( \frac{ a_1 }{ \rho } \log  \bigg( \frac{a_2 c}{\rho } + b \bigg) \bigg).
\end{align*}
Then
\begin{align*}
    \tau_{a_1,a_2} \ge f( \tau_{a_1,a_2} )  
    ~\Leftrightarrow~ &
    c \log \bigg( \frac{ a_1 }{ \rho } \log  \bigg( \frac{a_2 c}{\rho } + b \bigg) \bigg)
    \ge  c \log \bigg(  \frac{ 1 }{\rho}   \log\bigg[ac \log \bigg( \frac{ a_1 }{ \rho } \log  \bigg( \frac{a_2 c}{\rho } \bigg) \bigg) + b\bigg] \bigg)
    \\ \Leftrightarrow~ &
    a_1 \log  \bigg( \frac{a_2 c}{\rho } + b \bigg)  
    \ge   \log\bigg[ac \log \bigg( \frac{ a_1 }{ \rho } \log  \bigg( \frac{a_2 c}{\rho } \bigg) \bigg)+b \bigg] .
\end{align*}
Let $a_1\ge 1.4$, then  $x^{a_1} \ge x \log x$ for all $x \ge 1$. To obtain $\tau_{a_1,a_2} \ge f( \tau_{a_1,a_2} )  $, it suffices to have
\begin{align*}
      &
     \bigg( \frac{a_2 c}{\rho } + b \bigg)  
     \cdot \log  \bigg( \frac{a_2 c}{\rho } + b \bigg) 
      \ge  
     ac \log \bigg( \frac{ a_1 }{ \rho } \log  \bigg( \frac{a_2 c}{\rho } \bigg) \bigg)+b,
\end{align*}
which is implied by 
\begin{align*}
       \frac{a_2 c}{\rho } + b \ge e
       ~\text{ and }~
        \frac{a_2 c}{\rho }  
     \cdot \log  \bigg( \frac{a_2 c}{\rho } + b \bigg)  
      \ge  
       \frac{ ac a_1 }{ \rho } \log  \bigg( \frac{a_2 c}{\rho } \bigg) 
       .
\end{align*}
Conditioned on
$a_2c/\rho + b \ge e$,
the last inequality holds when $a_2\ge a \cdot a_1$.
Since $\tau - f(\tau)$ is monotonically increasing in $\tau$, and $ \tau_{a_1, a\cdot a_1} \ge f( \tau_{a_1, a\cdot a_1} ) $, i.e., $ \tau_{a_1, a\cdot a_1} - f( \tau_{a_1, a\cdot a_1} ) \ge 0$, we have
\begin{align*}
    \tau \ge \tau_{1.4, 1.4a }
    ~\Rightarrow~ 
    \tau - f(\tau)  \ge 0.
\end{align*}
In other words, when
$1.4ac/\rho + b \ge e$,
\begin{align*}
    \tau \le f(\tau) 
    ~\Rightarrow~
    \tau \le \tau_{1.4, 2.8} = c \log \bigg( \frac{ 1.4 }{ \rho } \log  \bigg( \frac{ 1.4a   c}{\rho } \bigg) \bigg).
\end{align*} 
%
%
%
%
\end{proof}

 \section{Analysis of the Pareto frontier of RM and BAI in stochastic bandits}
\label{append:sto_analyze_trade_off}

%
%

\subsection{Proof of Theorem~\ref{thm:sto_rm_bai_bern} }

%
%

\label{pf:thm_sto_rm_bai_bern}

\thmStoRmBaiBern*

 \begin{proof}

{\bf Step 1: Construct instances.}
To begin with, we fix $d_\ell \in (0, 1/4]$ for all $2\le \ell \le L$. We let $\mathrm{Bern}(a)$ denote the Bernoulli distribution with parameter $a$.
We define the following distributions:
\begin{align*}
    & \nu_1 :=\mathrm{Bern}( 1/2  ),  \quad \nu_\ell := \mathrm{Bern}(  1/2 - d_\ell  ) \quad \forall 1 < \ell \le L;
    \\&
    \nu_1' := \mathrm{Bern}( 1/2  ),  \quad \nu_\ell' := \mathrm{Bern}( 1/2 + d_\ell ) \quad \forall 1 < \ell \le L.
\end{align*}
We construct $L$ instances such that under instance $\ell$~($1\le \ell \le L$), the stochastic reward of item $i$ is drawn from distribution 
$$ \nu_i^\ell :=  b \cdot ( \nu_i \mathsf{1}\{i\ne \ell\} + \nu_i'\mathsf{1}\{ i = \ell \}  ),$$
where $b>0$.
Under instance $\ell$~($1\le \ell \le L$), we see item $\ell$ is optimal, and we define several other notations as follows: 
\begin{itemize}
    \item[(i)] We let $g_{i,t}^\ell  $ be the random reward of item $i$ at time step $t$. Then $g_{i,t}^\ell  \in \{0, b\} $.
    \item[(ii)] We let $ \Delta_{i,j}^\ell : = \bbE[ \sum_{t=1}^T g_{i,t}^\ell - g_{k,t}^\ell ] /T $ denote the gap between item $i$ and $j$. Then
        \begin{align*}
            \Delta_{1,j}^1 = b \cdot d_j \ \quad \forall 2\le j \le L, \quad
            \Delta_{\ell,1}^\ell = b \cdot  d_\ell, \ \Delta_{\ell,j}^\ell = b \cdot d_\ell + b \cdot d_j \quad \forall 2\le j,\ell \le L, j \neq \ell.
        \end{align*}
    \item[(iii)] We denote the difficulty of the instance with 
        \begin{align*}
            H_2(\ell)  := \sum_{ j\ne \ell }  (\Delta_{\ell, j}^\ell)^{-2} .
        \end{align*}
        Then $H_2(1) = \max\limits_{1\le \ell \le L} H_2(\ell)\le (L-1) b^{-2} \cdot \max\limits_{2\le \ell \le L }d_\ell^{-2 }$.
    \item[(iv)] We let $i_t^\ell$ be the pulled item at time step $t$, and $O_\ell^t = \{ i_u^\ell, g_{  i_u^\ell,u } \}^t_{u=1}$ be the sequence of pulled items and observed rewards up to and including time step $t$.
    \item[(v)] We let $\bbP_\ell^t $ be the measure on $ O_\ell^t$, and let $P_{\ell,i}$ be the measure on the rewards of item $i$.
\end{itemize} 
For simplicity, we abbreviate $\bbP_\ell^T$, $O_\ell^T$ as $\bbP_\ell$, $O_\ell$ respectively. 
%
Moreover,
we let $N_{i,t}$ denote the number of pulls of item $i$ up to and including time step $t$.

\noindent {\bf Step 2: Change of measure.} 
%
%
 %
First of all, we apply Lemmas~\ref{lemma:kl_to_event} and \ref{lemma:kl_decomp} to obtain that for all $1\le \ell \le L$, 
 \begin{align*}
     \Pr_{ O_1 }(\iout \neq 1  ) +   \Pr_{O_\ell} (\iout = 1  )
     \ge \frac{ 1}{2 } \exp( - \mathrm{KL}(\bbP_1  \parallel \bbP_{\ell}  ) )
     = \frac{ 1}{2 } \exp(   - \bbE_{\bbP_1 }  [ N_{ \ell ,T} ] \cdot \mathrm{KL} ( P_{1,\ell} \parallel P_{\ell, \ell} )    ).
\end{align*}          

Suppose the pseudo-regret is upper bounded by $\overline{\mathrm{Reg}}$, we have
\begin{align*}
    \overline{\mathrm{Reg}}
    & \ge \bbE_{ \bbP_1 } \Bigg[  \sum_{t=1}^T \mathsf{1}\{ i_t^1 \neq 1 \} \cdot ( g_{1,t}^1 - g_{i_t,t}^1 ) \Bigg]
    =   \sum_{t=1}^T \bbE_{ \bbP_1^t }  [   \mathsf{1}\{ i_t^1 \neq 1 \} \cdot ( g_{1,t}^1 - g_{i_t,t}^1 )   ] 
    \\&
    =   \sum_{\ell=2}^L \sum_{t=1}^T \bbE_{ \bbP_1^t }  [   \mathsf{1}\{ i_t^1 = \ell \} \cdot ( g_{1,t}^1 - g_{i_t,t}^1 )   ] 
    =   \sum_{\ell=2}^L \sum_{t=1}^T \bbE_{ \bbP_1^t }  [    ( g_{1,t}^1 - g_{i_t,t}^1 )  | i_t^1 = \ell   ]  \cdot  \bbE_{ \bbP_1^t }  [   \mathsf{1}\{ i_t^1= \ell \}  ]
    \\&
    =   \sum_{\ell=2}^L \sum_{t=1}^T  \Delta_{1,\ell}^1  \cdot  \bbE_{ \bbP_1^t }  [   \mathsf{1}\{ i_t^1= \ell \}  ]
    =   \sum_{\ell=2}^L b \cdot d_\ell  \cdot \bbE_{ \bbP_1 }[ N_{\ell,T} ] .
\end{align*}       
Since $H_2(\ell)  = \sum_{ j\ne \ell }  (\Delta_{\ell, j}^\ell)^{-2}$, we have
\begin{align*}   
    &
    \frac{ \overline{\mathrm{Reg}}  }{  H_2(1) }
    = \frac{  \sum_{\ell=2}^L b \cdot  d_\ell  \cdot \bbE_{ \bbP_1 }[ N_{\ell,T} ]   }{  \sum_{\ell=2}^L  (\Delta_{1, j}^1)^{-2} }
    = \frac{  \sum_{\ell=2}^L b \cdot d_\ell  \cdot \bbE_{ \bbP_1 }[ N_{\ell,T}]   }{  b^{-3} \cdot \sum_{\ell=2}^L  d_ j^{-2} }.
\end{align*}
Thus, by the pigeonhole principle, there exists $  2\le \ell_1 \le L $ such that
\begin{align*}
    b^3 d_{\ell_1}^3\bbE_{ \bbP_1 }  \cdot  [ N_{\ell_1,T} ] 
    = \frac{ b \cdot d_{\ell_1}  \cdot \bbE_{ \bbP_1 }[ N_{\ell_1,T}] }{ b^{-2} \cdot d_{\ell_1}^{-2} } \le    \frac{ \overline{\mathrm{Reg}}  }{  H_2(1) } 
    ~\Leftrightarrow~ 
    \bbE_{ \bbP_1 }[ N_{\ell_1,T} ]   \le    \frac{ \overline{\mathrm{Reg}}  }{  b^3 d_{\ell_1}^3 H_2(1)  }  .
\end{align*} 
%

Since $d_\ell \in (0,1/4]$ for all $2\le \ell \le L$, we apply Theorem~\ref{thm:pinsker} to obtain
%
%
%
 \begin{align*}
     &
     \Pr_{ O_1 }(\iout \neq 1  ) +   \Pr_{O_{\ell_1} } (\iout = 1  ) 
     \ge 
     \frac{ 1}{2 } \exp(   - \bbE_{\bbP_1 }  [ N_{\ell_1,T} ] \cdot \mathrm{KL} ( P_{1,\ell_1 } \parallel P_{\ell_1, \ell_1 } )    )
     \ge
     \frac{ 1}{2 } \exp\bigg(   -   \frac{    \overline{\mathrm{Reg}} }{ b^3 d_{\ell_1}^3  H_2(1)  }   \cdot  \frac{ (2d_{\ell_1})^2 }{ 1/4 }  \bigg)
     .
\end{align*}     

 Since $\Pr_{O_{\ell_j} } (\iout \neq \ell_1 ) \ge \Pr_{O_{\ell_1} } (\iout = 1  )$,
we have
\begin{align*}
    \max_{ 1 \le \ell \le L } \Pr_{O_\ell} (\iout \neq \ell ) \ge 
    \frac{ 1}{ 4 } \exp\bigg(   -   \frac{  8 \overline{\mathrm{Reg}} }{   H_2(1)  b^3 \cdot \min\limits_{2\le \ell \le L} d_\ell}    ~ \bigg)
    .
\end{align*} 

\noindent{\bf Step 3: Conclusion.}
We define 
$$\ell_2:=
\argmax_{ 1\le \ell \le L} \Pr_{O_\ell} (\iout \neq \ell ). 
$$
Suppose algorithm $\pi$ satisfies that
\begin{align*}
    &
    \Pr_{O_{\ell_2} } (\iout \neq \ell_2 )
    \le 
    \frac{ 1}{ 4 } \exp (   -   \phi_{T}   ),
\end{align*}
then we have
\begin{align*}
     \overline{\mathrm{Reg}} 
     \ge  \phi_{T}   \cdot    \frac{ H_2(1) b^3 \cdot \min\limits_{2\le \ell \le L} d_\ell }{8}.
\end{align*}
When $d = d_{\ell}>0$ for all $2\le \ell \le L$, we have $H_2(1) =(L-1)/(b^2 d^2)$.

{\bf Step 4: Classification of instances.}
Suppose algorithm $\pi$ satisfies that $e_T(\pi)
\le  \exp( -  \phi_{T} ) /4$. 
Let $\calB_1( \DeltaMin, \RMax  )$ denote the set of  stochastic instances where
(i) the minimal optimality gap $\Delta\ge \DeltaMin$;
and (ii) there exists $R_0\in \bbR$ such the rewards are bounded in $[R_0,R_0+\RMax]$.
Then
\begin{align*}
	\sup_{ \calI \in \calB_1( \DeltaMin , \RMax  ) }
	R_T (\pi,\calI )
	\ge 
	 \phi_{T}   \cdot  \frac{  (L-1)\RMax    }{ 8\DeltaMin  } 
	\quad 
	\forall \DeltaMin, \RMax  >0.
\end{align*} 
Let $\calB_2(\DeltaMin, \RMax , \HtwoMax )$ denote the set of stochastic instances 
that
(i) belong to $\calB_1(\DeltaMin, \RMax)$,
and
(ii)
are with hardness  parameter $H_2\le  \HtwoMax $. 
Then, we have
\begin{align*}
    \sup_{ \calI \in \calB_2( \DeltaMin, \RMax , \HtwoMax  ) }
	R_T(\pi,\calI  )
	\ge  
     \phi_{T}   \cdot  \frac{ \DeltaMin  \HtwoMax \RMax^3 }{8} 
	\quad 
	\forall \DeltaMin, \RMax , \HtwoMax >0.
\end{align*}
\end{proof}

\subsection{Proof of Theorem~\ref{thm:sto_rm_bai_gauss} }

\label{pf:thm_sto_rm_bai_gauss}

\thmStoRmBaiGauss*
 
 \begin{proof}

{\bf Step 1: Construct instances.}
To begin with, we fix any $ \sigma>0$, $d_\ell >0$ for all $2\le \ell \le L$.
We define the following distributions:
\begin{align*}
    & \nu_1 := \calN( 1/2 , \sigma^2 ),  \quad \nu_\ell := \calN( 1/2 - d_\ell, \sigma^2  ) \quad \forall 1 < \ell \le L;
    \\&
    \nu_1' := \calN( 1/2, \sigma^2 ),  \quad \nu_\ell' := \calN( 1/2 + d_\ell, \sigma^2  ) \quad \forall 1 < \ell \le L.
\end{align*}
We construct $L$ instances such that under instance $\ell$~($1\le \ell \le L$), the stochastic reward of item $i$ is drawn from distribution 
$$ \nu_i^\ell := \nu_i \mathsf{1}\{i\ne \ell\} + \nu_i'\mathsf{1}\{ i = \ell \} .$$
Under instance $\ell$~($1\le \ell \le L$), we see item $\ell$ is optimal, and we define several other notations as follows: 
\begin{itemize}
    \item[(i)] We let $g_{i,t}^\ell$ be the random reward of item $i$ at time step $t$.
    \item[(ii)] We let $ \Delta_{i,j}^\ell : = \bbE[ \sum_{t=1}^T g_{i,t}^\ell - g_{k,t}^\ell ] /T $ denote the gap between item $i$ and $j$. Then
        \begin{align*}
            \Delta_{1,j}^1 = d_j \ \quad \forall 2\le j \le L, \quad
            \Delta_{\ell,1}^\ell = d_\ell, \ \Delta_{\ell,j}^\ell = d_\ell + d_j \quad \forall 2\le j,\ell \le L, j \neq \ell.
        \end{align*}
    \item[(iii)] We denote the difficulty of the instance with 
        \begin{align*}
            H_2(\ell)  := \sum_{ j\ne \ell }  (\Delta_{\ell, j}^\ell)^{-2} .
        \end{align*}
        Then $H_2(1) = \max\limits_{1\le \ell \le L} H_2(\ell)\le (L-1) \cdot \max\limits_{2\le \ell \le L }d_\ell^{-2 }$.
    \item[(iv)] We let $i_t^\ell$ be the pulled item at time step $t$, and $O_\ell^t = \{ i_u^\ell, g_{  i_u^\ell,u } \}^t_{u=1}$ be the sequence of pulled items and observed rewards up to and including time step $t$.
    \item[(v)] We let $\bbP_\ell^t $ be the measure on $ O_\ell^t$, and let $P_{\ell,i}$ be the measure on the rewards of item $i$.
\end{itemize} 
For simplicity, we abbreviate $\bbP_\ell^T$, $O_\ell^T$ as $\bbP_\ell$, $O_\ell$ respectively. 
%
Moreover,
we let $N_{i,t}$ denote the number of pulls of item $i$ up to and including time step $t$.

\noindent {\bf Step 2: Change of measure.} 
%
%
 %
First of all, we apply Lemmas~\ref{lemma:kl_to_event} and \ref{lemma:kl_decomp} to obtain that for all $1\le \ell \le L$, 
 \begin{align*}
     \Pr_{ O_1 }(\iout \neq 1  ) +   \Pr_{O_\ell} (\iout = 1  )
     \ge \frac{ 1}{2 } \exp( - \mathrm{KL}(\bbP_1  \parallel \bbP_{\ell}  ) )
     = \frac{ 1}{2 } \exp(   - \bbE_{\bbP_1 }  [ N_{ \ell ,T} ] \cdot \mathrm{KL} ( P_{1,\ell} \parallel P_{\ell, \ell} )    ).
\end{align*}          

Suppose the pseudo-regret is upper bounded by $\overline{\mathrm{Reg}}$, we have
\begin{align*}
    \overline{\mathrm{Reg}}
    & \ge \bbE_{ \bbP_1 } \Bigg[  \sum_{t=1}^T \mathsf{1}\{ i_t^1 \neq 1 \} \cdot ( g_{1,t}^1 - g_{i_t,t}^1 ) \Bigg]
    =   \sum_{t=1}^T \bbE_{ \bbP_1^t }  [   \mathsf{1}\{ i_t^1 \neq 1 \} \cdot ( g_{1,t}^1 - g_{i_t,t}^1 )   ] 
    \\&
    =   \sum_{\ell=2}^L \sum_{t=1}^T \bbE_{ \bbP_1^t }  [   \mathsf{1}\{ i_t^1 = \ell \} \cdot ( g_{1,t}^1 - g_{i_t,t}^1 )   ] 
    =   \sum_{\ell=2}^L \sum_{t=1}^T \bbE_{ \bbP_1^t }  [    ( g_{1,t}^1 - g_{i_t,t}^1 )  | i_t^1 = \ell   ]  \cdot  \bbE_{ \bbP_1^t }  [   \mathsf{1}\{ i_t^1= \ell \}  ]
    \\&
    =   \sum_{\ell=2}^L \sum_{t=1}^T  d_\ell  \cdot  \bbE_{ \bbP_1^t }  [   \mathsf{1}\{ i_t^1= \ell \}  ]
    =   \sum_{\ell=2}^L d_\ell  \cdot \bbE_{ \bbP_1 }[ N_{\ell,T} ] .
\end{align*}       
Since $H_2(\ell)  = \sum_{ j\ne \ell }  (\Delta_{\ell, j}^\ell)^{-2}$, we have
\begin{align*}   
    &
    \frac{ \overline{\mathrm{Reg}}  }{  H_2(1) }
    = \frac{  \sum_{\ell=2}^L d_\ell  \cdot \bbE_{ \bbP_1 }[ N_{\ell,T} ]   }{  \sum_{\ell=2}^L  (\Delta_{1, j}^1)^{-2} }
    = \frac{  \sum_{\ell=2}^L d_\ell  \cdot \bbE_{ \bbP_1 }[ N_{\ell,T}  ]   }{  \sum_{\ell=2}^L  d_ j^{-2} }.
\end{align*}
Thus, by the pigeonhole principle, there exists $  2\le \ell_1 \le L $ such that
\begin{align*}
    d_{\ell_1}^3\bbE_{ \bbP_1 }  \cdot  [ N_{\ell_1,T} ] 
    = \frac{ d_{\ell_1}  \cdot \bbE_{ \bbP_1 }[ N_{\ell_1,T} ] }{d_{\ell_1}^{-2} } \le    \frac{ \overline{\mathrm{Reg}}  }{  H_2(1) } 
    ~\Leftrightarrow~ 
    \bbE_{ \bbP_1 }[ N_{\ell_1,T}  ]   \le    \frac{ \overline{\mathrm{Reg}}  }{  d_{\ell_1}^3 H_2(1)  }  .
\end{align*} 
%

Further, we apply Lemma~\ref{lemma:kl_gauss} to obtain
%
%
%
 \begin{align*}
     &
     \Pr_{ O_1 }(\iout \neq 1  ) +   \Pr_{O_{\ell_1} } (\iout = 1  ) 
     \ge 
     \frac{ 1}{2 } \exp(   - \bbE_{\bbP_1 }  [ N_{\ell_1,T} ] \cdot \mathrm{KL} ( P_{1,\ell_1 } \parallel P_{\ell_1, \ell_1 } )    )
     \ge
     \frac{ 1}{2 } \exp\bigg(   -   \frac{    \overline{\mathrm{Reg}} }{ d_{\ell_1}^3  H_2(1)  }   \cdot  \frac{ (2d_{\ell_1})^2 }{ 2\sigma^2 }  \bigg)
     .
\end{align*}

Since $\Pr_{O_{\ell_j} } (\iout \neq \ell_1 ) \ge \Pr_{O_{\ell_1} } (\iout = 1  )$,
we have
\begin{align*}
    \max_{ 1 \le \ell \le L } \Pr_{O_\ell} (\iout \neq \ell ) \ge 
    \frac{ 1}{ 4 } \exp\bigg(   -   \frac{  2 \overline{\mathrm{Reg}} }{   H_2(1)\sigma^2 \cdot \min\limits_{2\le \ell \le L} d_\ell}    ~ \bigg)
    .
\end{align*} 

\noindent{\bf Step 3: Conclusion.} 
We define 
$$\ell_2:=
\argmax_{ 1\le \ell \le L} \Pr_{O_\ell} (\iout \neq \ell ). 
$$
Suppose algorithm $\pi$ satisfies that
\begin{align*}
    &
    \Pr_{O_{\ell_2} } (\iout \neq \ell_2 )
    \le 
    \frac{ 1}{ 4 } \exp (   -   \phi_{T}   ),
\end{align*}
then we have
\begin{align*}
     \overline{\mathrm{Reg}} 
     \ge  \phi_{T}   \cdot  \frac{   H_2(1)\sigma^2 \cdot \min\limits_{2\le \ell \le L} d_\ell }{2}.
\end{align*}
When $d = d_{\ell}>0$ for all $2\le \ell \le L$, we have $H_2(1) =(L-1)/d^2$.

{\bf Step 4: Classification of instances.}
%
Suppose algorithm $\pi$ satisfies that $e_T(\pi)
\le  \exp( -  \phi_{T} ) /4$. 
Let $\calB_1'(\DeltaMin, \varMax )$ denote the set of  stochastic instances where 
(i) the minimal optimality gap $\Delta\ge \DeltaMin $;
(ii) for each item $i$, the variance $\sigma_i^2\le \varMax$. 
Then
\begin{align*}
	\sup_{ \calI \in \calB_1'( \DeltaMin, \varMax ) }
	R_T (\pi, \calI )
	\ge 
	 \phi_{T}   \cdot  \frac{ (L-1)  \varMax }{2 \DeltaMin } 
	\quad 
	\forall \DeltaMin, \varMax>0.
\end{align*} 
Let $\calB_2'(\DeltaMin, \varMax,\HtwoMax)$ denote the set of stochastic instances (i) that belong to  $\calB_1'(\DeltaMin, \varMax )$, and (ii) 
are with the hardness $H_2\le \HtwoMax$. 
We have
\begin{align*}
    \sup_{ \calI \in \calB_2'( \DeltaMin, \varMax,\HtwoMax ) }
	R_T (\pi, \calI)
	\ge  
     \phi_{T}   \cdot   \frac{ \DeltaMin \HtwoMax \varMax   }{2  }
	\quad 
	\forall \DeltaMin, \varMax,\HtwoMax>0.
\end{align*}
\end{proof}

\subsection{Proof of Corollary \ref{coro:rm_lower_bd_bern_lilUcb_para}}
\label{pf:coro_rm_lower_bd_bern_lilUcb_para_bern} 
 \coroRmLowerBdBernLilUcbPara*
\begin{proof}
    We consider the  stochastic instances in $\calB_2(\DeltaMin,1 ,\HtwoMax )$.
By the classification of instances in Theorem~\ref{thm:sto_rm_bai_bern}, these instances satisfy the conditions 
\begin{align}
    g_{i,t} \in [0,1] \quad \forall i,t,
    \quad \text{ and }  \quad
    H_2    \le \HtwoMax.
    \nonumber
\end{align}   
Therefore, the distribution $\nu_i$ is sub-Gaussian with scale $\sigma = 1/2$ for all $i\in[L]$. 
We assume $T $ is sufficiently large such that
\begin{align*} 
	 \frac{\log T }{T} 
	&
    \ge  
    \gamma_1
    =
       \sqrt{ 2.8 \log \bigg(   \frac{  6 \sqrt{2.8} \sigma  (1+\varepsilon)^2 }{  \DeltaMin   }   +\beta    \bigg) }
       \cdot 
    \exp \Bigg( -
    \frac{ T-L  }{ 144 \sigma^2 (1+ \varepsilon )^3 ( \HtwoMax + 1/\DeltaMin^2)  } 
    \Bigg)
    \\&
    =
     \sqrt{ 2.8 \log \bigg(   \frac{  3 \sqrt{2.8} (1+\varepsilon)^2 }{  \DeltaMin    }  + \beta      \bigg)  } \cdot
    \exp \Bigg( -
    \frac{  T-L   }{ 36 (1+ \varepsilon )^3 (\HtwoMax + \DeltaMin^{-2})  } 
    \Bigg). 
\end{align*}
As a result, for all instance in $\calB_2( \DeltaMin,1, \HtwoMax )$, since $\Delta\ge \DeltaMin$ and $H_2\le \HtwoMax$, we have
\begin{align*} 
	\frac{\log T }{T } 
	&
    \ge   
     \sqrt{ 2.8 \log \bigg(   \frac{  3 \sqrt{2.8}  (1+\varepsilon)^2 }{  \Delta    }   +\beta    \bigg) } \cdot
    \exp \Bigg( -
    \frac{ T-L   }{ 36  (1+ \varepsilon )^3 (H_2 + \Delta^{-2})  } 
    \Bigg). 
\end{align*}

Fix any $\gamma\in [ \gamma_1, (\log T)/L ]$. 
On one hand, for any instance in $\calB_2( \DeltaMin,1, \HtwoMax )$, Theorem~\ref{thm:rm_bd_lilUcb_para} implies that  
{\sc BoBW-lil'UCB$(\gamma )$} satisfies that
\begin{align*}
    R_T \le O \bigg(  \log\bigg(\frac{1}{\gamma} \bigg) \cdot H_1 \bigg).
\end{align*}
On the other hand, Theorem~\ref{thm:bai_bd_lilUcb_para} implies that 
\begin{align*}
    e_T  
    &
    \le \frac{ 2L(2+\varepsilon)}{ \varepsilon }  \bigg(  \frac{ \gamma}{  \log(1+\varepsilon) }  \bigg)^{1+\varepsilon}  
    . 
\end{align*}
Moreover, we can apply Theorem~\ref{thm:sto_rm_bai_bern} to obtain that
\begin{align*}
    \sup_{ \calI \in \calB_2( \DeltaMin, 1, \HtwoMax ) } R_T( \textsc{BoBW-lil'UCB$(\gamma )$}, \calI) 
    &
    \in
    \Omega 
    \bigg(
      \DeltaMin \HtwoMax   \log \bigg( \frac{1}{\gamma L} \bigg) 
    \bigg) 
    .
\end{align*} 
Altogether, we have  
\begin{align*}
    \sup_{ \calI \in \calB_2( \DeltaMin, 1, \HtwoMax ) } R_T( \textsc{BoBW-lil'UCB$(\gamma )$}, \calI) 
    &
    \in
    \Omega 
    \bigg(
      \DeltaMin \HtwoMax   \log \bigg( \frac{1}{\gamma L} \bigg) 
    \bigg) 
    \bigcap  
    O \bigg( \frac{  (L-1)  }{ \DeltaMin }  \log \bigg( \frac{1}{\gamma  } \bigg)  \bigg).
\end{align*}     
\end{proof}

\section{Analysis of {\sc Exp3.P} in adversarial bandits}
\label{pf:rm_bai_expThreeP_beta_zero}

\subsection{Proof of Theorem~\ref{thm:rm_bd_expThreeP_beta_zero}}
\label{pf:thm_rm_bd_expThreeP_beta_zero}

\thmRmBdExpThreePBetaZero*

\begin{proof}

The analysis is similar to that of 
Theorem 3.2 in \cite{bubeck2012regret} with $\beta=0$,
while the following lemma signifies the key difference.
\begin{restatable}[Implied by \citet{bubeck2012regret}, Lemma 3.1]{lemma}{lemmaRmRev}
\label{lemma:rm_rev}
For any item $i$, with probability at least $1-\delta$,
\begin{align*}
    \sum_{t=1}^T g_{i,t} \le \sum_{t=1}^T \tilde{g}_{i,t}  + \ln  \bigg(   \frac{    LT }{  \eta \delta } \bigg).
\end{align*}

\end{restatable}
Replacing Lemma 3.1 in \citet{bubeck2012regret} by Lemma~\ref{lemma:rm_rev},
we can adopt the analysis of Theorem 3.2 in \cite{bubeck2012regret} and show that with probability $1-\delta$,
\begin{align*}
    \barR_T  \le  \gamma T +  \eta L T +  \ln  \bigg(   \frac{    L^2 T }{  \eta \delta } \bigg) +  \frac{ \ln L }{ \eta}.
\end{align*}
Moreover, we derive that
\begin{align*}
    &
    W' =  \barR_T -\bigg [  \gamma T +  \eta L T +  \ln  \bigg(   \frac{    L^2 T }{  \eta   } \bigg) +  \frac{ \ln L }{ \eta} \bigg],
    \\&
    \bbP( W' > \ln \frac{1}{\delta} )
    = \Pr \bigg( \barR_T - \bigg[   \gamma T + \eta L T +  \ln  \bigg(   \frac{    L^2 T }{  \eta \delta   } \bigg) +  \frac{ \ln L }{ \eta} \bigg]  > 0 \bigg)
    \le \delta,
    \\&
    \bbE \barR_T -\bigg [     \gamma T +  \eta L T +  \ln  \bigg(   \frac{    L^2 T }{  \eta   } \bigg) +  \frac{ \ln L }{ \eta} \bigg] \le 1,
    \\&
    \bbE \barR_T \le     \gamma T + \eta L T +  \ln  \bigg(   \frac{    L^2 T }{  \eta   } \bigg) +  \frac{ \ln L }{ \eta}  + 1.
\end{align*}
%
%
\end{proof}

\subsection{Proof of Theorem~\ref{thm:bd_bai_expThreeP_zero}}
\label{sec:analyze_bd_bai_expThreeP}

\thmBdBaiExpThreePZero*

\begin{proof}
 For brevity, we assume $G_{1,T} \ge G_{2,T} \ge \ldots \ge G_{L,T}$ and abbreviate $\barDelta_{i,j,T}$ as $\Delta_{i,j}$ for any $i,j \in [L]$.
Consequently, the optimal item $\bari^* = 1$.

\textbf{Step 1: Construction of martingale.}
Let 
\begin{align*}
    y_{i,t} = \tilde{g}_{i,t}  - g_{i,t}    , \quad
    X_{i,t} = \tilde{G}_{i,t}  - G_{i,t}     = \sum_{u=1}^t   ( \tilde{g}_{i,u}  - g_{i,u}  ) = \sum_{u=1}^t y_{i,u}.
\end{align*}
Now we fix arbitrary $i\in[L]$ and abbreviate $y_{i,t}$ as $y_t$, $X_{i,t}$ as $X_t$ for brevity when there is no ambiguity.
Then we have
\begin{align*}
    &
    X_{t} - X_{t-1} = y_{ t} =  \tilde{g}_{i,t}  - g_{i,t} 
    =    g_{i,t} \cdot \bigg( \frac{   \mathbb{I} \{ i_t = i \}    }{  p_{i,t} }  - 1\bigg)   
    =  g_{i,t} \cdot \bigg( \frac{   \mathbb{I} \{ i_t = i \}    }{  p_{i,t} }  - 1\bigg) ,
    \\&
     -1 \le  y_{ t}
    \le \frac{1   }{ p_{i,t} } - 1
    \quad\text{ since } g_{i,t}  \le 1,
    \\&
    \bbE [y_t | \calF_{t-1}] = \bbE \bigg[ g_{i,t} \cdot \bigg( \frac{  \bbE[ \mathbb{I} \{ i_t = i \} | \calF_{t-1}   ]  }{  p_{i,t} }  - 1\bigg)  \bigg| \calF_{t-1} \bigg] = 0,
        \quad 
        \bbE y_t = \bbE [ \bbE [y_t | \calF_{t-1}] ] = 0
    .
\end{align*}
Since $y_t$ is a martingale, we can apply Theorem~\ref{thm:conc_var_ub} and \ref{thm:conc_var_lb} for the analysis.

Meanwhile, note that $g_{i,t}, p_{i,t} \in \calF_{t-1}$ and again $g_{i,t}  \le 1$. Since $\bbP( i_t=i | \calF_{t-1} ) = p_{i,t}$, the variance conditioned on $\calF_{t-1}$ is the variance of the Bernoulli random variable with parameter $p_{i,t}$, scaled to the range $[0, g_{i,t}/p_{i,t} ]$. Hence, we have
\begin{align*} 
    \var( X_t| \calF_{t-1}  ) 
    &
    = \var( y_t | \calF_{t-1} )
        = \var\bigg( \frac{  g_{i,t} \cdot ( \mathbb{I} \{ i_t = i \}  - 1 ) }{  p_{i,t} } ~ \bigg|~ \calF_{t-1} \bigg)
        = \var\bigg( \frac{  g_{i,t} \cdot \mathbb{I} \{ i_t = i \}  }{  p_{i,t} }~ \bigg|~ \calF_{t-1} \bigg)
        \\&
        = \var\bigg( \frac{  g_{i,t} \cdot \mathbb{I} \{ i_t = i \}   }{  p_{i,t} } ~\bigg|~ \calF_{t-1} \bigg)
        = \frac{g^2_{i,t} }{ p^2_{i,t} } \cdot \var(   \mathbb{I} \{ i_t = i \}     |  \calF_{t-1}  ) =  \frac{  p_{i,t} (1-p_{i,t})g^2_{i,t}  }{ p^2_{i,t} } =  \frac{   (1-p_{i,t})g^2_{i,t}  }{ p_{i,t} }
    \\&        
    = g_{i,t}^2 \cdot \bigg(  \frac{1}{ p_{i,t}  } -1 \bigg) \le \frac{1}{ p_{i,t}  } -1 := \sigma_t^2 
    \qquad\text{ since } g_{i,t}  \le 1
        .
    %
\end{align*}

On one hand, in order to apply Theorem~\ref{thm:conc_var_ub} to upper bound  $ \tilde{G}_{i,T}  - G_{i,T} $, we need to upper bound $\var( X_t| \calF_{t-1}  )$, $y_t$, and lower bound $p_{i,t}$. These bounds will depend on the lower bound on $p_{i,t}$.
On the other hand, to lower bound $ \tilde{G}_{i,T}  - G_{i,T} $ with Theorem~\ref{thm:conc_var_lb}, we need to upper bound $\var( X_t| \calF_{t-1}  )$, $p_{i,t}$, and lower bound $y_t$. 
%
This motivates to derive bounds on $p_{i,t}$.
Since $1-\gamma + \frac{\gamma}{L} - \frac{1}{L} =  ( 1- \frac{ 1 }{L} )( 1-\gamma ) > 0$, there are global bounds on $\{ p_{i,t}\}_{i,t}  $:
\begin{align*} 
    \frac{ \gamma }{ L } \le p_{i,t} \le 1-\gamma + \frac{\gamma}{L}.
\end{align*}

\textbf{Step 2: Bound $\tilde{G}_{i,T}  - G_{i,T}$ with high probability.} 

{\bf Upper bound on $\tilde{G}_{i,T}  - G_{i,T}$.}
%
We first derive upper bounds on $\sum_{t=1}^T \sigma^2$, $y_t$ with
lower bounds on $p_{i,t}$:
\begin{align*}
    \text{ (i) } ~
    &
    \sum_{t=1}^T \sigma_t^2 =  \sum_{t=1}^T  \bigg( \frac{1}{ p_{i,t}  } -1  \bigg)
    \le  T \cdot  \bigg( \frac{L}{  \gamma } -1  \bigg)
    \\
    \text{ (ii) } ~
    &
    y_{t} 
    \le \frac{  1  - \beta }{ p_{i,t}  } -1 \le \frac{L   }{  \gamma } -1 :=M .
\end{align*}
Let $a_t=0$. 
We apply Theorem~\ref{thm:conc_var_ub}.
For all $\lambda\in (0,1)$, 
\begin{align*}
    &
    \Pr( X_T - \bbE X_T \ge \lambda  )  \le \exp \Bigg(  - \frac{ \lambda^2 }{  \sum_{t=1}^T   \sigma_t^2 + { M \lambda} /{3}  }  \Bigg),
    \qquad
    \bbE X_T = \sum_{t=1}^T  \bbE y_t =0.
\end{align*}
Therefore, for all $i\in[L]$, $\lambda_i \in (0,1)$,
\begin{align}
    \bbP( \tilde{G}_{i,T}  - G_{i,T}  \ge \lambda_i)
    \le  \exp \Bigg(  - \frac{ \lambda_i^2 }{  \sum_{t=1}^T   \sigma_t^2 + { M \lambda_i} /{3}  }  \Bigg)
    .  \label{eq:ub_per_item}
\end{align}

\noindent
{\bf Lower bound on $ \tilde{G}_{i,T}  - G_{i,T} $.} 
Similarly, 
we have
$
    X_{t-1} - X_{t } = - y_{ t}  
        \le  1 :=M'$.
We apply Theorem~\ref{thm:conc_var_lb}.
Therefore, for all $i \in [L]$, $\lambda_i \in (0,1)$, we have 
\begin{align}
    &
    \Pr( \tilde{G}_{i,T}  - G_{i,T} \le  
    - \lambda_i)
    \le 
    \exp \Bigg(  - \frac{ \lambda_i^2 }{  \sum_{t=1}^T   \sigma_t^2 + M' \lambda_i/3     }  \Bigg)
    .
    \label{eq:lb_per_item}
\end{align}

\textbf{Step 3: Last step.}
We decompose the failure probability as follows:
\begin{restatable}{lemma}{lemmaBdErrProbPreliminary}
\label{lemma:bd_err_prob_preliminary}
    For any fixed time budget $T$, we have
    \begin{align*}
        \Pr( \iout \ne 1  )  
        \le \Pr\bigg(  \tilde{G}_{1,T} - G_{1,T} \le - \frac{ T \Delta_{1,2} }{2} \bigg) 
        + \sum_{i=2}^L \Pr \bigg( \tilde{G}_{i,T} - G_{i,T} \ge \frac{ T \Delta_{1,2} }{2}  + T\cdot \Delta_{2,i}  \bigg).
    \end{align*}
\end{restatable}

Combining~\eqref{eq:ub_per_item}, \eqref{eq:lb_per_item} and
Lemma~\ref{lemma:bd_err_prob_preliminary} 
, we see that
\begin{align*}
    & 
    \lambda_1     \le \frac{ T \Delta_{1,2} }{2}  ,
    \quad
      \lambda_i \le \frac{ T \Delta_{1,2} }{2}  + T\cdot \Delta_{2,i} \quad \forall i \neq 1,
    \\ \Rightarrow~ &
    \Pr( \iout \ne 1 )
    \le \exp \Bigg(  - \frac{ \lambda_1^2 }{   \sum_{t=1}^T   \sigma_t^2 + { M' \lambda_1} /{3}    }  \Bigg)
        +
        \sum_{i=2}^L \exp \Bigg(  - \frac{ \lambda_i^2 }{    \sum_{t=1}^T   \sigma_t^2   + { M \lambda_i} /{3}    }  \Bigg).
\end{align*}

Let
\begin{align*}
    &
    \lambda_1:=
    \frac{ T \Delta_{1,2} }{2}  
    \qquad
    \lambda_i := \frac{ T \Delta_{1,2} }{2}  + T\cdot \Delta_{2,i} 
    \quad \forall i \neq 1.
\end{align*}
We now complete the proof:
\begin{align*}
    \Pr( \iout \ne 1 )
    &
    \le \exp \Bigg(  - \frac{ T\Delta_{1,2}^2/4 }{  L / \gamma  }  \Bigg)
        +
        \sum_{i=2}^L \exp \Bigg(  - \frac{ T(\Delta_{1,2}/2 + \Delta_{2,i} )^2 }{ {  L(3  +\Delta_{1,2}/2  + \Delta_{2,i}  )}/{ (3 \gamma) }   }  \Bigg)
    \\&
    = \exp \Bigg(  - \frac{  \gamma T\Delta_{1,2}^2  }{ 4L   }  \Bigg)
        +
        \sum_{i=2}^L \exp \Bigg(  - \frac{  3\gamma T(\Delta_{1,2}/2 + \Delta_{2,i} )^2 }{   L(3  +\Delta_{1,2}/2 + \Delta_{2,i}   )}    \Bigg)
    .
\end{align*}
%
%
\end{proof}

\subsection{Proof of Lemma~\ref{lemma:bd_err_prob_preliminary} }

\lemmaBdErrProbPreliminary*

\begin{proof}
We observe that 
\begin{align}
    &
    \Pr( \iout \ne 1  ) = \Pr(  \exists i \ne 1: \tilde{G}_{i,T} \ge \tilde{G}_{1,T}  )
    \nonumber
    \\&
    \le \Pr \bigg(  \exists i \ne 1: \tilde{G}_{i,T} - G_{i,T} \ge \frac{ T \Delta_{1,2} }{2}  + T\cdot \Delta_{2,i}, \text{ or } \tilde{G}_{1,T} - G_{1,T} \le - \frac{ T \Delta_{1,2} }{2}  
    \bigg)
    \label{eq:bd_err_prob_decomp}
    \\&
    \le \Pr\bigg(  \tilde{G}_{1,T} - G_{1,T} \le - \frac{ T \Delta_{1,2} }{2} \bigg) 
        + \sum_{i=2}^L \Pr \bigg( \tilde{G}_{i,T} - G_{i,T} \ge \frac{ T \Delta_{1,2} }{2}  + T\cdot \Delta_{2,i}  \bigg).
        \label{eq:bd_err_prob_preliminary}
\end{align}
It is trivial to obtain~\eqref{eq:bd_err_prob_preliminary} with \eqref{eq:bd_err_prob_decomp}. Now we complete the proof with the derivation of \eqref{eq:bd_err_prob_decomp}.
We denote
\begin{align*}
    & \calE_{1,T} := \bigg\{ \tilde{G}_{1,T} - G_{1,T} \le - \frac{ T \Delta_{1,2} }{2} \bigg\},
    \qquad
    \calE_{i,T} := \bigg\{ \tilde{G}_{i,T} - G_{i,T} \ge \frac{ T \Delta_{1,2} }{2}  + T\cdot \Delta_{2,i} \bigg\}  \quad \forall i \neq 1 .
\end{align*}
We can rewrite \eqref{eq:bd_err_prob_decomp} as $\bbP( \exists i \neq 1: \tilde{G}_{i,T} \ge \tilde{G}_{1,T} ) \le \bbP( \bigcup_{i=1}^L \calE_{i,T} )$. Hence, it is sufficient to show 
$$ \bigg\{ \bigcap_{i=1}^{L} \calE_{i,T}^C \bigg\}  ~\Rightarrow~  \{ \tilde{G}_{i,T} < G_{i,T} \quad \forall i \ne 1  \} $$
as follows. When $\bigcap_{i=1}^{L} \calE_{i,T}^C$ holds, for any $i\ne 1$, we have
\begin{align*}
    & \tilde{G}_{i,T} < G_{i,T} + \frac{ T \Delta_{1,2} }{2}  + T\cdot \Delta_{2,i}
        = G_{1,T} - T\cdot \Delta_{1,T} + \frac{ T \Delta_{1,2} }{2}  + T\cdot \Delta_{2,i} 
    \\&
     < \tilde{G}_{1,T} + \frac{ T \Delta_{1,2} }{2} - T\cdot \Delta_{1,T} + \frac{ T \Delta_{1,2} }{2}  + T\cdot \Delta_{2,i} 
          = \tilde{G}_{1,T}  T \cdot \Delta_{1,2}  - ( T \cdot \Delta_{1,2} + T\cdot \Delta_{2,T} )  + T\cdot \Delta_{2,i} 
   \\&
   = \tilde{G}_{1,T} .
\end{align*}
\end{proof}

\section{Analysis of global performances of adversarial algorithms}
 
\label{sec:pf_adv_lb_trade_off}

\subsection{Proof of Theorem~\ref{thm:adv_bai_lb} }
\label{pf:thm_adv_bai_lb}
 
 \thmAdvBaiLb*

\begin{proof}

{\bf Step 1: Construct instances.}
To begin, we let $Z_1, Z_2, \ldots, Z_T$ be a sequence of i.i.d.\ Gaussian random variables with
mean $1/2$ and variance $\sigma^2 \in  [ 1, \infty )$. Let $\varepsilon \in (0,1/2)$ be a constant that will be chosen
differently in each proof.
Under instance $ \ell$~($1 \le \ell \le  L$),
Let $g_{i,t}^\ell$ be the random gain of item $i$ at time step $t$, where
\begin{align*}
    &
    g_{i,1}^\ell = 
    \left\{
        \begin{array}{ll}
            1/2  & \text{if }  i=1 \\
            1/2 +   \varepsilon    & \text{if }  i=\ell \neq 1 \\
            1/2 -  \varepsilon     & \text{else} 
        \end{array}
    \right.
    ,
    \quad
    g_{i,t}^\ell = 
    \left\{
        \begin{array}{ll}
            \mathrm{clip}_{[0,1]} ( Z_t  )  & \text{if }  i=1 \\
            \mathrm{clip}_{[0,1]} ( Z_t +   \varepsilon )  & \text{if }  i=\ell \neq 1 \\
            \mathrm{clip}_{[0,1]} ( Z_t -  \varepsilon  )  & \text{else}
        \end{array}
    \right. 
    \quad \forall t>1.
\end{align*}
Note that $\mathrm{clip}_{[a,b]}x:= \max\{  a, \min\{ b , x\}  \} $ for $a\le b$.
Under instance $\ell$~($1\le \ell \le L$), we define notations as follows: 
\begin{itemize}
    \item[(i)] We let $G_{i,t}^\ell = \sum_{u=1}^t g_{i,t}^\ell$ and $ T \cdot \barDelta_{i,j}^\ell = G_{i,T}^\ell - G_{j,T}^\ell$ for all $i,j\in[L]$, which indicates that $\ell =\argmax_{i\in[L]}  G_{i,T}^\ell$ is the optimal item. 
    \item[(ii)] We let $i_t^\ell$ be the pulled item at time step $t$, and $O_\ell^t = \{ i_u^\ell, g_{  i_u^\ell,\tau } \}^t_{u=1}$ be the sequence of pulled items and observed gains up to and including time step $t$.
    \item[(iii)] We let $\bbP_\ell^t $ be the measure on $ O_\ell^t$, and let $P_{\ell,i}$ be the measure on the gain of item $i$.
    \item[(iv)] We define $\barDelta_{\min}^\ell := \min\limits_{j\ne \ell }   \barDelta_{\ell,j}^\ell $.
\end{itemize} 
For simplicity, we abbreviate $\bbP_\ell^T$, $O_\ell^T$ as $\bbP_\ell$, $O_\ell$ respectively. 
%
Moreover,
we let $N_{i}(t)$ denote the number of pulls of item $i$ up to and including time step $t$.

\noindent {\bf Step 2:  Change of measure.} 
%
%
 %
 First of all, we apply Lemmas~\ref{lemma:kl_to_event} and \ref{lemma:kl_decomp} obtain that for all $1\le \ell \le L$, 
 \begin{align*}
     \Pr_{ O_1 }( \iout \neq 1  ) +   \Pr_{O_\ell} ( \iout  = 1  )
     \ge \frac{ 1}{2 } \exp( - \mathrm{KL}(\bbP_1  \parallel \bbP_{\ell}  ) )
     = \frac{ 1}{2 } \exp(   - \bbE_{\bbP_1 }  [ N_{ \ell }(T) ] \cdot \mathrm{KL} ( P_{1,\ell} \parallel P_{\ell, \ell} )    ).
\end{align*}

Now we turn to bound $\bbE_{\bbP_1 }  [ N_{ \ell }(T) ]$.
Since $\sum_{\ell =1}^L \bbE_{\bbP_1} [N_\ell(T)] = T$, there exists $2\le \ell_2 \le  L $ such that $\bbE_{\bbP_1} [N_{\ell_2}(T) ] \le T/L$.

Further, with Lemma~\ref{lemma:kl_gauss_clip},
we can see that 
 \begin{align*} 
     %
     %
     %
     \Pr_{ O_1 } \iout  \neq 1  ) +   \Pr_{ \bbP_{\ell_2} } (\iout = 1  )
     &
     \ge 
     \frac{ 1}{2 } \exp(   - \bbE_{\bbP_1 }  [ N_{ \ell_2 }(T) ] \cdot \mathrm{KL} ( P_{1,\ell_2 } \parallel P_{\ell_2, \ell_2 } )    )
     \\&
     \ge
    \frac{ 1}{2 } \exp\bigg(   -   \frac{T}{L}   \cdot  \frac{ (2\varepsilon)^2 }{ 2\sigma^2 }  \bigg).
\end{align*}     
Since $\Pr_{O_{\ell_2} } (\iout \neq \ell_2 ) \ge \Pr_{O_{\ell_2} } (\iout = 1  )$, we have
\begin{align*}
    \max_{ 1 \le \ell \le L } \Pr_{O_\ell} (\iout \neq \ell ) \ge 
    \frac{ 1}{4 } \exp\bigg(   -   \frac{  2T \varepsilon^2 }{  \sigma^2  L } \bigg)
    . 
\end{align*}
 
\noindent {\bf Step 3: Comparison between $\varepsilon$ and $  \barDelta_{\min }^\ell   $.}
(i)
\underline{Under instance $1$}, since $G_{1,T}^1 \ge G_{i,T}^1 =G_{j,T}^1$ for all $i,j\ne 1$, i.e.,  item $1$ is the optimal item and all other items are suboptimal with identical rewards after $T$ time steps, we have
    $$\barDelta_{\min}^1  = \min_{j\ne 1 }   \barDelta_{1,j}^1 =     \barDelta_{1,2}^1 = \frac{ G_{1,T}^1 - G_{2,T}^1 }{T} =  \frac{1}{T}  \cdot \sum_{t=1}^T  ~[ \mathrm{clip}_{[0,1]} ( Z_t  ) - \mathrm{clip}_{[0,1]} ( Z_t -  \varepsilon  )  ].$$

Let $X_t =  \mathrm{clip}_{[0,1]} ( Z_t  ) - \mathrm{clip}_{[0,1]} ( Z_t -  \varepsilon  )   $ and   $X = \sum^T_{t=1} X_t$. 
Then $X = T \barDelta_{\min}^1 $.
We have
\begin{align*}
    X_t
    \ge  [ \mathrm{clip}_{[0,1]} [ Z_t + \varepsilon ) - \mathrm{clip}_{[0,1]} ( Z_t   ) ] \cdot \mathsf{1}\{  \varepsilon \le Z_t \le 1- \varepsilon  \} 
    = \varepsilon \cdot  \mathsf{1}\{   \varepsilon \le Z_t \le 1-  \varepsilon  \}  .
\end{align*}
%
Let $z\sigma = 1/2 - \varepsilon$. Theorem~\ref{thm:conc_gauss_single}  implies that
\begin{align}
    &
    \Pr_{Z_t} \bigg(~ \bigg| Z_t - \frac{ 1 }{2} \bigg| \ge \frac{ 1 }{2} - \varepsilon \bigg) \le
    \exp\bigg( - \frac{  (1/2 - \varepsilon)^2 }{ 2 \sigma^2 }\bigg)
    \nonumber \\
    ~\Rightarrow~
    &
    \Pr_{Z_t}  ( \varepsilon \le  Z_t  \le 1 - \varepsilon ) \ge 1- \exp\bigg(-\frac{ (1 -2\varepsilon)^2}{  8 \sigma^2 }\bigg) : = p(\varepsilon, \sigma).
    \nonumber 
\end{align}
Hence,we have
$
    \bbE X_t \ge 
     \varepsilon \cdot p( \varepsilon,\sigma ) 
$.

Since $X_t\in[0, \varepsilon]$ for all $t$, $X_1/\varepsilon, \ldots  , X_T/\varepsilon$ are independent $[0, 1]$-valued random variables, Theorem~\ref{thm:conc_var_chernoff} indicates that for all $a \in(0,1)$,
     \begin{align} 
         \Pr (      X  - \mathbb{E} X  \le -a \mathbb{E} X )    \le  \exp \bigg(     - \frac{ a^2  \mathbb{E} X   }{3 \varepsilon} \bigg)
         ~\Leftrightarrow~
         \Pr( X \ge (1-a) \bbE X ) \ge 1 -   \exp \bigg(     - \frac{ a^2  \mathbb{E} X   }{3 \varepsilon} \bigg)
         .
         \label{eq:gauss_apply_conc}
     \end{align}
Let $b=1-a$.   
Since $\bbE X_t \ge \varepsilon\cdot p(\varepsilon,\sigma)$ for all $t$,
\begin{align*}
    & 
    \Pr( \ X \ge b T \varepsilon\cdot p(\varepsilon,\sigma) \ ) \ge 1 -   \exp \bigg(     - \frac{ (1-b)^2  T  \cdot p(\varepsilon,\sigma)  }{3  } \bigg) \quad \forall b  \in (0,1) .
\end{align*}
In order words,
\begin{align*}
    & 
    \Pr  ( \ \barDelta_{\min }^1 \ge  b \varepsilon\cdot p(\varepsilon,\sigma) ~)  \ge 1 -   \exp \bigg(     - \frac{ (1-b)^2  T  \cdot p(\varepsilon,\sigma)  }{3  } \bigg) \quad \forall b  \in (0,1), \ \varepsilon \in (0,1/2) .
\end{align*}

 \noindent (ii) 
\underline{Under instance $\ell~(\ell \ne 1)$}, 
since $G_{\ell,T}^\ell \ge G_{1,T}^\ell \ge G_{i,T}^\ell =G_{j,T}^\ell$ for all $i,j\notin \{ 1,\ell\}$, i.e.,  item $\ell$ is the optimal item, item $1$ is the second optimal item, and all other items are with identical smaller rewards after $T$ time steps, we have
    $$\barDelta_{\min}^\ell  = \min_{j\ne \ell }   \barDelta_{\ell ,j}^\ell =     \barDelta_{\ell,1}^\ell = \frac{ G_{\ell,T}^\ell - G_{1,T}^\ell }{T} =  \frac{1}{T}  \cdot \sum_{t=1}^T  ~[ \mathrm{clip}_{[0,1]} ( Z_t +\Delta ) - \mathrm{clip}_{[0,1]} ( Z_t   )  ].$$ 
    
    Let $X'_t =  \mathrm{clip}_{[0,1]} ( Z_t  + \varepsilon ) - \mathrm{clip}_{[0,1]} ( Z_t   )   $ and   $X' = \sum^T_{t=1} X'_t$. 
Then $X' = T \barDelta_{\min}^\ell $.
We have
\begin{align*}
    X'_t
    \ge  [   \mathrm{clip}_{[0,1]} ( Z_t  + \varepsilon ) - \mathrm{clip}_{[0,1]} ( Z_t   )   ] \cdot \mathsf{1}\{  \varepsilon \le Z_t \le 1- \varepsilon  \} 
    = \varepsilon \cdot  \mathsf{1}\{   \varepsilon \le Z_t \le 1-  \varepsilon  \}  .
\end{align*}
%
We again apply Theorem~\ref{thm:conc_gauss_single}  to  
$
    \bbE X'_t \ge 
     \varepsilon \cdot p( \varepsilon,\sigma ) 
$.
Moreover, Theorem~\ref{thm:conc_var_chernoff} implies that 
\begin{align*}
    & 
    \Pr  ( \ \barDelta_{\min}^\ell \ge  b \varepsilon\cdot p(\varepsilon,\sigma) ~)  \ge 1 -   \exp \bigg(     - \frac{ (1-b)^2  T  \cdot p(\varepsilon,\sigma)  }{3  } \bigg) \quad \forall b  \in (0,1), \ \varepsilon \in (0,1/2) .
\end{align*}

\noindent (iii)
\underline{Altogether, for all $1\le \ell \le L$}, we have
\begin{align}
    &
   \Pr  ( \ \barDelta_{\min}^\ell \ge  b \varepsilon\cdot p(\varepsilon,\sigma)  ~)  \ge 1 -   \exp \bigg(     - \frac{ (1-b)^2 T  \cdot p(\varepsilon,\sigma)  }{3   } \bigg) \quad \forall b  \in (0,1), \ \varepsilon \in (0,1/2) .
    \label{eq:sto_reg_bai_lb_gap}
\end{align}

\noindent 
 {\bf Step 4: Consider the instance with the largest error probability.}
Let $ 1\le \ell_3 \le L $ satisfy that
\begin{align*}
    &
    \Pr_{O_{\ell_3} } (\iout \neq \ell_3 ) 
    \ge   
    \frac{ 1}{4 } \exp\bigg(   -   \frac{  2T \varepsilon^2 }{  \sigma^2  L } \bigg)
    . 
\end{align*}
Note that $\barDelta_{\min}^1 ,\ldots, \barDelta_{\min}^L $  are all determined by $ \{ Z_t \}_{t=1}^T$.
We let $ O' :=    O_{\ell_3}   \cup \{ Z_t \}_{t=1}^T $.
Then for all $ b\in (0,1) $, 
\begin{align*}  
    \Pr_{O_{\ell_3}  } (\iout \neq \ell_3 ) 
    & 
    \ge 
    \Pr_{O' } \ \big(~ \iout \neq \ell_3, \ \barDelta_{\min}^{\ell_3} \ge b \varepsilon\cdot p(\varepsilon,\sigma) ~\big)
    \\&
    \ge
    \Pr_{O' } \big(~ \iout \neq \ell_3 ~\big|~\barDelta_{\min}^{\ell_3} \ge b \varepsilon \cdot p(\varepsilon,\sigma) ~\big) \cdot \Pr_{O' } \big( ~ \barDelta_{\min}^{\ell_3} \ge b \varepsilon\cdot p(\varepsilon,\sigma) ~ \big)
    \\
    &
    \ge  
    \frac{ 1}{4 } \exp\bigg(   -   \frac{  2T \varepsilon^2 }{  b^2 \sigma^2  L \cdot p^2(\varepsilon,\sigma) } \bigg)\cdot 
        \bigg[  1 -   \exp \bigg(     - \frac{ (1-b)^2 T \cdot p(\varepsilon,\sigma)  }{3 } \bigg) \bigg].
\end{align*}
Let $\varepsilon = 1/10$, $\sigma = 1/3$, $b=7/10$.
Then 
\begin{align*} 
    &
    p(\varepsilon, \sigma)
    = 1- \exp\bigg(-\frac{ (1 -2\varepsilon)^2}{  8 \sigma^2 }\bigg) 
    = 1- \exp\bigg(-  \frac{9}{8}\cdot \bigg( \frac{4}{5} \bigg)^2   \bigg) 
    = 1- \exp\bigg(-  \frac{18}{ 25}  \bigg)  \ge \frac{ 1 }{2}
    .
\end{align*}
Moreover,
\begin{align*} 
    \Pr_{O_{\ell_3} } (\iout \neq \ell_3 ) 
    & 
    \ge \frac{ 1}{4 } \exp\bigg(   -   \frac{  2T (\barDelta_{\min}^{\ell_3})^2 }{  (7/10)^2  \cdot(1/3)^2  \cdot L \cdot (1/2)^2 } \bigg)\cdot 
        \bigg[  1 -   \exp \bigg(     - \frac{ (3/10)^2  \cdot T \cdot (1/2)  }{3 } \bigg) \bigg]
   \\&
    = \frac{ 1}{4 } \exp\bigg(   -   \frac{  7200 T (\barDelta_{\min}^{\ell_3})^2 }{  49 L } \bigg)\cdot 
        \bigg[  1 -   \exp \bigg(     - \frac{  3T    }{ 200 } \bigg) \bigg]
    \\&
    \ge \frac{ 1}{4 } \exp\bigg(   -   \frac{  150 T (\barDelta_{\min}^{\ell_3})^2 }{  L } \bigg)\cdot 
        \bigg[  1 -   \exp \bigg(     - \frac{  3T    }{ 200 } \bigg) \bigg]
        .
\end{align*}
When $T \ge 10$, since $ 1 -   \exp (     - {  3T    }/{ 200 } )  \ge 8/65$,
%
we have 
\begin{align*}
    \Pr_{O_{\ell_3} } (\iout \neq \ell_3 )  
    \ge  
%
    \frac{ 2}{65 } \exp\bigg(   -   \frac{  150 T (\barDelta_{\min}^{\ell_3} )^2 }{  L } \bigg) 
    .
\end{align*}

\noindent
{\bf Step 5: Classification of instances.}
Let $\barcalB_1( \underline{\Delta}_{ T} ,\barR  )$ denote the set of instances where 
(i)  the empirically-minimal optimality gap $  \barDelta_{ T}
\ge \underline{\Delta}_{ T} $;
and (ii) there exists $R_0 \in \bbR$ such the rewards are
bounded in $[R_0, R_0 + R]$. 
Then
\begin{align*}
	\sup_{ \barcalB_1(  \underline{\Delta}_{ T}, 1) }
	\Pr ( \iout \ne \bari^*_T )
	\ge 
	\frac{ 1 -   \exp (     - {  3T    }/{ 200 } )  }{4}  \cdot
      \exp\bigg(   -   \frac{  150 T  \underline{\Delta}_{ T}^2 }{  L } \bigg) 
      \quad \forall 0<\underline{\Delta}_{ T}\le 1.
\end{align*}
When $T\ge 10$, 
\begin{align*}
	\sup_{ \barcalB_1(  \underline{\Delta}_{ T},1) }
	\Pr ( \iout \ne \bari^*_T )
	\ge 
    \frac{ 2}{65 } \exp\bigg(   -   \frac{  150 T  \underline{\Delta}_{ T}^2 }{  L } \bigg) 
     \quad \forall 0<\underline{\Delta}_{ T}\le 1 .
\end{align*}
%
%
%
\end{proof}

\subsection{Proof of Theorem~\ref{thm:adv_rm_bai_gauss} }
\label{pf:thm_adv_rm_bai_gauss}

\thmAdvRmBaiGauss*

\begin{proof}


 The analysis is similar to that of Theorem~\ref{thm:adv_bai_lb} (See Appendix~\ref{pf:thm_adv_bai_lb}).
 We construct $L$ instances in the same way.
 
 \noindent
 {\bf Step 1: Change of measure.} We apply Lemmas~\ref{lemma:kl_decomp} and \ref{lemma:kl_to_event} to show that 
 \begin{align*}
     \Pr_{ O_1 }(\iout \neq 1  ) +   \Pr_{O_\ell} (\iout = 1  )
     \ge \frac{ 1}{2 } \exp( - \mathrm{KL}(\bbP_1  \parallel \bbP_{\ell}  ) )
     = \frac{ 1}{2 } \exp(   - \bbE_{\bbP_1 }  [ N_{ \ell }(T) ] \cdot \mathrm{KL} ( P_{1,\ell} \parallel P_{\ell, \ell} )    ).
\end{align*}          
In order to upper bound $\bbE_{\bbP_1 }  [ N_{ \ell }(T) ]$, we lower bound the number of time steps that $g_{1,t}^1-g_{\ell,t}^1 = \varepsilon$, i.e., $Z_t \in [0,1 - \varepsilon]$.
Let $z\sigma = 1/2 - \varepsilon$. We again apply Theorem~\ref{thm:conc_gauss_single} to obtain \eqref{eq:gauss_apply_conc}:
\begin{align}
    &
    \Pr_{Z_t}  ( \varepsilon \le  Z_t  \le 1 - \varepsilon ) \ge 1- \exp\bigg(-\frac{ (1 -2\varepsilon)^2}{  8 \sigma^2 }\bigg)   = p(\varepsilon, \sigma).
    \nonumber 
\end{align}
Since the expectation of the empirical-regret is upper bounded by $\overline{\mathrm{Reg}}$, we have
\begin{align*}
    \overline{\mathrm{Reg}}
    & \ge \bbE_{ \bbP_1 } \Bigg[  \sum_{t=1}^T \mathsf{1}\{ i_t^1 \neq 1 \} \cdot ( g_{1,t}^1 - g_{i_t,t}^1 ) \Bigg]
    \ge \bbE_{ \bbP_1 } \Bigg[  \sum_{t=1}^T \mathsf{1}\{ i_t^1 \neq 1,  \Delta \le Z_t \le 1-\varepsilon \} \cdot ( g_{1,t}^1 - g_{ i_t ,t}^1 ) \Bigg]
    \\&
    =  \sum_{t=1}^T \bbE_{ \bbP_1^t }  [   \mathsf{1}\{ i_t^1 \neq 1 \} \cdot [\mathrm{clip}_{[0,1]} ( Z_t  ) - \mathrm{clip}_{[0,1]} ( Z_t -\varepsilon ) ] ~|~\varepsilon \le  Z_t \le 1 - \varepsilon   ] \cdot \Pr_{Z_t} (  \varepsilon \le Z_t \le 1 - \varepsilon )
    \\&
    =  \sum_{t=1}^T \bbE_{ \bbP_1^t }  [   \mathsf{1}\{ i_t^1 \neq 1 \}  \cdot [ Z_t    -  ( Z_t -\varepsilon ) ]   ~|~\varepsilon \le  Z_t \le 1 - \varepsilon   ] \cdot  p(\varepsilon, \sigma)
    \\&
    =  \sum_{t=1}^T \bbE_{ \bbP_1^t }  [   \mathsf{1}\{ i_t^1 \neq 1 \}  ]\cdot \varepsilon   \cdot  p(\varepsilon, \sigma)
    =\varepsilon \cdot p(\varepsilon, \sigma) \cdot  \sum_{\ell=2}^L \bbE_{ \bbP_1 }[ N_\ell(T) ]   . 
\end{align*}   
Hence, there exists $2 \le \ell_0 \le L$ such that
\begin{align*}
    \bbE_{ \bbP_1 }[ N_{\ell_0} (T) ]  \le \frac{  \overline{\mathrm{Reg}} }{ \varepsilon \cdot p(\varepsilon, \sigma) \cdot (L-1) }  .
\end{align*}
Further, we again apply Lemma~\ref{lemma:kl_gauss_clip} to obtain
%
 \begin{align*}
     \Pr_{ O_1 }(\iout \neq 1  ) +   \Pr_{O_{\ell_0} } (\iout = 1  )
     &
     \ge 
     \frac{ 1}{2 } \exp(   - \bbE_{\bbP_1 }  [ N_{ \ell_0 }(T) ] \cdot \mathrm{KL} ( P_{1,\ell_0 } \parallel P_{\ell_0, \ell_0 } )    )
     \\&
     \ge
    \frac{ 1}{2 } \exp\bigg(   -    \frac{  \overline{\mathrm{Reg}} }{ \varepsilon \cdot p(\varepsilon, \sigma) \cdot (L-1) }   \cdot  \frac{ (2\varepsilon)^2 }{ 2\sigma^2 }  \bigg).
\end{align*}

Since $\Pr_{O_{\ell_0} } (\iout \neq \ell_0 ) \ge \Pr_{O_{\ell_0} } (\iout = 1  )$, we have
\begin{align*}
    \max_{ 1 \le \ell \le L } \Pr_{O_\ell} (\iout \neq \ell ) \ge 
    \frac{ 1}{4 } \exp\bigg(   -    \frac{  2 \varepsilon \overline{\mathrm{Reg}} }{  \sigma^2   \cdot p(\varepsilon, \sigma) \cdot (L-1) }  \bigg)
\end{align*}
where
$
 p(\varepsilon, \sigma) = 1- \exp [- { (1 -2\varepsilon)^2 }/{  (8 \sigma^2 )} ]  
 $.  
 
\noindent 
 {\bf Step 2: Consider the instance with the largest error probability.} 
Recall~\eqref{eq:sto_reg_bai_lb_gap} from the analysis of Theorem~\ref{thm:adv_bai_lb} in Appendix~\ref{pf:thm_adv_bai_lb}. 
For all $1\le \ell \le L$, we have
\begin{align*}
    &
   \Pr  ( \ \barDelta_{\min}^\ell \ge  b \varepsilon\cdot p(\varepsilon,\sigma)  ~)  \ge 1 -   \exp \bigg(     - \frac{ (1-b)^2 T  \cdot p(\varepsilon,\sigma)  }{3   } \bigg) \quad \forall b  \in (0,1), \ \varepsilon \in (0,1/2) .
\end{align*}

Let $ 1\le \ell_2 \le L $ satisfy that
\begin{align*}
    &
    \Pr_{O_{\ell_2} } (\iout \neq \ell_2 ) 
    \ge   
    \frac{ 1}{4 } \exp\bigg(   -    \frac{  2 \varepsilon \overline{\mathrm{Reg}} }{  \sigma^2   \cdot p(\varepsilon, \sigma) \cdot (L-1) }  \bigg).
\end{align*}
Note that $\barDelta_{\min}^1 ,\ldots, \barDelta_{\min}^L $  are all determined by $ \{ Z_t \}_{t=1}^T$.
We let $ O' :=    O_{\ell_2}  \cup \{ Z_t \}_{t=1}^T $.
Then for all $ b\in (0,1) $,  $1\le \ell \le L$,
\begin{align*}  
    \Pr_{ O_{\ell_2}  } (\iout \neq \ell_2 ) 
    & 
    \ge 
    \Pr_{O' } \ \big(~ \iout \neq \ell_2, \ \barDelta_{\min}^{\ell } \ge b \varepsilon\cdot p(\varepsilon,\sigma) ~\big)
    \\&
    \ge
    \Pr_{O' } \big(~ \iout \neq \ell_2  ~\big|~\barDelta_{\min}^{\ell } \ge b \varepsilon \cdot p(\varepsilon,\sigma) ~\big) \cdot \Pr_{O' } \big( ~ \barDelta_{\min}^{\ell } \ge b \varepsilon\cdot p(\varepsilon,\sigma) ~ \big)
    \\
    &
    \ge  
    \frac{ 1}{4 } \exp\bigg(   -    \frac{  2 \barDelta_{\min}^{\ell } \cdot \overline{\mathrm{Reg}} }{  \sigma^2   \cdot b \cdot p^2(\varepsilon, \sigma) \cdot (L-1) }  \bigg) \cdot 
        \bigg[  1 -   \exp \bigg(     - \frac{ (1-b)^2 T \cdot p(\varepsilon,\sigma)  }{3 } \bigg) \bigg].
\end{align*}
Since this inequality holds for all $1\le \ell \le L $, we let $\barDelta_{\min} = \min_{1\le \ell \le L}  \barDelta_{\min}^\ell$, and have
\begin{align*}  
    \Pr_{ O_{\ell_2}  } (\iout \neq \ell_2 )  
    \ge  
    \frac{ 1}{4 } \exp\bigg(   -    \frac{  2 \barDelta_{\min} \cdot \overline{\mathrm{Reg}} }{  \sigma^2   \cdot b \cdot p^2(\varepsilon, \sigma) \cdot (L-1) }  \bigg) \cdot 
        \bigg[  1 -   \exp \bigg(     - \frac{ (1-b)^2 T \cdot p(\varepsilon,\sigma)  }{3 } \bigg) \bigg].
\end{align*}
%
We again let $\varepsilon = 1/10$, $\sigma = 1/3$, $b=7/10$.
Then $ p( \varepsilon,\sigma ) \ge 1/2 $, and
\begin{align*}  
    \Pr_{O_{\ell_2} } (\iout \neq \ell_2 ) 
    &  
    \ge  
    \frac{ 1}{4 } \exp\bigg(   -    \frac{  2 \barDelta_{\min} \cdot \overline{\mathrm{Reg}} }{  (1/3)^2   \cdot (7/10) \cdot (1/2)^2 \cdot (L-1) }  \bigg) \cdot 
        \bigg[  1 -   \exp \bigg(     - \frac{ (3/10)^2 \cdot T \cdot (1/2)  }{3 } \bigg) \bigg]
     \\&
    \ge  
    \frac{ 1}{4 } \exp\bigg(   -    \frac{  720 \barDelta_{\min} \cdot \overline{\mathrm{Reg}} ~}{  7 (L-1) }  \bigg) \cdot 
        \bigg[  1 -   \exp \bigg(     - \frac{ 3T   }{ 200 } \bigg) \bigg]
    \\&
    \ge  
    \frac{ 1}{4 } \exp\bigg(   -    \frac{  103 \barDelta_{\min}  \cdot \overline{\mathrm{Reg}} ~}{  L-1 }  \bigg) \cdot 
        \bigg[  1 -   \exp \bigg(     - \frac{ 3T   }{ 200 } \bigg) \bigg].
\end{align*}
When $T \ge 10$, since $ 1 -   \exp (     - {  3T    }/{ 200 } )  \ge 8/65$, we have
\begin{align*}  
    \Pr_{O_{\ell_2}} (\iout \neq \ell_2 ) 
    &   
    \ge  
    \frac{ 2}{ 65 } \exp\bigg(   -    \frac{  103 \barDelta_{\min} \cdot \overline{\mathrm{Reg}} ~}{  L-1 }  \bigg)  .
\end{align*}
Suppose algorithm $\pi$ satisfies that
\begin{align*}
    \Pr_{O_{\ell_2}} (\iout \neq \ell_2 )  \le \frac{2}{65} \exp(-\psi_T),
\end{align*} 
then we have 
\begin{align*}
    \overline{\mathrm{Reg} } \ge \frac{  L-1 }{  103 \barDelta_{\min}  }  \cdot \psi_T.
\end{align*}
 
\noindent
{\bf Step 3: Classification of instances.}
Suppose algorithm $\pi$ satisfies that $\bare_T (\pi) \le 2\exp(-\psi_T )/65$.
We again consider  $\barcalB_1( \underline{\Delta}_{ T},1   )$ as in Theorem~\ref{thm:adv_bai_lb}.
Recall that $\barcalB_1( \underline{\Delta} ,\barR  )$ denote the set of instances where 
(i)  the empirically-minimal optimality gap $  \barDelta_{\min,T}
\ge \underline{\Delta}_{ T} $ in $T$ time steps;
and (ii) there exists $R_0 \in \bbR$ such the rewards are
bounded in $[R_0, R_0 + R]$. 
When $T\ge 10$,
\begin{align*}
    \sup_{ \calI \in \barcalB_1(\underline{\Delta}_{ T}, 1) }
    \bbE \barR_T (\pi,\calI) \ge \psi_T \cdot \frac{  L-1 }{  103 \underline{\Delta}_{ T} } \quad \forall 0<\underline{\Delta}_{ T}\le 1.
\end{align*}
%
%
%
%
\end{proof}

\section{Additional numerical results}

\label{append:extra_experiment}
We present the failure probabilities and counts of algorithms in different instances in Appendix \ref{append:experiment_table_emp_failure}.
%
We provide additional numerical results for both synthetic and real datasets in Appendices \ref{append:extra_experiment_syn_data} and \ref{append:extra_experiment_real_data} respectively. We also elaborate more details about the experiment setup of the PKIS2 dataset in Appendix~\ref{append:extra_experiment_real_data}.

\subsection{Empirical failure probability of {\sc BoBW-lil'UCB$(\gamma)$}}
\label{append:experiment_table_emp_failure}

\begin{table}[H]
  \centering
    \renewcommand{\arraystretch}{1.4}
  \caption{Empirical failure probability below $1\%$}
    \begin{tabular}{l "c | c | c | c |c | c}
          & \multicolumn{4}{c|}{Bernoulli instances} & \multicolumn{1}{c|}{ML-25M}  & \multicolumn{1}{c}{PKIS2} \\
          \hline
    $L$ & \multicolumn{2}{c|}{$64$} & \multicolumn{2}{c|}{$128$} &  $22$  & $109$\\
    \hline
    $\Delta$ & \multicolumn{1}{c|}{$0.05$} & \multicolumn{1}{c|}{$0.1$} & \multicolumn{1}{c|}{$0.1$} & \multicolumn{1}{c|}{$0.2$} & \multicolumn{1}{l}{ }   & \\
    \thickhline
    {\sc BoBW$(6 \times10^{-7})$} & $0.83\%$ & $0.22\%$ & $0.78\%$ & $0.20\%$ & $0.96\%$ & $0.92\%$ \\
    \hline
    {\sc BoBW$(6 \times10^{-4})$} & $0.65\%$ & $0.17\%$ & $0.62\%$ & $0.18\%$ & $0.72\%$ & $0.23\%$ \\
    \hline
    {\sc BoBW$(6 \times10^{-1})$} & $0.11\%$ & $0.03\%$ & $0.14\%$ & $0.04\%$ & $0.06\%$ & $0$ \\
%
    \end{tabular}%
  \label{tab:err_prob_err001}%
    \renewcommand{\arraystretch}{1}
\end{table}%


\begin{table}[H]
  \centering
    \renewcommand{\arraystretch}{1.4}
  \caption{Empirical failure probability below $ 2\%$}
    \begin{tabular}{l "c | c | c | c |c | c}
          & \multicolumn{4}{c|}{Bernoulli instances} & \multicolumn{1}{c|}{ML-25M}  & \multicolumn{1}{c}{PKIS2} \\
          \hline
    $L$ & \multicolumn{2}{c|}{$64$} & \multicolumn{2}{c|}{$128$} &  $22$  & $109$\\
    \hline
    $\Delta$ & \multicolumn{1}{c|}{$0.05$} & \multicolumn{1}{c|}{$0.1$} & \multicolumn{1}{c|}{$0.1$} & \multicolumn{1}{c|}{$0.2$} & \multicolumn{1}{l}{ }   & \\
    \thickhline
    {\sc BoBW$(6 \times10^{-7})$} & $1.78\%$ & $1.60\%$ & $1.06\%$ & $0.31\%$ & $0.96\%$ & $1.68\%$ \\
    \hline
    {\sc BoBW$(6 \times10^{-4})$} & $1.21\%$ & $0.95\%$ & $0.61\%$ & $0.16\%$ & $0.72\%$ & $0.42\%$ \\
    \hline
    {\sc BoBW$(6 \times10^{-1})$} & $0.28\%$ & $0.23\%$ & $0.06\%$ & $0.03\%$ & $0.06\%$ & $0$ \\
%
    \end{tabular}%
  \label{tab:addlabel}%
    \renewcommand{\arraystretch}{1}
\end{table}%

\newpage
\subsection{Empirical regret of algorithms with empirical failure probability below $1\%$}

\label{append:experiment_table_regret}

\begin{table}[htbp]
  \centering
  \caption{Empirical regret of algorithms using synthetic data}
    \begin{tabular}{l | l |  l | l | l}
    Algorithm & $L$ & $\Delta$ &  Average regret  &  Standard deviation of regret  \\
    \thickhline
    BoBW$(9 \times 10^{-1})$ & $64$ & $0.05$ & $6.45\times 10^{3}$ & $5.10\times 10^{1}$ \\
    \hline
    BoBW$(9\times 10^{-4})$ & $64$ & $0.05$ & $6.59\times 10^{3}$ & $1.38\times 10^{1}$ \\
    \hline
    BoBW$(9\times 10^{-7})$ & $64$ & $0.05$ & $6.61\times 10^{3}$ & $9.27 $ \\
    \hline
    UCB$_{3.0}$ & $64$ & $0.05$ & $1.84\times 10^{4}$ & $2.31\times 10^{3}$ \\
    \hline
    UCB$_{4.5}$ & $64$ & $0.05$ & $2.34\times 10^{4}$ & $3.41\times 10^{3}$ \\
    \hline
    UCB$_{6.0}$ & $64$ & $0.05$ & $2.64\times 10^{4}$ & $4.22\times 10^{3}$ \\
    \thickhline
    BoBW$(9\times 10^{-1})$ & $64$ & $0.1$ & $3.78\times 10^{3}$ & $3.74\times 10^{1}$ \\
    \hline
    BoBW$(9\times 10^{-4})$ & $64$ & $0.1$ & $3.90\times 10^{3}$ & $8.53 $ \\
    \hline
    BoBW$(9\times 10^{-7})$ & $64$ & $0.1$ & $3.91\times 10^{3}$ & $5.64 $ \\
    \hline
    UCB$_{3.0}$ & $64$ & $0.1$ & $8.89\times 10^{3}$ & $1.12\times 10^{3}$ \\
    \hline
    UCB$_{4.5}$ & $64$ & $0.1$ & $1.13\times 10^{4}$ & $1.68\times 10^{3}$ \\
    \hline
    UCB$_{6.0}$ & $64$ & $0.1$ & $1.28\times 10^{4}$ & $2.07\times 10^{3}$ \\
    \thickhline
    BoBW$(9\times 10^{-1})$ & $128$ & $0.1$ & $6.82\times 10^{3}$ & $3.26\times 10^{1}$ \\
    \hline
    BoBW$(9\times 10^{-4})$ & $128$ & $0.1$ & $6.91\times 10^{3}$ & $7.37 $ \\
    \hline
    BoBW$(9\times 10^{-7})$ & $128$ & $0.1$ & $6.92\times 10^{3}$ & $4.83 $ \\
    \hline
    UCB$_{3.0}$ & $128$ & $0.1$ & $1.84\times 10^{4}$ & $2.21\times 10^{3}$ \\
    \hline
    UCB$_{4.5}$ & $128$ & $0.1$ & $2.35\times 10^{4}$ & $3.39\times 10^{3}$ \\
    \hline
    UCB$_{6.0}$ & $128$ & $0.1$ & $2.67\times 10^{4}$ & $4.19\times 10^{3}$ \\
    \thickhline
    BoBW$(9\times 10^{-1})$ & $128$ & $0.2$ & $3.87\times 10^{3}$ & $2.49\times 10^{1}$ \\
    \hline
    BoBW$(9\times 10^{-4})$ & $128$ & $0.2$ & $3.95\times 10^{3}$ & $4.47 $ \\
    \hline
    BoBW$(9\times 10^{-7})$ & $128$ & $0.2$ & $3.96\times 10^{3}$ & $2.86 $ \\
    \hline
    UCB$_{3.0}$ & $128$ & $0.2$ & $8.72\times 10^{3}$ & $1.09\times 10^{3}$ \\
    \hline
    UCB$_{4.5}$ & $128$ & $0.2$ & $1.11\times 10^{4}$ & $1.62\times 10^{3}$ \\
    \hline
    UCB$_{6.0}$ & $128$ & $0.2$ & $1.26\times 10^{4}$ & $2.01\times 10^{3}$ 
    \end{tabular}%
  \label{tab:experiment_reg_syn}%
\end{table}%

\begin{table}[htbp]
  \centering
  \caption{Empirical regret of algorithms using the ML-25M dataset}
    \begin{tabular}{l | l | l | l}
    Algorithm & $L$ & Average regret  &  Standard deviation of regret  \\
    \thickhline
    BoBW $(9\times 10^{-1}) $ & $22$    &  $4.05\times 10^3 $ &  $113.98  $ \\
    \hline
    BoBW $(9\times 10^{-4}) $ & $22$    &  $5.16\times 10^3 $ &  $43.93  $ \\
    \hline
    BoBW $(9\times 10^{-7}) $ & $22$    &  $5.38\times 10^3 $ &  $31.79  $ \\
    \hline
    UCB $_{3.0} $ & $22$    &  $7.05\times 10^3 $ &  $927.60  $ \\
    \hline
    UCB $_{4.5} $ & $22$    &  $11.22\times 10^3 $ &  $1.77\times 10^3 $ \\
    \hline
    UCB $_{6.0} $ & $22 $   &  $15.04\times 10^3 $ &  $2.67\times 10^3 $ \\
    \end{tabular}%
  \label{tab:experiment_reg_ml25}%
\end{table}%

\begin{table}[htbp]
  \centering
  \caption{Empirical regret of algorithms using the PKIS2 dataset}
    \begin{tabular}{l | l | l | l}
    Algorithm & $L$ & Average regret  &  Standard deviation of regret  \\
    \thickhline
    BoBW$(9\times 10^{-1})$ & $109$ & $8.63\times 10^6$ & $2.31\times 10^3$ \\
    \hline
    BoBW$(9\times 10^{-4})$ & $109$ & $8.73\times 10^6$ & $1.55\times 10^3$ \\
    \hline
    BoBW$(9\times 10^{-7})$ & $109$ & $8.78\times 10^6$ & $1.32\times 10^3$ \\
    \hline
    UCB$_{3.0}$ & $109$ & $17.01\times 10^6$ & $2.33\times 10^6$ \\
    \hline
    UCB$_{4.5}$ & $109$ & $21.50\times 10^6$ & $2.94\times 10^6$ \\
    \hline
    UCB$_{6.0}$ & $109$ & $26.31\times 10^6$ & $3.71\times 10^6$ \\
    \end{tabular}%
  \label{tab:experiment_reg_pkis2}%
\end{table}%

\vfill
\newpage
\subsection{Experiments using synthetic data}
\label{append:extra_experiment_syn_data}

In this section, we present more numerical results for larger instances with $L=128$ items. These figures  yield the same conclusions as in Section~\ref{sec:experiment_syn}.  

{\bf Experiments with empirical failure probabilities below $1\%$.}

 
\begin{figure}[H]
    \begin{flushright}
        \includegraphics[width=.38\textwidth]{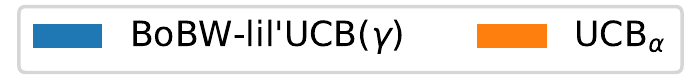} \hphantom{aaaaaaaaaaaaaaaaa}
    \end{flushright}
	\centering
	\includegraphics[width= .8\textwidth]{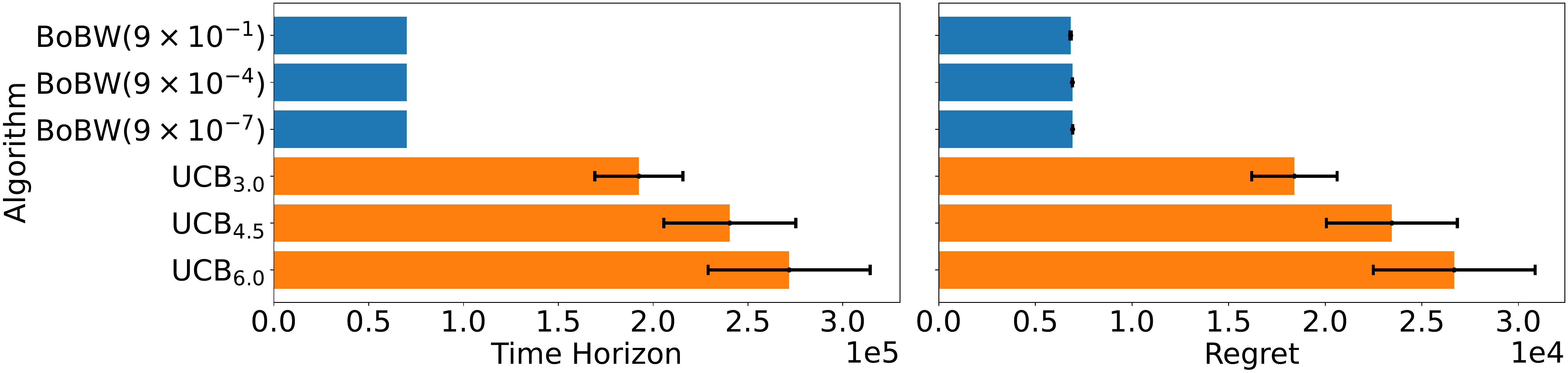}	
	\vspace{-.5em}
	\caption{Empirical failure probability $\le 1\%$: $L=128$, $ \Delta =0.1$, $\nu_i = \mathrm{Bern}(w_i)$.}
	\label{pic:bern_L128_wGapMin0_1_err001}  
\end{figure}

\begin{figure}[H]
	\centering
	\includegraphics[width= .8\textwidth]{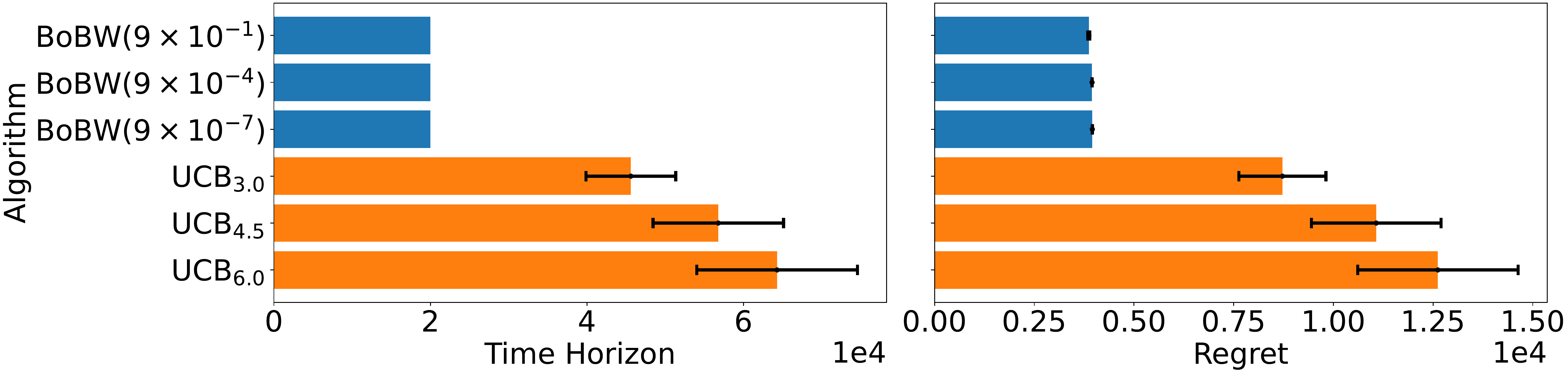}	
	\vspace{-.5em}
	\caption{Empirical failure probability $\le 1\%$: $L=128$, $ \Delta =0.2$, $\nu_i = \mathrm{Bern}(w_i)$.}
	\label{pic:bern_L128_wGapMin0_2_err001}  
\end{figure}

{\bf Experiments with empirical failure probabilities below $2\%$.}

\begin{figure}[H]
	\centering
	\includegraphics[width= .8\textwidth]{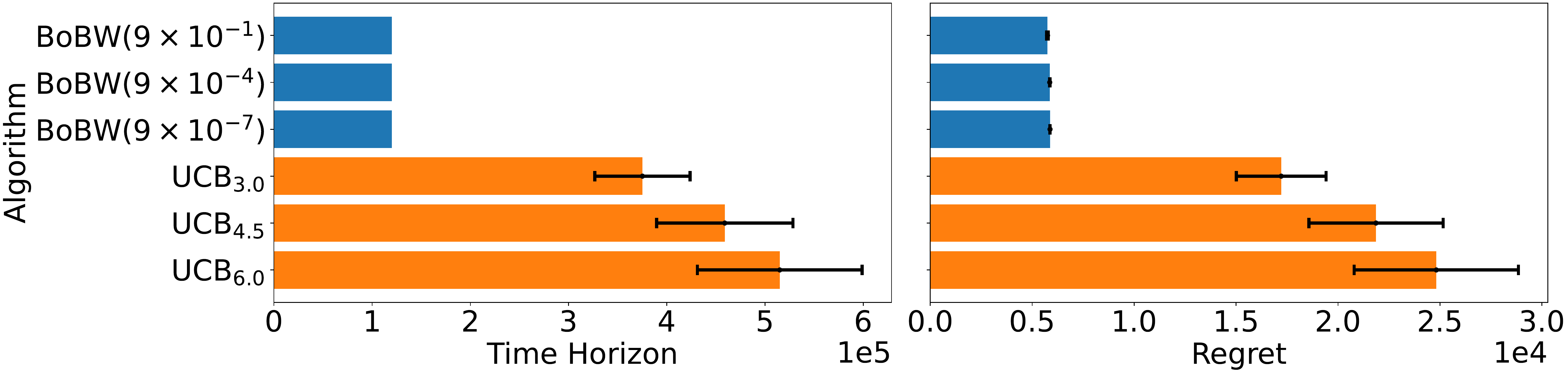}	
	\vspace{-.5em}
	\caption{Empirical failure probability $\le 2\%$: $L=64$, $ \Delta =0.05$, $\nu_i = \mathrm{Bern}(w_i)$.}
	\label{pic:bern_L64_wGapMin0_05_err002}  
\end{figure}

\begin{figure}[H]
	\centering
	\includegraphics[width= .8\textwidth]{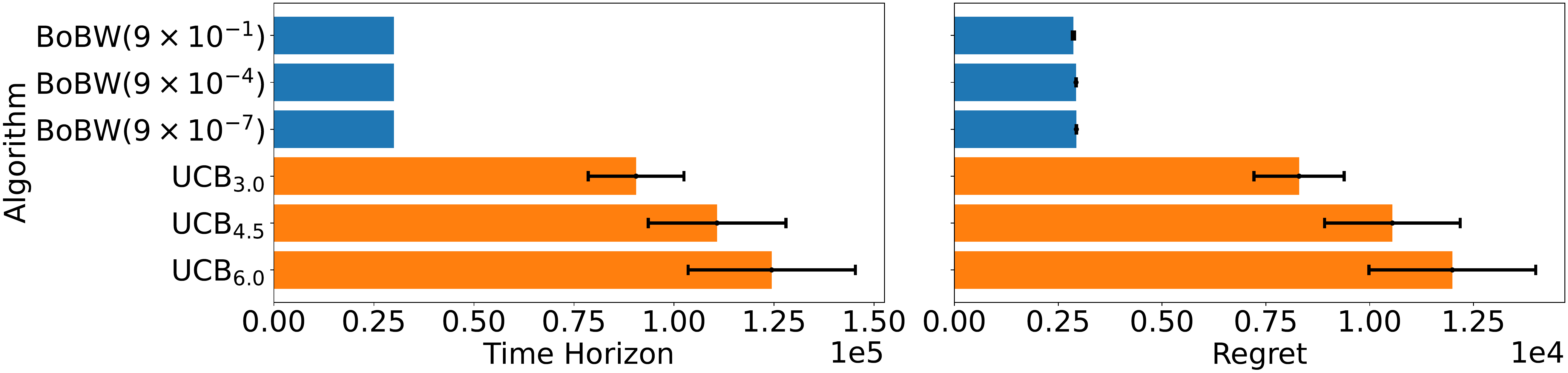}	
	\vspace{-.5em}
	\caption{Empirical failure probability $\le 2\%$: $L=64$, $ \Delta =0.1$, $\nu_i = \mathrm{Bern}(w_i)$.}
	\label{pic:bern_L64_wGapMin0_1_err002}  
\end{figure}

\begin{figure}[H]
	\centering
	\includegraphics[width= .8\textwidth]{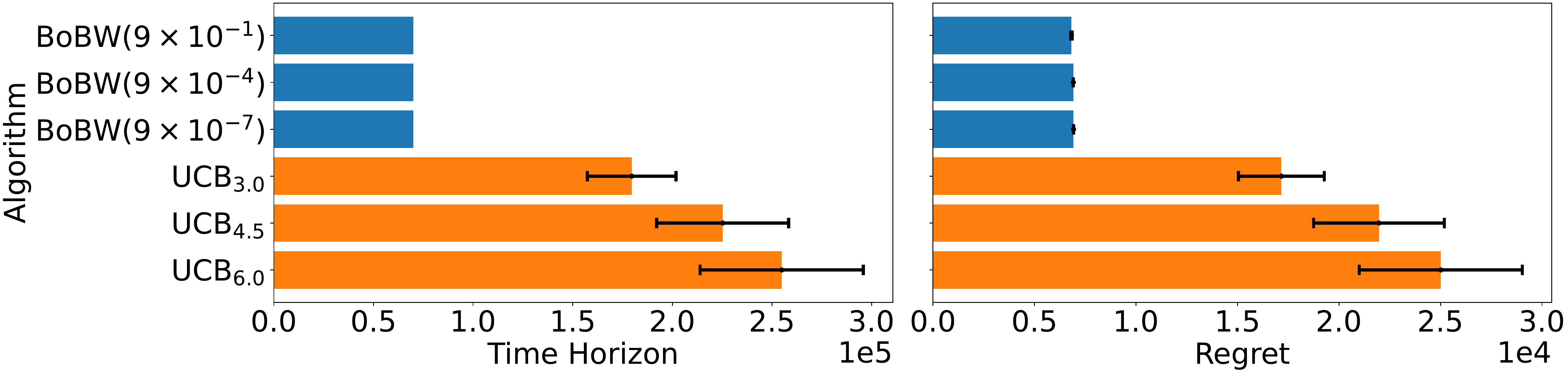}	
	\vspace{-.5em}
	\caption{Empirical failure probability $\le 2\%$: $L=128$, $ \Delta =0.1$, $\nu_i = \mathrm{Bern}(w_i)$.}
	\label{pic:bern_L128_wGapMin0_1_err002}  
\end{figure}

\begin{figure}[H]
	\centering
	\includegraphics[width= .8\textwidth]{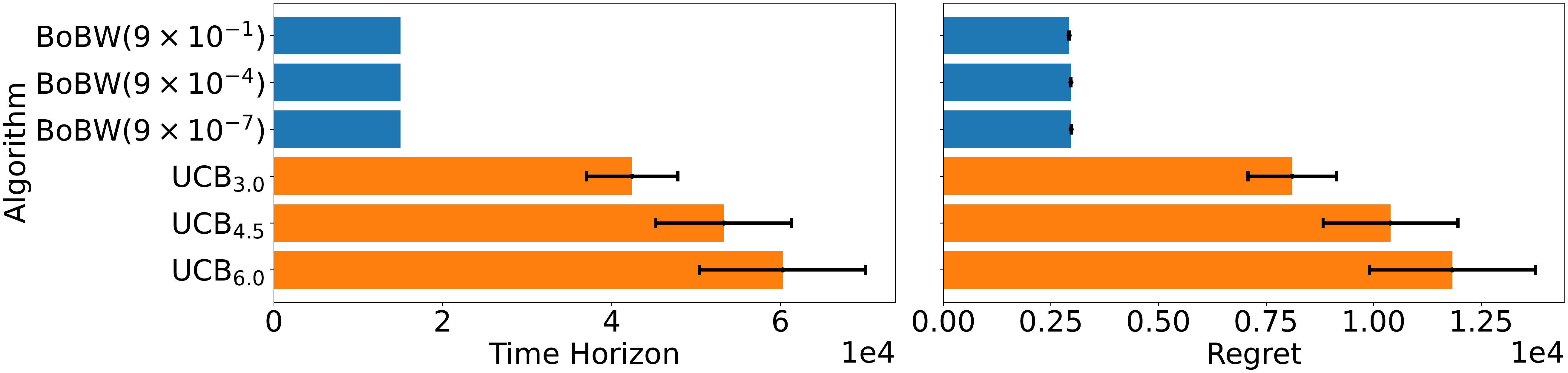}	
	\vspace{-.5em}
	\caption{Empirical failure probability $\le 2\%$: $L=128$, $ \Delta =0.2$, $\nu_i = \mathrm{Bern}(w_i)$.}
	\label{pic:bern_L128_wGapMin0_2_err002}  
\end{figure}

\subsection{Experiments using real data}
\label{append:extra_experiment_real_data}

{\bf PKIS2 dataset.}
The repository tests 
$641$
 small molecule compounds (kinase inhibitor)  against 
 $406$
 protein kinases.
This experiment aims to find the most effective inihibitor against a targeted kinase, and is a fundamental study in cancer drug discovery.
%
PKIS2 presents a `percentage inhibition' for each inhibitor, which is averaged over several trials. 
For each entry, we normalize it to be between $0$ and $1$, and then obtain the \emph{percentage control} by subtracting each of the normalized entries from $1$. The percentage control can help understand how effective the inhibitor is against the targeted kinase. 
Since \citet{christmann2016unprecedently} reported that these values have log-normal distributions with variance less than $1$, we sample random variables form a standard normal distribution with the log of the percent control as the mean;
the similar setup was used in \citet{mason2020finding,mukherjee2021mean}.
In our experiment, we select the inhibitors tested against one specific kinase MAPKAPK5. We aim to find out the most effective inhibitor with the highest percentage control against  MAPKAPK5, and also obtain high percentage controls cumulatively during the process.
Our results may benefit the experiments that test inhibitors with genuine cancer patients, which helps to identify the most effective inhibitor with a fixed number of tests and provide effective solutions to the attendants during the tests.

{\bf Experiments with empirical failure probabilities below $2\%$.}

\begin{figure}[th]
    \vspace{-.3em} 
	\centering
	\includegraphics[width=.8\textwidth]{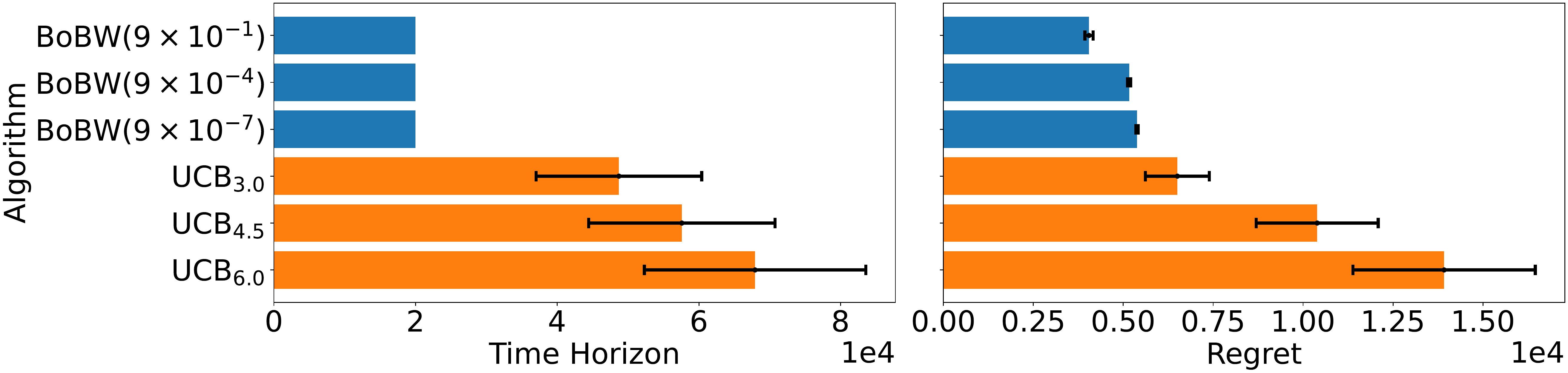} 
	\caption{Empirical failure probability $\le 2\%$. ML-25M:  $L=22$ movies with at least $50,000$ ratings.}
	\label{pic:ml25m_rating50_err002}   
    \vspace{-.3em}
\end{figure}

\newpage
\begin{figure}[th]
    \vspace{-.3em} 
	\centering
	\includegraphics[width=.8\textwidth]{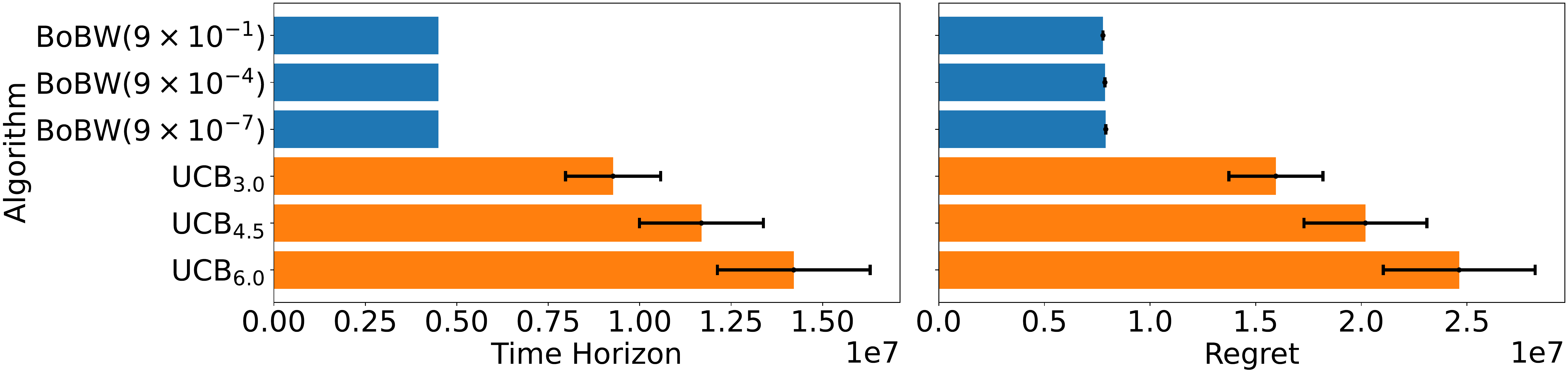} 
	\caption{Empirical failure probability $\le 2\%$. PKIS2: $L=109$ inhibitors tested against MAPKAPK5.}
	\label{pic:pkis2_MAPKAPK5_err002}   
    \vspace{-.3em}
\end{figure}

\end{document}